\documentclass{amsart}

\usepackage{amsthm}

\usepackage{thm-restate}

\usepackage{amsmath,amsfonts,bm}

\def\eqref#1{Equation~\ref{#1}}

\def\1{\bm{1}}

\def\va{{\bm{a}}}
\def\vb{{\bm{b}}}

\def\vx{{\bm{x}}}

\def\mC{{\bm{C}}}

\def\mG{{\bm{G}}}

\def\mO{{\bm{O}}}

\def\mLambda{{\bm{\Lambda}}}

\DeclareMathAlphabet{\mathsfit}{\encodingdefault}{\sfdefault}{m}{sl}
\SetMathAlphabet{\mathsfit}{bold}{\encodingdefault}{\sfdefault}{bx}{n}

\newcommand{\R}{\mathbb{R}}

\DeclareMathOperator*{\argmax}{arg\,max}
\DeclareMathOperator*{\argmin}{arg\,min}

\newcommand{\N}{\mathbb{N}}

\newcommand{\NM}{\mathcal{N}}
\newcommand{\veps}{\varepsilon}
\newcommand{\eps}{\veps}
\newcommand{\M}{\mathcal{M}}
\newcommand{\E}{\mathbf{E}}

\newcommand{\eee}{e_1}
\newcommand{\e}{e_2}
\newcommand{\ee}{e_3}
\newcommand{\x}{\mathbf{x}}

\newcommand{\I}{\mathbf{I}}

\newcommand{\ma}{\mathbf{A}}
\newcommand{\mb}{\mathbf{B}}
\newcommand{\mz}{\mathbf{Z}}
\newcommand{\mSigma}{\boldsymbol{\Sigma}}
\newcommand{\mxs}{\mathbf{X}^*}
\newcommand{\tx}{\mathbf{\mathcal{A}_n}}
\newcommand{\my}{\mathbf{Y}}
\newcommand{\mys}{\mathbf{Y}^*}
\newcommand{\mU}{\mathbf{U}}
\newcommand{\mD}{\mathbf{D}}
\newcommand{\mV}{\mathbf{V}}

\newcommand{\F}{\mathcal{F}}
\newcommand{\X}{\mathcal{X}}

\newcommand{\0}{\mathbf{0}}

\newcommand{\nc}{\normalcolor}

\newcommand{\mv}{\mathbf{v}}

\newcommand{\rF}{\mathrm{F}}

\global\long\def\defeq{\stackrel{\textrm{def}}{=}}%

\newtheorem{proposition}{Proposition}
\newtheorem{theorem}{Theorem}
\numberwithin{equation}{section}
\newtheorem{assumption}{Assumption}
\newtheorem{lemma}{Lemma}
\newtheorem{remark}{Remark}
\newtheorem{definition}[theorem]{Definition}
\newtheorem{corollary}{Corollary}

\numberwithin{equation}{section}
\numberwithin{lemma}{section}
\numberwithin{assumption}{section}
\numberwithin{remark}{section}

\usepackage{graphicx}
\usepackage{tcolorbox}
\usepackage{hyperref}
\usepackage{url}
\usepackage{subcaption}

\usepackage{booktabs, makecell} 
\usepackage{enumitem} 
\usepackage[noabbrev,capitalize]{cleveref} 
\newlist{propenum}{enumerate}{1} 
\setlist[propenum]{label=\alph*), ref=\theproposition~(\alph*)}
\crefalias{propenumi}{proposition} 
\crefformat{eq}{Eq.~(#2#1#3)}
\crefformat{sec}{Section~#2#1#3}
\crefformat{def}{Definition~#2#1#3}
\crefformat{opt}{~(#2#1#3)}
\crefformat{ineq}{Inequality~(#2#1#3)}
\crefformat{lemma}{Lemma~#2#1#3}
\crefformat{cor}{Corollary~#2#1#3}
\crefformat{theorem}{Theorem~#2#1#3}
\crefformat{alg}{Algorithm~#2#1#3}
\crefformat{claim}{Claim~#2#1#3}
\crefformat{prop}{Proposition~#2#1#3}

\usepackage{hyperref}

\usepackage{cleveref}
\usepackage{enumitem}
\makeatletter
\newcommand{\mylabel}[2]{#2\def\@currentlabel{#2}\label{#1}}
\makeatother

\usepackage{amsfonts}
\usepackage{amssymb}%
\usepackage{mathtools}
\usepackage{thmtools, thm-restate}

\numberwithin{theorem}{section}

 \usepackage{natbib}

\title[SNN: Approximation Theory and Optimization Landscape]{Spectral Neural Networks: Approximation Theory and Optimization Landscape}

\date{\today}

\author{Chenghui Li}
\address{Department of Statistics\\
University of Wisconsin Madison\\
Madison, Wisconsin, 53701, USA \\}
\email{cli539@stat.wisc.edu}

\author{Rishi Sonthalia}
\address{Department of Mathematics \\
University of California, Los Angeles \\
Los Angeles, California, 90095, USA}
\email{rsonthal@math.ucla.edu}

\author{Nicol\'as Garc\'ia Trillos}
\address{Department of Statistics\\
University of Wisconsin Madison\\
Madison, Wisconsin, 53701, USA }
\email{garciatrillo@.wisc.edu}
\thanks{\textbf{Acknowledgements:} The authors would like to thank Abiy Tasissa and Yuetian Luo for enlightening discussions on the topics covered in this paper. This material is based upon work supported by the National Science Foundation under Grant Number
DMS 1641020 and was started during the summer of 2022 when the authors participated in the AMS-MRC program: \textit{Data
Science at the Crossroads of Analysis, Geometry, and Topology.} NGT was supported by the NSF
grants DMS-2005797 and DMS-2236447. CL and NGT would like to thank the IFDS at UW-Madison and NSF through TRIPODS grant 2023239 for their support. }

\makeatletter
\renewcommand{\paragraph}{%
  \@startsection{paragraph}{4}%
  {\z@}{0ex \@plus 0ex \@minus .2ex}{-1em}%
  {\normalfont\normalsize\bfseries}%
}
\makeatother

\begin{document}

\maketitle

\begin{abstract}

There is a large variety of machine learning methodologies that are based on the extraction of spectral geometric information from data. However, the implementations of many of these methods often depend on traditional eigensolvers, which present limitations when applied in practical online big data scenarios. To address some of these challenges, researchers have proposed different strategies for training neural networks as alternatives to traditional eigensolvers, with one such approach known as Spectral Neural Network (SNN). In this paper, we investigate key theoretical aspects of SNN. First, we present quantitative insights into the tradeoff between the number of neurons and the amount of spectral geometric information a neural network learns. Second, we initiate a theoretical exploration of the optimization landscape of SNN's objective to shed light on the training dynamics of SNN. Unlike typical studies of convergence to global solutions of NN training dynamics, SNN presents an additional complexity due to its non-convex ambient loss function.

\end{abstract}

\section{Introduction}

In the past decades, researchers from a variety of disciplines have studied the use of spectral geometric methods to process, analyze, and learn from data. These methods have been used in supervised learning \cite{ando2006learning,belkin2006manifold,smola2003kernels}, clustering \cite{ng2001spectral,von2007tutorial}, dimensionality reduction \cite{belkin2001laplacian,coifman2005geometric}, and contrastive learning \cite{haochen2021provable}. 
While the aforementioned methods have strong theoretical foundations, their algorithmic implementations often depend on traditional eigensolvers. These eigensolvers tend to underperform in practical big data scenarios due to high computational demands and memory constraints. Moreover, they are particularly vulnerable in online settings since the introduction of new data typically necessitates a full computation from scratch.




To overcome some of the drawbacks of traditional eigensolvers, new frameworks for learning from spectral geometric information that are based on the training of neural networks have emerged. A few examples are Eigensolver net (See in Appendix \ref{appendix:other SNN}), Spectralnet \cite{shaham2018spectralnet}, and Spectral Neural Network (SNN) \cite{haochen2021provable}. In all the aforementioned approaches, the goal is to find neural networks that can approximate the spectrum of a large target matrix, and the differences among these approaches lie mostly in the specific loss functions used for training; here we focus on SNN, and provide some details on Eigensolver net and Spectralnet in Appendix \ref{appendix:other SNN} for completeness. To explain the training process in SNN, consider a data set $\X=\{x_1, \dots, x_n \}$ in $\R^d$ and a $n \times n$ adjacency matrix $\tx$ describing similarity among points in $\X_n$. A NN is trained by minimizing the \textit{spectral constrastive} loss function:
\begin{align}\label{eq:Appr}
    \min _{\theta \in \Theta}\,\, L(\theta)\defeq \ell(\my_\theta),\quad \text{ where } \quad \ell(\my) \defeq \left\|\my \my^{\top} - \tx \right\|_{\mathrm{F}}^2, \quad \my \in \R^{n \times r},
\end{align}
through first-order optimization methods; see more details in Appendix \ref{Appendix:Supplementary:Advantage of SNN}. In the above and in the sequel, $\theta$ represents the vector of parameters of the neural network $f_\theta: \R^d \rightarrow \R^r$, here a multi-layer ReLU neural network --see a detailed definition in Appendix \ref{app:MultilayerReLU}--, which can be interpreted as a feature or representation map for the input data; the matrix $\my_{\theta}$ is the $n \times r$ matrix whose rows are the outputs $f_\theta(x_1), \dots, f_\theta(x_n)$;  $\lVert \cdot \rVert_{\mathrm{F}}$ is the Frobenius norm. 

Compared with plain eigensolver approaches, SNN has the following advantages:
    \begin{enumerate}[nosep, leftmargin=*]
        \item \textbf{Training:} the spectral contrastive loss $\ell$ lends itself to minibatch training. Moreover, each iteration in the mini-batch training is cheap and only requires knowing the local structure of the adjacency matrix around a given point, making this approach suitable for online settings; see Appendix \ref{Appendix:Supplementary:Advantage of SNN} for more details.
        \item \textbf{Memory:} when the number data points is large, storing an eigenvector of $\tx$ may be costly, while SNN can trade-off between accuracy and memory by selecting the dimension of the space of parameters of the neural network. 
        \item \textbf{Out-of-sample extensions:} A natural out-of-sample extension is built by simple evaluation of the trained neural network at an arbitrary input point.
        
    \end{enumerate}

Motivated by these algorithmic advantages, in this paper we investigate some of SNN's theoretical underpinnings. In concrete terms, we explore the following three questions:
\begin{tcolorbox}
    \begin{enumerate}[nosep, leftmargin=*, label=\textnormal{(\arabic*)}]
        \item[\mylabel{question:1}{Q1}] Is it possible to approximate the eigenvectors of a large adjacency matrix with a neural network? How large does the neural network need to be to achieve a certain degree of approximation? 
        \item[\mylabel{question:2}{Q2}] Is it possible to use \eqref{eq:Appr} to build an approximating neural network? 
        \item[\mylabel{question:3}{Q3}] What can be said about the landscape of the objective function in \ref{eq:Appr}? 
    \end{enumerate}
\end{tcolorbox}

\paragraph{Contributions} We provide answers to the above three questions in a specific setting to be described shortly. We also formulate and discuss open problems that, while motivated by our current investigation, we believe are of interest in their own right.  

To make our setting more precise, through our discussion we adopt the \textit{manifold hypothesis} and assume the data set $\X=\{x_1, \dots, x_n \}$ to be supported on a low dimensional manifold $\M$ embedded in $\R^d$; see precise assumptions in Assumptions \ref{assum:data}. We also assume that $\X$ is endowed with a similarity matrix $\mG^\varepsilon$ with entries
\begin{align}\label{eq:eps-graph}
    \mG^\eps_{i j}=\eta\left(\frac{\left\|x_{i}-x_{j}\right\|}{\varepsilon}\right),
\end{align}
where $\|x-y\|$ denotes the Euclidean distance between $x$ and $y$, $\varepsilon$ is a proximity parameter, and $\eta$ is a decreasing, non-negative function. In short, $\mG^\veps$ measures the similarity between points according to their proximity. From $\mG^\veps$ we define the adjacency matrix $\tx$ appearing in \eqref{eq:Appr} by
\begin{align}\label{eq-def:Graph Laplacian}
    \tx\defeq \mD_\mG^{-\frac{1}{2}}\mG\mD_\mG^{-\frac{1}{2}}+a\I,
\end{align}
where $\mD_\mG$ is the degree matrix associated to $\mG$ and $a> 1$ is a fixed quantity. Here we distance ourselves slightly from the choice made in the original SNN paper \cite{haochen2021provable}, where $\tx$ is taken to be $\mG$ itself, and instead consider a normalized version. This is due to the following key properties satisfied by our choice of $\tx$ (see also Remark \ref{rem:AnVsDeltan} in Appendix \ref{app:PropertiesAn}) that make it more suitable for theoretical analysis.
\begin{proposition}\label{prop:An positive definite}
The matrix $\tx$ defined in \eqref{eq:Appr} satisfies the following properties:
\begin{enumerate}
    \item $\tx$ is symmetric positive definite.\label{prop:property1}
    \item  $\tx$'s $r$ \textit{top} eigenvectors (the ones corresponding to the $r$ largest eigenvalues) coincide with the eigenvectors of the $r$ \textit{smallest} eigenvalues of the symmetric normalized graph Laplacian matrix (see \cite{von2007tutorial}):
    \begin{equation}
      \Delta_n \defeq  \I - \mD_\mG^{-1/2}\mG \mD_\mG^{-1/2}.
      \label{eqn:GraphLaplacian}
    \end{equation}   \label{prop:property2}
\end{enumerate}
\end{proposition}




The above two properties, proved in Appendix \ref{app:PropertiesAn}, are useful when combined with recent results on the regularity of graph Laplacian eigenvectors over proximity graphs \cite{calder2022lipschitz} (see Appendix \ref{Append:Auxiliary Lemma:Manifold learning}) and some results on the approximation of Lipschitz functions on manifolds using neural networks \cite{chen2019nonparametric} (see Appendix \ref{append:NNUniversal}). In particular, we answer question \ref{question:1}, which belongs to the realm of approximation theory, by providing a concrete bound on the number of neurons in a multi-layer ReLU NN that are necessary to approximate the $r$ smallest eigenvectors of the normalized graph Laplacian matrix $\Delta_n$ (as defined in \ref{eqn:GraphLaplacian}) and thus also the $r$ largest eigenvectors of $\tx$;
this is the content of Theorem \ref{thm-appr:main}. 


\begin{figure}[htbp]
		\par\medskip
		\centering
		\begin{minipage}[t]{0.45\textwidth}
			\centering
			\includegraphics[width=5.5cm]{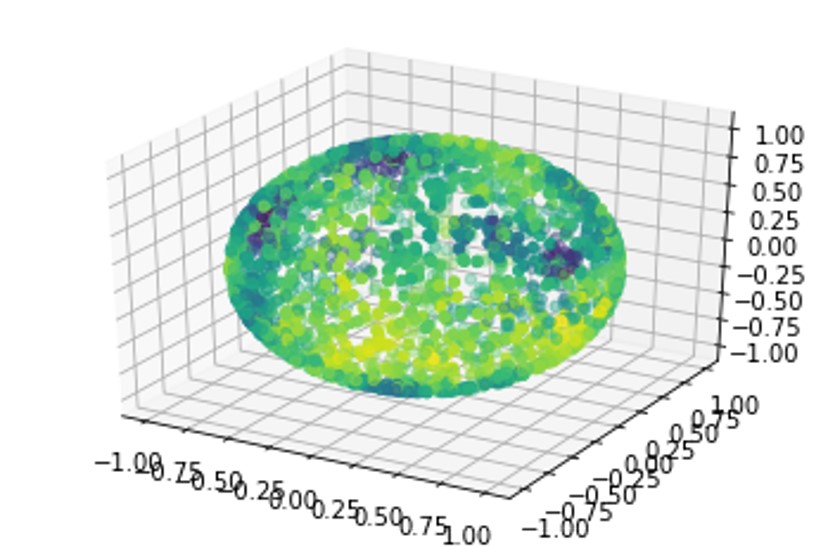}
			\caption{(A)}
			\label{fig:Illustration NN}
		\end{minipage}
		\begin{minipage}[t]{0.45\textwidth}
			\centering
			\includegraphics[width=5.5cm]{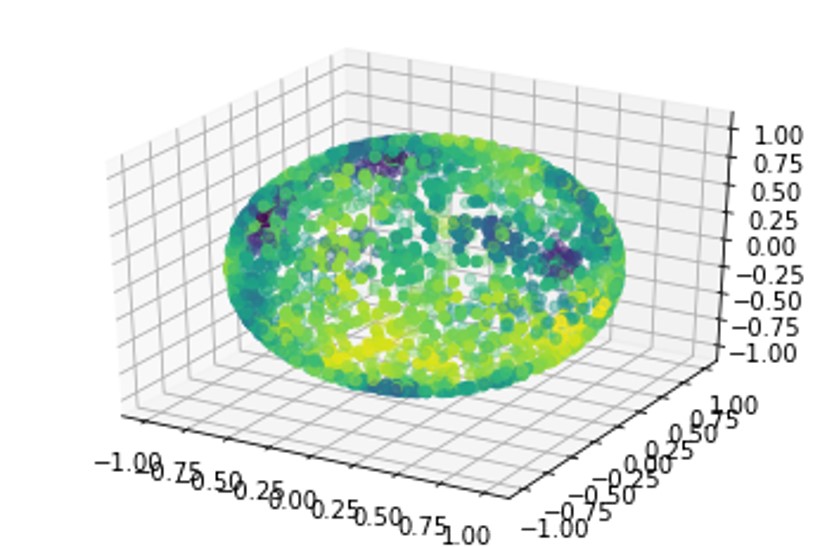}
			\caption{(B)}
			\label{fig:Illustration-Eigendecomposition}
		\end{minipage}	
		\caption{(B) shows the first eigenvector for the Laplacian of a proximity graph from data points sampled from $S^2$ obtained using an eigensolver. (A) shows the same eigenvector but obtained using SNN. The difference between the two figures is minor, showing that the neural network learns the eigenvector of the graph Laplacian well. See details in Appendix \ref{app:exp1}.}
		\label{fig:illustration}
	\end{figure}

\begin{figure}
    \centering
    \subfloat[Initialized Near Optimal \label{fig:NN-opt}]{\includegraphics[width=0.33\linewidth]{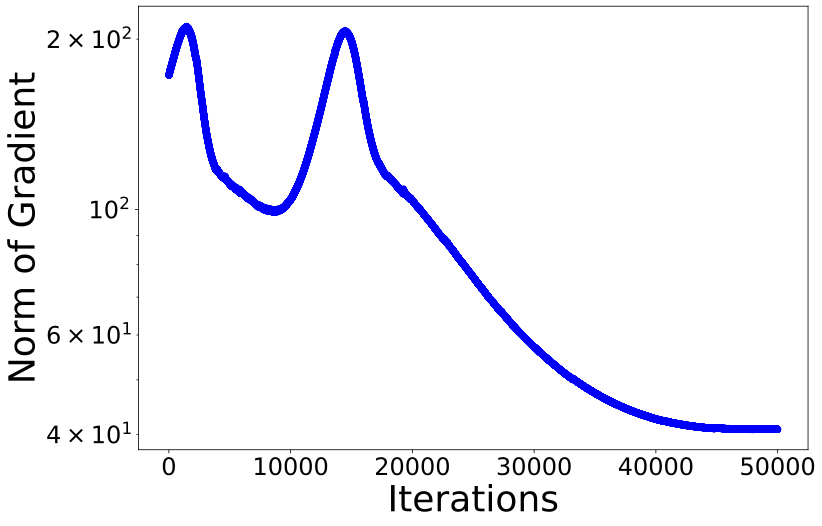}}
    \subfloat[Initialized Near Saddle \label{fig:NN-saddle}]{\includegraphics[width=0.33\linewidth]{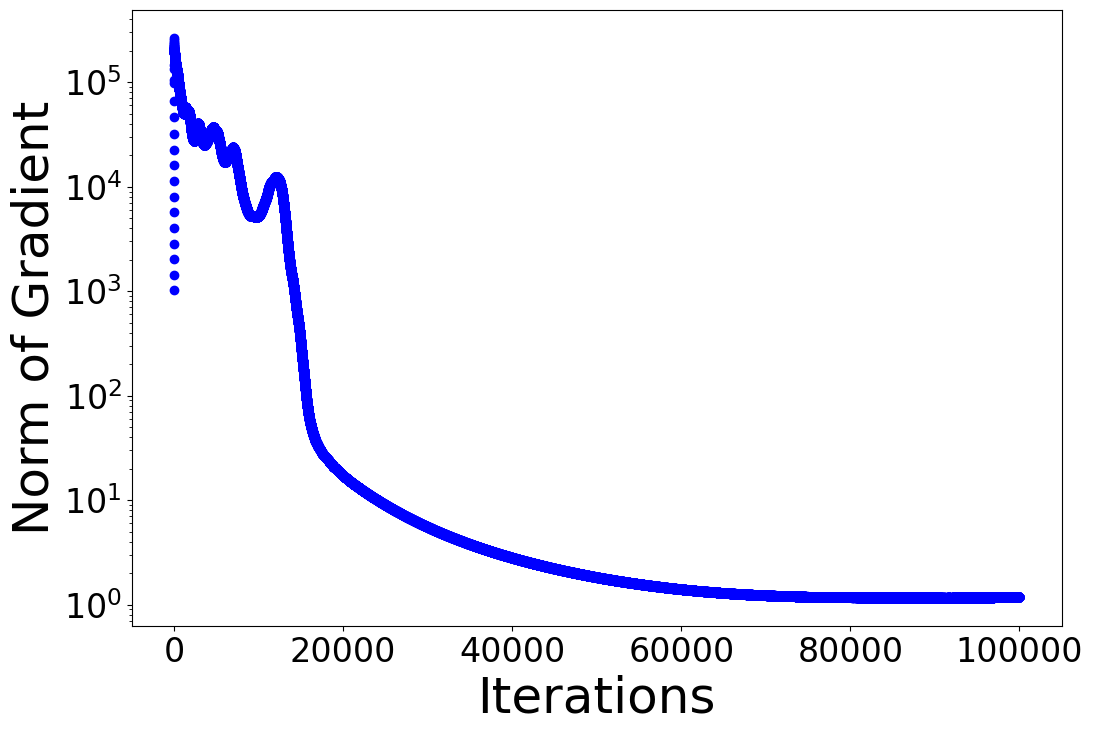}}
    \subfloat[Initialized Near Saddle \label{fig:NN-saddle-distance}]{\includegraphics[width=0.33\linewidth]{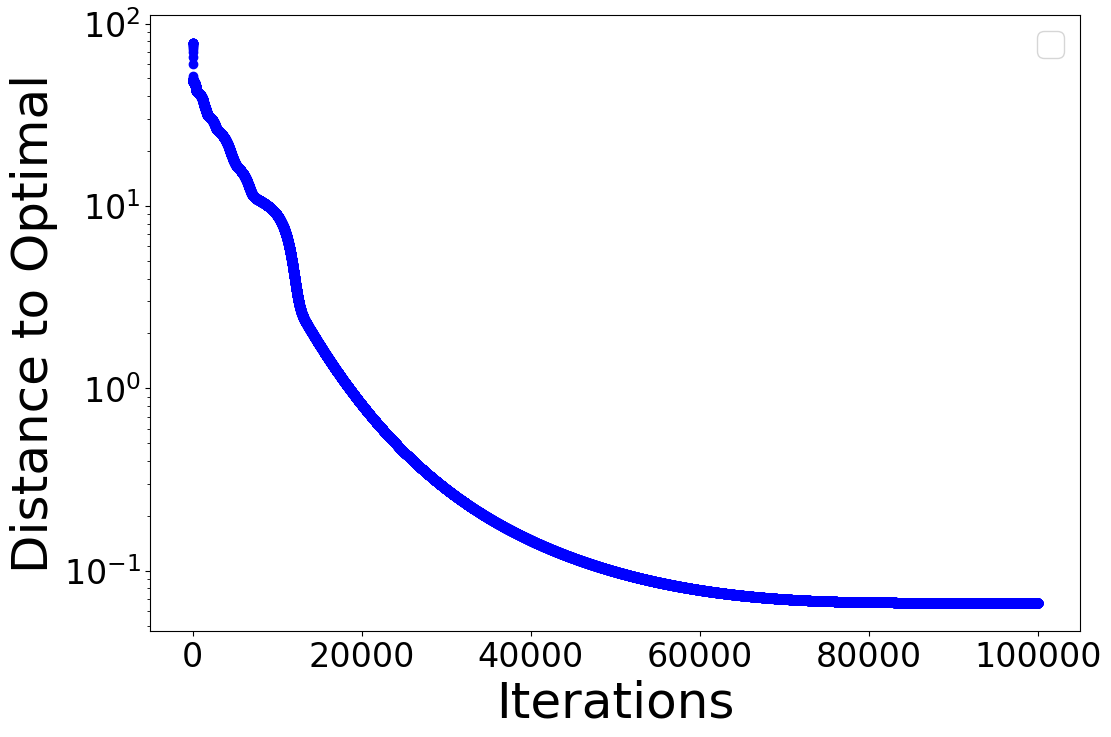}}
    \caption{(a) and (b) Sum of the norms of the gradients for a two-layer ReLU Neural Network. In (a), the network is initialized near the global optimal solution and in (b) the network is initialized near a saddle point. (c) shows the distance between the current outputs of the neural network and the optimal solution for the case when it was initialized near a saddle point. More details are presented in Appendix \ref{app:exp2}.}
    \label{fig:NN}
\end{figure}

\begin{figure}
    \centering
    \subfloat[Initialized Near Optimal \label{fig:ambient-opt}]{\includegraphics[width=0.33\linewidth]{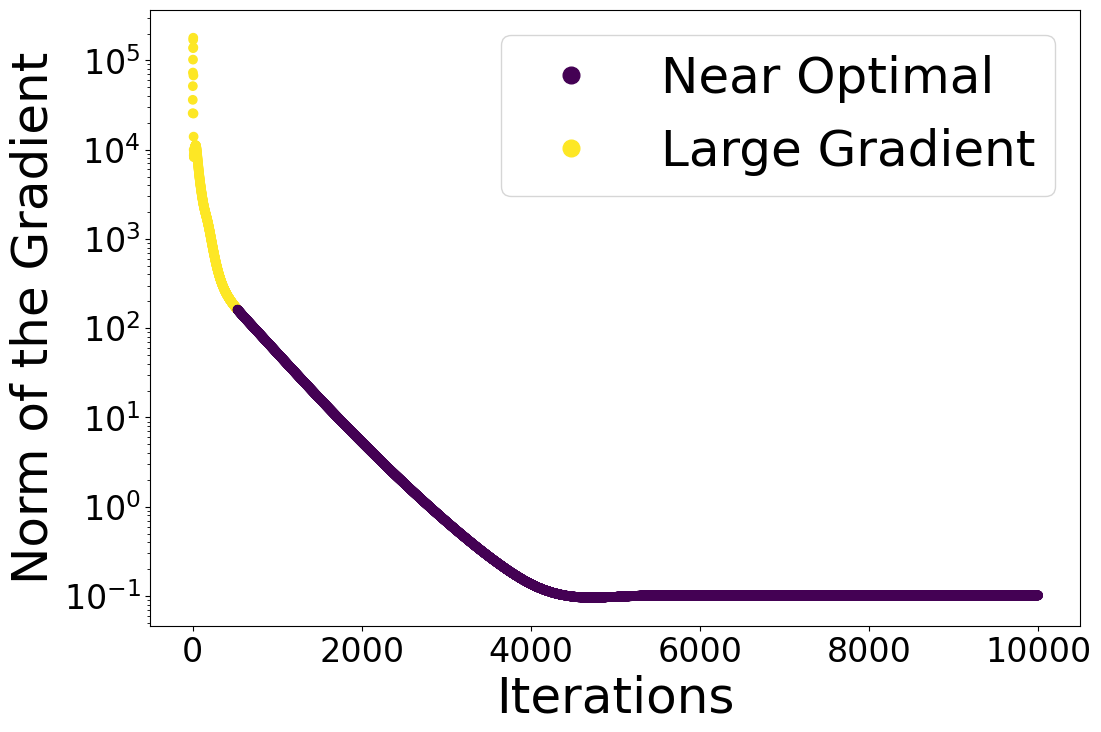}}
    \subfloat[Initialized Near Saddle \label{fig:ambient-saddle}]{\includegraphics[width=0.33\linewidth]{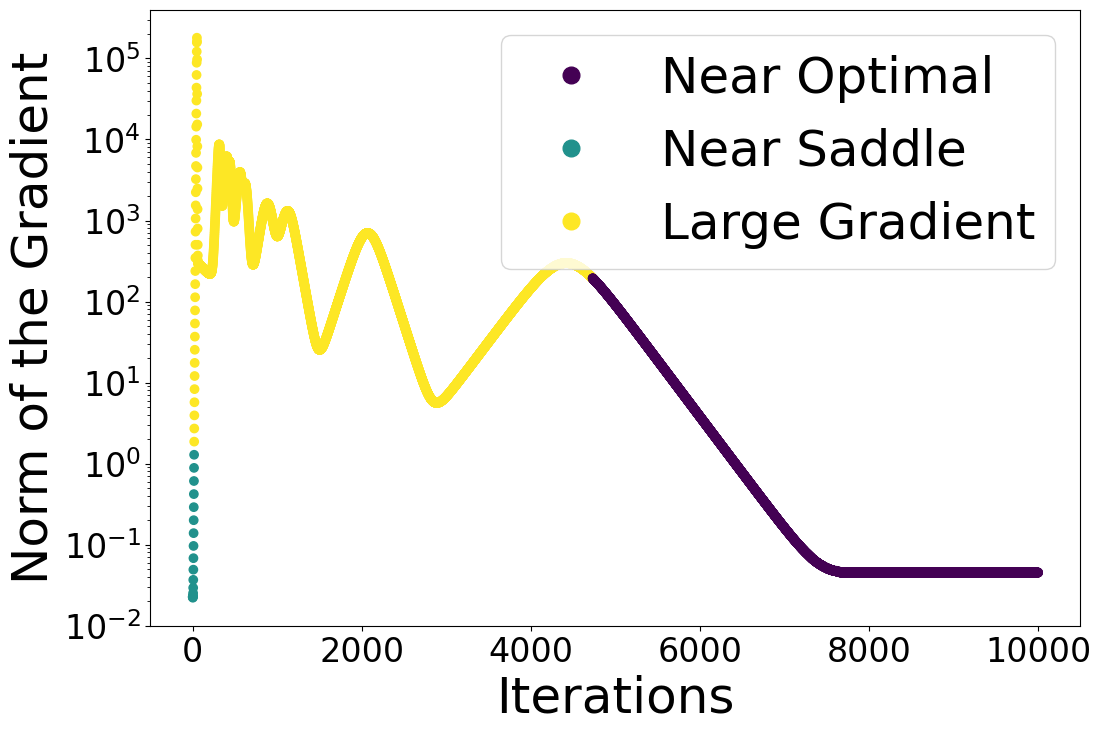}}
    \subfloat[Initialized Near Saddle \label{fig:ambient-saddle-distance}]{\includegraphics[width=0.33\linewidth]{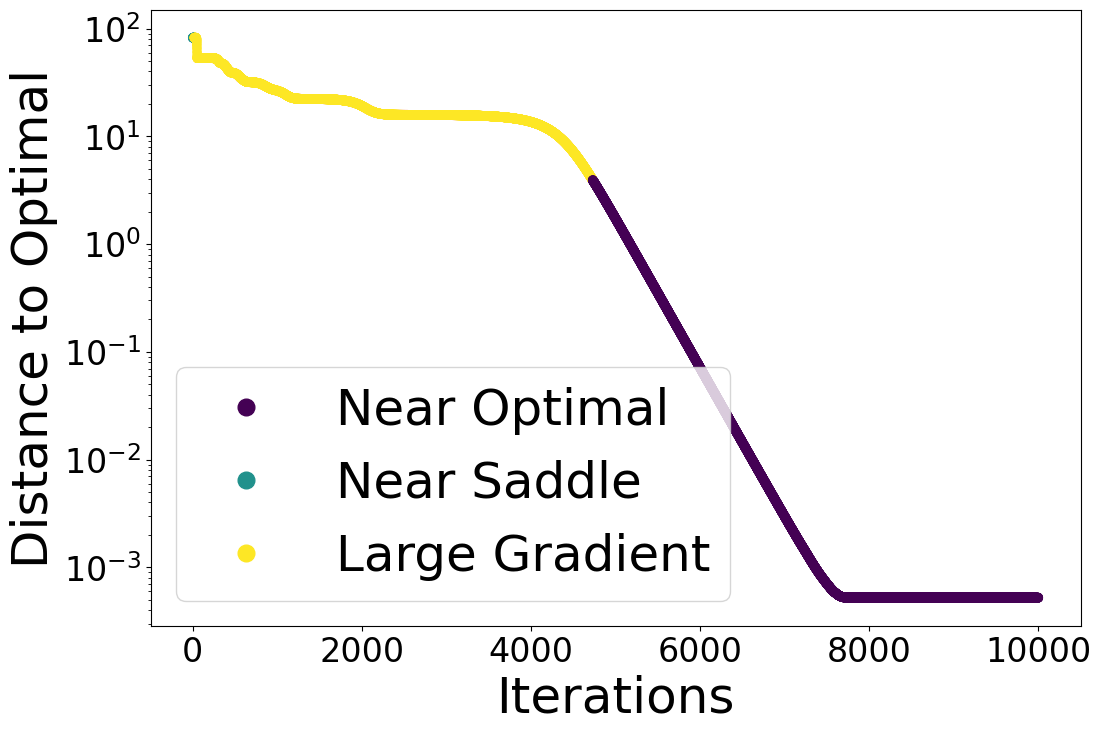}}
    \caption{Norms of the gradients for the ambient problem and the distance to the optimal solution. In (a), $\my$ is initialized near the global optimal solution, and in (b) $\my$ is initialized near a saddle point. c) shows the distance between $\my$ and the optimal solution for the case when it was initialized near a saddle point.}
    \label{fig:ambient}
\end{figure}

While our answer to question \ref{question:1} addresses the existence of a neural network approximating the spectrum of $\tx$, it does not provide a \textit{constructive} way to find one such approximation. We thus address question \ref{question:2} and prove that an approximating NN can be constructed by solving the optimization problem \ref{eq:Appr}, i.e., by finding a global minimizer of SNN's objective function. A precise statement can be found in Theorem \ref{thm:eigenappr}. To prove this theorem, we rely on our estimates in Theorem \ref{thm-appr:main} and on some auxiliary computations involving a global optimizer $\my^*$ of the ``ambient space problem":
\begin{align}\label{eq:optimization}
    \min_{\my \in \R^{n \times r}}\,\, \ell(\my).
\end{align}
For that we also make use of property \ref{prop:property1} in Proposition \ref{prop:An positive definite}, which allows us to guarantee, thanks to the Eckart–Young–Mirsky theorem (see \cite{eckart1936approximation} ), that solutions $\my$ to \eqref{eq:optimization} coincide, up to multiplication on the right by a $r\times r$ orthogonal matrix, with a $n \times r$ matrix $\my^*$ whose columns are scaled versions of the top $r$ normalized eigenvectors of the matrix $\tx$; see a detailed description of $\my^*$ in Appendix \ref{app:Y*}.




After discussing our spectral approximation results, we move on to discussing question \ref{question:3}, which is related to the hardness of optimization problem \ref{eq:Appr}. Notice that, while $\my_{\theta^*}$ is a good approximator for $\tx$'s spectrum according to our theory, it is unclear whether $\theta^*$ can be reached through a standard training scheme. In fact, question \ref{question:3}, as stated, is a challenging problem. This is not only due to the non-linearities in the neural network, but also because, in contrast to more standard theoretical studies of training dynamics of over-parameterized NNs (e.g., \cite{ChizatBach18,wojtowytsch2020convergence}), the spectral contrastive loss function $\ell$ is non-convex in the ``ambient space" variable $\my$. Despite this additional difficulty, numerical experiments —see Figure \ref{fig:illustration} for an illustration— suggest that first order optimization methods can find global solutions to \eqref{eq:Appr}, and our goal here is to take a first step in the objective of understanding this behavior mathematically.

To begin, we present some numerical experiments where we consider different initializations for the training of SNN. Here we take 100 data points from MNIST and let $\mathcal{A}_n$ be the $n \times n$ gram matrix for the data points for simplicity. We remark that while we care about a $\mathcal{A}_n$ with a specific form for our approximation theory results, our analysis of the loss landscape described below holds for an arbitrary positive semi-definite matrix. In Figure \ref{fig:NN}, we plot the norm of the gradient during training when initialized in two different regions of parameter space. Concretely, in a region of parameters for which $\my_\theta$ is close to a solution $\my^*$ to problem \ref{eq:optimization} and a region of parameters for which $\my_\theta$ is close to a saddle point of the ambient loss $\ell$. We compare these plots to the ones we produce from the gradient descent dynamics for the ambient problem \ref{eq:optimization}, which are shown in Figure \ref{fig:ambient}. We notice a similar qualitative behavior with the training dynamics of the NN, suggesting that the landscape of problem \ref{eq:Appr}, if the NN is properly overparameterized, inherits properties of the landscape of $\ell$. 

Motivated by the previous observation, in section \ref{sec:Riemannian Optimization} we provide a careful landscape analysis of the loss function $\ell$ introduced in \eqref{eq:Appr}. We deem this landscape to be ``benign", in the sense that it can be fully covered by the union of three regions described as follows: 1) the region of points close to global optimizers of \eqref{eq:optimization}, where one can prove (Riemannian) strong convexity under a suitable quotient geometry; 2) the region of points close to saddle points, where one can find escape directions; and, finally, 3) the region where the gradient of $\ell$ is large. Points in these regions are illustrated in Figures \ref{fig:ambient-saddle} and \ref{fig:ambient-saddle-distance}. The relevance of this global landscape characterization is that it implies convergence of most first-order optimization methods, or slight modifications thereof, toward global minimizers of the ambient space problem \ref{eq:optimization}. This characterization is suggestive of analogous properties for the NN training problem in an overparameterized regime, but a full theoretical analysis of this is left as an open problem.


\nc

%

In summary, the main contributions of our work are the following:
\begin{itemize}[nosep]
    \item We show that we can approximate the eigenvectors of a large adjacency matrix with a NN, provided that the NN has sufficiently many neurons; see Theorem \ref{thm-appr:main}. Moreover, we show that by solving \ref{eq:Appr} one can \textit{construct} such approximation provided the parameter space of the NN is rich enough; see Theorem \ref{thm:eigenappr}.
    \item We provide precise error bounds for the approximation of eigenfunctions of a Laplace-Beltrami operator with NNs; see Corollary \ref{cor:universalappmanifold}. In this way, we present an example of a setting where we can rigorously quantify the error of approximation of a solution to a PDE on a manifold with NNs.
    \item Motivated by numerical evidence, we begin an exploration of the optimization landscape of SNN and in particular provide a full description of SNN's associated ambient space optimization landscape. This landscape is shown to be benign; see discussion in Section \ref{sec:Riemannian Optimization}.
\end{itemize}

\subsection{Related work}\label{sec:literature}

\paragraph{Spectral clustering and manifold learning} Several works have attempted to establish precise mathematical connections between the spectra of graph Laplacian operators over proximity graphs and the spectrum of weighted Laplace-Beltrami operators over manifolds. Some examples include \cite{Shi2015Convergence,Dmitri2014Graph,trillos2018error,Lu2019GraphAT,calder2022improved,calder2022lipschitz,Dunson2021Spectral,Wormell2021Spectral}. In this paper we use adaptations of the results in \cite{calder2022lipschitz} to infer that, with very high probability, the eigenvectors of the normalized graph Laplacian matrix $\Delta_n$ defined in \eqref{eqn:GraphLaplacian} are essentially Lipschitz continuous functions. These regularity estimates are one of the crucial tools for proving our Theorem \ref{thm-appr:main}.

\paragraph{Contrastive Learning} Contrastive learning is a self-supervised learning technique that has gained considerable attention in recent years due to its success in computer vision, natural language processing, and speech recognition \cite{chen2020simple,chen2020big,chen2020improved,he2020momentum}. Theoretical properties of contrastive representation learning were first studied by \cite{arora2019theoretical,Tosh2021contrastive,lee2021predicting} where they assumed conditional independence. \cite{haochen2021provable} relaxes the conditional independence assumption by imposing the manifold assumption. With the spectral contrastive loss \eqref{eq:Appr} crucially in use, \cite{haochen2021provable} provides an error bound for downstream tasks. In this work, we analyze how the neural network can approximate and optimize the spectral loss function \eqref{eq:Appr}, which is the pertaining step of \cite{haochen2021provable}. 

\paragraph{Neural Network Approximations.} Given a function $f$ with certain amount of regularity, many works have studied the tradeoff between width, depth, and total number of neurons needed and the approximation \cite{petersen2020neural, lu2021deep}. Specifically, \cite{SHEN201974} looks at the problem Holder continuous functions on the unit cube, \cite{DBLP:conf/colt/Yarotsky18, Shen2019DeepNA} for continuous functions on the unit cube, and \cite{petersen2020neural, schmidt2019deep, haochen2021provable} consider the case when the function is defined on a manifold. A related area is that of neural network memorization of a finite number of data points \cite{NEURIPS2019_dbea3d0e}. In this paper, we use these results to show that for our specific type of regularity, we can prove similar results. 
\paragraph{Neural Networks and Partial Differential Equations} \cite{raissi2019physics} introduced Physics Inspired Neural Networks as a method for solving PDEs using neural networks. Specifically, \cite{Weinan2017TheDR, bhatnagar2019prediction, raissi2019physics} use neural networks to parameterize the solution as use the PDE as the loss function. Other works such as \cite{10.1145/2939672.2939738,zhu2018bayesian,adler2017solving,bhatnagar2019prediction} use neural networks to parameterize the solution operator on a given mesh on the domain. Finally, we have that eigenfunctions of operators on function spaces have a deep connection to PDEs. Recent works such as \cite{kovachki2021neural, li2020fourier, li2020neural} demonstrate how to learn these operators. In this work we show that we can approximate eigenfunctions to a weighted Laplace-Beltrami operator using neural networks.  

\paragraph{Shallow Linear Networks and Non-convex Optimization in Linear Algebra Problems} One of the main objects of study is the ambient problem Equation \ref{eq:Appr}. This formulation of the problem is related to linear networks. Linear networks are neural networks with identity activation. A variety of prior works have studied many different aspects of shallow linear networks such as the loss landscape and optimization dynamics \cite{BALDI198953, pmlr-v139-tarmoun21a, pmlr-v139-min21c, brechet2023critical}, and generalization for one layer networks \cite{Dobriban2015HighDimensionalAO, Hastie2019SurprisesIH, Bartlett2020BenignOI, kausik2023generalization}. Of relevance are also other works in the literature studying optimization problems very closely related to \eqref{eq:optimization}. For example, in Section 3 in \cite{NonconvexGeometryLowRank}, there is a landscape analysis for problem \ref{eq:optimization} when the matrix $\tx$ is assumed to have rank smaller than or equal to $r$. That setting is typically referred to as overparameterized or exactly parameterized, whereas here our focus is on the underparameterized setting. On the other hand, the case studied in section 3 in \cite{OverviewNonConvexandLowRank} is the simplest case we could consider for our problem and corresponds to $r=1$. In this simpler case, the non-convexity of the objective is completely due to a sign ambiguity, which makes the analysis more straightforward and the need to introduce quotient geometries less pressing.

\section{Spectral Approximation with neural networks}\label{sec:universal approx}

Through this section we make the following assumption on the generation process of the data $\X_n$.

\begin{assumption}\label{assum:data}
The points $x_{1}, \ldots, x_{n}$ are assumed to be sampled from a distribution supported on an $m$-dimensional manifold $\M$ that is assumed to be smooth, compact, orientable, connected, and without a boundary.  We assume that this sampling distribution has a smooth density $\rho : \M \rightarrow \R_+$ with respect to $\M's$ volume form, and assume that $\rho$ is bounded away from zero and also bounded above by a constant. 

\end{assumption}

\subsection{Spectral approximation with multilayer ReLU NNs}\label{sec:universal appr:main result}

\begin{restatable}[Spectral approximation of normalized Laplacians with neural networks]{thm}{UniversalAppr}\label{thm-appr:main}
Let $r \in \N$ be fixed. Under Assumptions \ref{assum:data}, there are constants $c, C$ that depend on $\M, \rho$, and the embedding dimension $r$, such that, with probability at least  
    \[
        1-C \varepsilon^{-6m} \exp \left(-c n \varepsilon^{m+4}\right)
    \]
    for every $\delta \in (0,1)$ there are $\kappa,L,p, N$ and a ReLU neural network $f_\theta \in \mathcal{F}(r, \kappa, L, p, N)$ (defined in Equation~\ref{eq-def:F class}), such that:
\begin{enumerate}[nosep,leftmargin=*]
    \item $\sqrt{n}\| \my_\theta -\mys\|_{\infty,\infty} \leq C(\delta+ \eps^2)$, and thus also $ \lVert \my_\theta - \my^* \rVert_{\rF} \leq C\sqrt{r}(\delta + \veps^2)$ .
    \item The depth of the network, $L$, satisfies: $L \le C\left(\log \frac{1}{\delta}+\log d\right)$, and its width, $p$, satisfies $p \le C\left(\delta^{-m}+d\right)$.
    \item The number of neurons of the network, $N$, satisfies: $N \le C r\left(\delta^{-m} \log \frac{1}{\delta}+d \log \frac{1}{\delta}+d \log d\right)$, and the range of weights, $\kappa$, satisfies $\kappa \le \frac{C}{n^{1/(2L)}}$. 
\end{enumerate}
\end{restatable}

Theorem \ref{thm-appr:main} uses regularity properties of graph Laplacian eigenvectors and a NN approximation theory result for functions on manifolds. A summary of important auxiliary results needed to prove Theorem \ref{thm-appr:main} is presented in Appendix \ref{app:AuxApproxResults} and the proof of the theorem itself is presented in Appendix \ref{appen:universal approximation:main result}.


\begin{remark}
    Any improvement of the approximations estimates in \cite{chen2019nonparametric} can be immediately applied to improve Theorem \ref{thm-appr:main}. We have relied on the results in \cite{chen2019nonparametric} due to the fact that in their estimates the ambient space dimension $d$ does not appear in any exponent.
\end{remark}

So far we have discussed approximations of the eigenvectors of $\tx$ (and thus also of $\Delta_n$) with neural networks, but more can be said about generalization of these NNs. In particular, the NN in our proof of Theorem \ref{thm-appr:main} can be shown to approximate eigenfunctions of the weighted Laplace-Beltrami operator $\Delta_\rho$ defined in Appendix \ref{Append:Auxiliary Lemma:Manifold learning}. Precisely, we have the following result.


\begin{restatable}{corollary}{universalappmanifold}
\label{cor:universalappmanifold}
    Under the same setting, notation, and assumptions as in Theorem \ref{thm-appr:main}, the neural network $f_\theta: \R^d \rightarrow \R^r$ can be chosen to satisfy
\[ \left \lVert \sqrt{\frac{n}{1+a}} f_\theta^i - f_i \right \rVert_{L^\infty(\M)} \leq C( \delta + \veps ), \quad \forall i=1, \dots, r. \]
 In the above, $f_\theta^1, \dots, f_\theta^r$ are the coordinate functions of the vector-valued neural network $f_\theta$, and the functions $f_1, \dots, f_r$ are the normalized eigenfunctions of the Laplace-Beltrami operator $\Delta_\rho$ that are associated to $\Delta_\rho$'s $r$ smallest eigenvalues.  
\end{restatable}

\begin{remark}
The $\veps^2$ term that appears in the bound for $ \lVert \my_\theta - \my^* \rVert_{\rF} $ in Theorem \ref{thm-appr:main} cannot be obtained simply from convergence of eigenvectors of $\Delta_n$ toward eigenfunctions of $\Delta_\rho$ in $L^\infty$. It turns out that we need to use a  stronger notion of convergence (almost $C^{0,1}$) that in particular implies sharper regularity estimates for eigenvectors of $\Delta_n$ (see Corollary \ref{cor:RegularityEigenvectors} in Appendix \ref{Append:Auxiliary Lemma:Manifold learning} and Remark \ref{rem:AboutEpsSq} below it). In turn, the sharper $\veps^2$ term is essential for our proof of Theorem \ref{thm:eigenappr} below to work; see the discussion starting in Remark \ref{rem:AboutEpsSq}. 
\end{remark}


\subsection{Spectral approximation with global minimizers of SNN's objective}
\label{sec:SpectralApproxSNN}
After discussing the \textit{existence} of approximating NNs, we turn our attention to \textit{constructive} ways to approximate $\my^*$ using neural networks. We give a precise answer to question \ref{question:2}.


\begin{restatable}[Optimizing SNN approximates eigenvectors up to rotation]{thm}{Eigenappr}\label{thm:eigenappr}
Let $r \in \N$ be fixed and suppose that $\Delta_\rho$ is such that $\Delta_\rho$ has a spectral gap between its $r$ and $r+1$ smallest eigenvalues, i.e., in the notation in Appendix \ref{Append:Auxiliary Lemma:Manifold learning}, assume that $ \sigma_r^\M < \sigma_{r+1}^\M.$
   For given $\kappa,L, p, N$ (to be chosen below), let $f_{\theta^*}$ be such that $ f_{\theta^*} \in \argmin_{f_\theta \in \mathcal{F}(r,\kappa, L, p, N )} \lVert \my_\theta \my^\top_\theta  - \tx \rVert_\rF^2$. 
   
   Under Assumptions \ref{assum:data}, there are constants $c, C$ that depend on $\M, \rho$, and the embedding dimension $r$, such that, with probability at least $
        1-C \varepsilon^{-6m} \exp \left(-c n \varepsilon^{m+4}\right),$
    for every $\tilde \delta \in (0,c)$ (i.e., $\tilde{\delta}$ sufficiently small) and for $ \kappa = \frac{C}{n^{1/(2L)}}$ , $L= C\left(\log \frac{1}{\tilde \delta \veps}+\log d\right)$, $p = C\left((\tilde \delta \veps)^{-m}+d\right)$ and $N=\infty$, 
we have
%
 %
  %
   %
    %
    \begin{equation}
    \begin{split}
        \min_{\mO \in \mathbb{O}_r}\lVert \my_{\theta^*} - \my^* \mO \rVert_\rF\le C \eps(\tilde \delta + \veps).
    \end{split}
    \label{eq:UpTorotationConvergence}
    \end{equation}
\end{restatable}

\begin{remark}
Equation \ref{eq:UpTorotationConvergence} says that $\my_{\theta^*}$ approximates a minimizer of the ambient problem \ref{eq:optimization} and that $\my_{\theta^*}$ can be recovered but only up to rotation. This is unavoidable, since the loss function $\ell$ is invariant under multiplication on the right by a $r \times r$ orthogonal matrix. On the other hand, to set $N=\infty$ means we do not enforce sparsity constraints in the optimization of the NN parameters. This is convenient in practical settings and this is the reason why we state the theorem in this way. However, we can also set $N= r\left((\tilde \delta \veps)^{-m} \log \frac{1}{\tilde \delta \veps}+d \log \frac{1}{\tilde \delta \veps}+d \log d \right)$ without affecting the conclusion of the theorem. 
\end{remark}

%

\section{Landscape of SNN's Ambient Optimization Problem}\label{sec:Riemannian Optimization}

While in prior sections we considered a specific $\mathcal{A}_n$, the analysis in this section 
only relies on $\tx$ being positive definite with an eigengap between its $r$-th and $(r+1)$th top eigenvalues.
We analyze the global optimization landscape of the non-convex Problem~\ref{eq:optimization} under a suitable Riemannian \textit{quotient geometry} \cite{absil2009optimization,boumal_2023}. The need for a quotient geometry comes from the fact that if $\my$ is a stationary point of \ref{eq:optimization}, then $\my\mO$ is also a stationary point for any $r \times r$ orthogonal matrix $\mO\in \mathbb{O}_r$. This implies that the loss function $\ell$ is non-convex in any neighborhood of a stationary point (\citealp[Proposition 2]{li2019nonconvex}).  Despite the non-convexity of $\ell$, we show that under this geometry, Equation~\ref{eq:optimization} is geodesically convex in a local neighborhood around the optimal solution. 

Let $\overline{\NM}^n_{r+}$ be the space of $n \times r$ matrices with full column rank. 
To define the quotient manifold, we encode the invariance mapping, i.e., $\my \to \my\mO$, by defining the equivalence classes $[\my]=\{\my\mO:\mO\in\mathbb{O}_r\}$. From \cite{lee2018introduction}, we have $\NM_{r_+}^n\defeq \overline{\NM}_{r_+}^n / \mathbb{O}_r$ is a quotient manifold of $\overline{\NM}_{r+}^n$. See a detailed introduction to Riemannian optimization in \cite{boumal_2023}.
Since the loss function in \ref{eq:optimization} is invariant along the equivalence classes of $\overline{\NM}_{r_+}^n$, $\ell$ induces the following optimization problem on the quotient manifold $\NM_{r_+}^n$:
\begin{equation}\label{eq:optimizationRiemannian}
    \min _{[\mathbf{Y}] \in \NM_{r_+}^n} H([\mathbf{Y}])\defeq \frac{1}{2}\left\|\mathbf{Y} \mathbf{Y}^{\top}-\tx\right\|_{\mathrm{F}}^2
\end{equation}
To analyze the landscape for Equation~\ref{eq:optimizationRiemannian}, we need expressions for the Riemannian gradient, the Riemannian Hessian, as well as the geodesic distance $d$ on this quotient manifold. By Lemma 2 from \cite{luo2022nonconvex}, we have that
\[
    d\left(\left[\mathbf{Y}_1\right],\left[\mathbf{Y}_2\right]\right) = \min _{\mathbf{Q} \in \mathbb{O}_r}\left\|\mathbf{Y}_2 \mathbf{Q}-\mathbf{Y}_1\right\|_{\mathrm{F}}
\]
and from Lemma 3 from \cite{luo2022nonconvex}, we have that
\begin{equation}\label{eq-def:gradHessian in text}
    \begin{split}
        \overline{\operatorname{grad} H([\mathbf{Y}])}&=2\left(\mathbf{Y} \mathbf{Y}^{\top}-\tx\right) \mathbf{Y},\\
        \overline{\operatorname{Hess} H([\mathbf{Y}])}\left[\theta_{\mathbf{Y}}, \theta_{\mathbf{Y}}\right]&=\left\|\mathbf{Y} \theta_{\mathbf{Y}}^{\top}+\theta_{\mathbf{Y}} \mathbf{Y}^{\top}\right\|_{\mathrm{F}}^2+2\left\langle\mathbf{Y} \mathbf{Y}^{\top}-\tx, \theta_{\mathbf{Y}} \theta_{\mathbf{Y}}^{\top}\right\rangle.
    \end{split}
\end{equation}
Finally, by the classical theory on low-rank approximation (Eckart–Young–Mirsky theorem \cite{eckart1936approximation}), $\left[\mathbf{Y}^*\right]$ is the unique global minimizer of Equation~\ref{eq:optimizationRiemannian}. 
Let $\kappa^*=\sigma_1\left(\mathbf{Y}^*\right) / \sigma_r\left(\mathbf{Y}^*\right)$ be the condition number of $\mathbf{Y}^*$. Here,  $\sigma_i(A)$ is the $i^{\mathrm{th}}$ largest singular value of $A$, and $\|A\|=\sigma_1(A)$ is its spectral norm. Our precise assumption on the matrix $\mathcal{A}_n$ for this section is as follows.   


\begin{assumption}[Eigengap]\label{assum-riemannian:eigengap in text}
    $\sigma_{r+1}(\tx)$ is strictly smaller than $\sigma_r(\tx)$.
\end{assumption}



Let $\mu, \alpha, \beta, \gamma \geqslant 0$. We then split the landscape of $H([\mathbf{Y}])$ into the following five regions (not necessarily non-overlapping). 
\begin{equation}\label{eq-def:R in text}
    \begin{split}
        &\mathcal{R}_1\defeq \left\{\mathbf{Y} \in \mathbb{R}_*^{n \times r} \middle| d\left([\mathbf{Y}],\left[\mathbf{Y}^*\right]\right) \leqslant \mu \sigma_r\left(\mathbf{Y}^*\right) / \kappa^*\right\}, \\
        &\mathcal{R}_2\defeq \left\{\mathbf{Y} \in \mathbb{R}_*^{n \times r} \middle| 
        \begin{array}{l}
            d\left([\mathbf{Y}],\left[\mathbf{Y}^*\right]\right)>\mu \sigma_r\left(\mathbf{Y}^*\right) / \kappa^*,\|\overline{\operatorname{grad} H([\mathbf{Y}])}\|_{\mathrm{F}} \leqslant \alpha \mu \sigma_r^3\left(\mathbf{Y}^*\right) /\left(4 \kappa^*\right), \\
            \|\mathbf{Y}\| \leqslant \beta\left\|\mathbf{Y}^*\right\|,\left\|\my\mathbf{Y}^{\top}\right\|_{\mathrm{F}} \leqslant \gamma\left\|\mathbf{Y}^* \mathbf{Y}^{* \top}\right\|_{\mathrm{F}}
        \end{array}\right\}, \\
        &\mathcal{R}_3^{\prime}\defeq \left\{\mathbf{Y} \in \mathbb{R}_*^{n \times r} \middle| 
        \begin{array}{l}
            \|\overline{\operatorname{grad} H([\mathbf{Y}])}\|_{\mathrm{F}}>\alpha \mu \sigma_r^3\left(\mathbf{Y}^*\right) /\left(4 \kappa^*\right),\|\mathbf{Y}\| \leqslant \beta\left\|\mathbf{Y}^*\right\|,\\ \left\|\my\mathbf{Y}^{\top}\right\|_{\mathrm{F}} \leqslant \gamma\left\|\mathbf{Y}^* \mathbf{Y}^{* \top}\right\|_{\mathrm{F}}
        \end{array}\right\}, \\
        &\mathcal{R}_3^{\prime \prime}\defeq \left\{\mathbf{Y} \in \mathbb{R}_*^{n \times r} \middle| \| \mathbf{Y}\|>\beta\| \mathbf{Y}^*\|,\| \mathbf{Y} \mathbf{Y}^{\top}\left\|_{\mathrm{F}} \leqslant \gamma\right\| \mathbf{Y}^* \mathbf{Y}^{* \top} \|_{\mathrm{F}}\right\}, \\
        &\mathcal{R}_3^{\prime \prime \prime}\defeq \left\{\mathbf{Y} \in \mathbb{R}_*^{n \times r} \middle| \|\mathbf{Y} \mathbf{Y}^{\top}\right\|_{\mathrm{F}}>\gamma\left\|\mathbf{Y}^* \mathbf{Y}^{* \top}\right\|_{\rF}\},
    \end{split}
\end{equation}

We show that for small values of $\mu$, the \emph{loss function is geodesically convex} in $\mathcal{R}_1$. $\mathcal{R}_2$ is then defined as the region outside of $\mathcal{R}_1$ such that the Riemannian gradient is small relative to $\mu$. Hence this is the region in which we are close to the saddle points. We show that for this region there is \emph{always an escape direction} (i.e., directions where the Hessian is strictly negative). Finally, $\mathcal{R}_3^\prime$, $\mathcal{R}_3^{\prime \prime}$, and $\mathcal{R}_3^{\prime \prime \prime}$ are the remaining regions. We show that the \emph{Riemannian gradient is large} (relative to $\mu$) in these regions. Finally, it is easy to see that $\mathcal{R}_1 \bigcup \mathcal{R}_2 \bigcup \mathcal{R}_3^{\prime} \cup \mathcal{R}_3^{\prime \prime} \bigcup \mathcal{R}_3^{\prime \prime \prime}= \mathbb{R}_*^{n \times r}$.

We are now ready to state the first of our main results from this section.

\begin{restatable}[Local Geodesic Strong Convexity and Smoothness of Equation~\ref{eq:optimizationRiemannian}]{thm}{geodesicallyConvex}\label{thm:R1intext}
Suppose $0 \leqslant \mu \leqslant \kappa^* / 3$. Given that Assumption \ref{assum-riemannian:eigengap in text} holds, for any $\mathbf{Y} \in \mathcal{R}_1$ defined in Equation~\ref{eq-def:R in text}. 
\[
    \begin{aligned}
        &\sigma_{\min }(\overline{\operatorname{Hess} H([\mathbf{Y}])}) \geqslant \left(2\left(1-\mu / \kappa^*\right)^2-(14 / 3) \mu\right) \sigma_r\left(\tx\right)-2\sigma_{r+1}(\tx), \\
        & \sigma_{\max }(\overline{\operatorname{Hess} H([\mathbf{Y}])}) \leqslant 4\left(\sigma_1\left(\mathbf{Y}^*\right)+\mu \sigma_r\left(\mathbf{Y}^*\right) / \kappa^*\right)^2+14 \mu \sigma_r^2\left(\mathbf{Y}^*\right) / 3
    \end{aligned}
\]
In particular, if $\mu$ is further chosen such that $\left(2\left(1-\mu / \kappa^*\right)^2-(14 / 3) \mu\right) \sigma_r\left(\tx\right)-2\sigma_{r+1}(\tx)>0$, we have $H([\mathbf{Y}])$ is geodesically strongly convex and smooth in $\mathcal{R}_1$.
\end{restatable}

Theorem \ref{thm:R1intext} guarantees that the optimization problem Equation~\ref{eq:optimizationRiemannian} is geodesically strongly convex and smooth in a neighborhood of $[\mys]$. It also shows that if $\my$ is close to the global minimizer, then Riemannian gradient descent converges to the global minimizer of the quotient space linearly.

Next, to analyze $\mathcal{R}_2$, we need to understand the other first-order stationary points (FOSP). 


\begin{restatable}
[FOSP of Equation~\ref{eq:optimizationRiemannian}]{thm}{fosp}\label{Thm:FOSP in text} Let $\mathbf{Y}^* = \overline{\mU}\cdot \overline{\mLambda}\cdot\overline{\mV}^\top$ and $\my=\mU\mD\mV^\top$ be the SVDs. Then for any $S$ subset of $[n]$, we have that $\left[\overline{\mU}_S\mLambda_S\mV_S^\top\right]$ is a Riemannian FOSPs of Equation~\ref{eq:optimizationRiemannian}. Further, these are the only Riemannian FOSPs. 
\end{restatable}

Theorem \ref{Thm:FOSP in text} shows that the linear combinations of eigenvectors can be used to construct Riemannian first-order stationary points (FOSP) of Equation~\ref{eq:optimizationRiemannian}. This theorem also shows that there are many FOSPs of Equation~\ref{eq-def:gradHessian in text}. This is quite different from the regime studied in \cite{luo2022nonconvex}. In general, gradient descent is known to converge to a FOSP. Hence one might expect that if we initialized near one of the saddle points, then we might converge to that saddle point. However, our next main result of the section shows that even if we initialize near the saddle, there always exist escape directions. 

\begin{restatable}[Escape Directions]{thm}{escapeDirection} \label{thm-R2:Region with Negative Eigenvalue in the Riemannian Hessian in text}
Assume that Assumption \ref{assum-riemannian:eigengap in text} holds. Then for sufficiently small $\alpha$ and any $\my \in \mathcal{R}_2$ that is not an FOSP, there exists $C_1(\tx)>0$ and $\theta_\my$ such that 
\begin{equation*}
    \begin{split}
        \overline{\operatorname{Hess} H([\mathbf{Y}])}\left[\theta_{\mathbf{Y}}, \theta_{\mathbf{Y}}\right] \leqslant
        &-C_1(\tx)\left\|\theta_{\mathbf{Y}}\right\|_{\mathrm{F}}^2.
    \end{split}
\end{equation*}
\end{restatable}
In particular, it is possible to exactly quantify the size of $\alpha$ and the explicitly construct the escape direction $\theta_\my$. See Theorem \ref{thm-R2:Region with Negative Eigenvalue in the Riemannian Hessian} in the appendix for more details. 

\begin{remark}
    Theorem \ref{thm-R2:Region with Negative Eigenvalue in the Riemannian Hessian in text} guarantees that if $\my$ is close to some saddle points, then $\theta_\my$ will make its escape from the saddle point linearly.
\end{remark}


Finally, the next result tells that is we are not close to a FOSP, then we have large gradients. 

\begin{restatable}[(Regions with Large Riemannian Gradient of
\eqref{eq:optimization}]{thm}{largeGradient}\label{thm:R3}
$ $
\begin{enumerate}[nosep]
    \item $\|\overline{\operatorname{grad} H([\mathbf{Y}])}\|_{\mathrm{F}} > \alpha \mu \sigma_r^3\left(\mathbf{Y}^*\right) /\left(4 \kappa^*\right),$ $\forall \mathbf{Y} \in \mathcal{R}_3^{\prime}$;
    \item $\|\overline{\operatorname{grad} H([\mathbf{Y}])}\|_{\rF} \geqslant 2\left(\|\mathbf{Y}\|^3-\|\mathbf{Y}\|\left\|\mathbf{Y}^*\right\|^2\right)>2\left(\beta^3-\beta\right)\left\|\mathbf{Y}^*\right\|^3, \quad \forall \mathbf{Y} \in \mathcal{R}_3^{\prime \prime}$;
    \item $\langle\overline{\operatorname{grad} H([\mathbf{Y}])}, \mathbf{Y}\rangle>2(1-1 / \gamma)\left\|\mathbf{Y} \mathbf{Y}^{\top}\right\|_{\mathrm{F}}^2, \quad \forall \mathbf{Y} \in \mathcal{R}_3^{\prime \prime \prime}$.
\end{enumerate}
In particular, if $\beta>1$ and $\gamma>1$, we have the Riemannian gradient of $H([\mathbf{Y}])$ has large magnitude in all regions $\mathcal{R}_3^{\prime}, \mathcal{R}_3^{\prime \prime}$ and $\mathcal{R}_3^{\prime \prime \prime}$.
\end{restatable}


\begin{remark}
    These results can be seen as an under-parameterized generalization to the regression problem of Section 5 in \cite{luo2022nonconvex}. The proof in \cite{luo2022nonconvex} is simpler because in their setting there are no saddle points or local minima that are not global. Conceptually, \cite{Salma2021OverparametrizedLinear} proves that in the setting $r \ge n$, the gradient flow for \eqref{eq:optimization} converges to a global minimum linearly. We complement this result by studying the case $r < n$.

\end{remark}

\begin{remark}
    To demonstrate strong geodesic convexity, the eigengap assumption is necessary as it prevents multiple global solutions. However, it is possible to relax this assumption and instead deduce a PL condition, which would also imply a linear convergence rate for a first-order method.
\end{remark}

\begin{remark}
    In the specific case of $\tx$ as in \eqref{eq-def:Graph Laplacian}, and under Assumptions \ref{assum:data}, Assumption \ref{assum-riemannian:eigengap in text} should be interpreted as $\sigma_r^\M < \sigma_{r+1}^\M$, as suggested by Remark \ref{rem:OrderEigenvaluesAn}. Also, $\mu$ must be taken to be in the order $\veps^2$. The scale $\veps^2$ is actually a natural scale for this problem, since, as discussed in Remark \ref{rem:SaddleEnergyGap}, the energy gap between saddle points and the global minimizer $[\my^*]$ is $O(\veps^2)$. 
\end{remark}

%
%
%
\section{Conclusions}\label{sec:discussion}
We have explored some theoretical aspects of Spectral Neural Networks (SNN), a framework that substitutes the use of traditional eigensolvers with suitable neural network parameter optimization. Our emphasis has been on approximation theory, specifically identifying the minimum number of neurons of a multilayer NN required to capture spectral geometric properties in data, and investigating the optimization landscape of SNN, even in the face of its non-convex ambient loss function.


For our approximation theory results we have assumed a specific proximity graph structure over data points that are sampled from a distribution over a smooth low-dimensional manifold. A natural future direction worth of study is the generalization of these results to settings where data points, and their similarity graph, are sampled from other generative models, e.g., as in the application to contrastive learning in \cite{haochen2021provable}. To carry out this generalization, an important first step is to study the regularity properties of eigenvectors of an adjacency matrix/graph Laplacian generated from other types of probabilistic models.

At a high level, our approximation theory results have sought to bridge the extensive body of research on graph-based learning methods, their ties to PDE theory on manifolds, and the approximation theory for neural networks. While our analysis has focused on eigenvalue problems, such as those involving graph Laplacians or Laplace Beltrami operators, we anticipate that this overarching objective can be extended to develop provably consistent methods for solving a larger class of PDEs on manifolds with neural networks. We believe this represents a significant and promising research avenue.

On the optimization front, we have focused on studying the landscape of the ambient space problem \ref{eq:optimization}. This has been done anticipating the use of our estimates in a future analysis of the training dynamics of SNN. We reiterate that the setting of interest here is different from other settings in the literature that study the dynamics of neural network training in an appropriate scaling limit —leading to either a neural tangent kernel (NTK) or to a mean field limit. This difference is mainly due to the fact that the spectral contrastive loss $\ell$ (see \ref{eq:Appr}) of SNN is non-convex, and even local strong convexity around a global minimizer does not hold in a standard sense and instead can only be guaranteed when considered under a suitable quotient geometry.

\bibliography{reference}
\bibliographystyle{agsm}

\newpage

\appendix



\section{Training of neural networks for spectral approximations}
\subsection{Training}\label{Appendix:Supplementary:Advantage of SNN}
Two of the main issues of standard eigensolvers are the need to store large matrices in memory and the need to redo computations from scratch if new data points are added. As mentioned, SNN can overcome this issue using mini-batch training. Specifically, the loss function $\ell(\my)$ can be written as,
\begin{align}
    \ell(\my_\theta)=\sum_{i=1}^n\sum_{j=1}^n \left((\tx)_{ij}-(\my_\theta\my_\theta^\top)_{ij})\right) = \sum_{i=1}^n\sum_{j=1}^n \Big((\tx)_{ij}-\big\langle f_\theta(x_i),f_\theta(x_j)\big\rangle\Big)^2
\end{align}
where $(\tx)_{ij}$ represents the $(i,j)$ entry of $\tx$ and $f_\theta$ is the neural network. Hence, in every iteration, one can randomly generate $1$ index $(i,j)$ from $[n]\times [n]$, compute the loss and gradient for that term in the summation, and then perform one iteration of gradient descent. 

\subsection{Other Training Approaches}\label{appendix:other SNN}

Besides SNN, there are two alternative ways of training spectral neural networks: \textit{Eigensolver Net} and \textit{SpectralNet} \cite{shaham2018spectralnet}. We compare these three different tools of neural network training and highlight the relative advantages and disadvantages of SNN. 

\textbf{Eigensolver Net:} Given the matrix $\Delta_n$, one option could be to compute the eigendecomposition of $\Delta_n$ using traditional eigensolvers to get eigenvectors $\mv_1, \ldots, \mv_r$. Then, to learn an eigenfunction (that is, the function that maps data points to the corresponding entries of an eigenvector), we can minimize the following $\ell_2$ loss:
\begin{align}\label{eq-def:eigensolverNet}
    \min_{\theta} \|f_{\theta}(\mathcal{X}_n)-\mv\|^2,
\end{align}
where $\mv=[\mv_1,\mv_2\dots,\mv_r]$ and $\mathcal{X}_n$ is the data. 



In general, the Eigensolver net is a natural way to extend to out-of-sample data and can be used to learn the eigenvector for matrices that are not PSD.
On the other hand, the Eigensolver net has some drawbacks. Specifically, one still needs to compute the eigendecomposition using traditional eigensolvers. 

\textbf{SpectralNet:}
SpectralNet aims at minimizing the \textit{SpectralNet loss},
\begin{align}\label{eq-def:SpectralNet Loss}
    \mathcal{L}_{\text {SpectralNet }}(\theta)=\frac{1}{n^2}\sum_{i=1}^n\sum_{j=1}^n  \eta\left(\frac{|x_i- x_j|}{\eps}\right)\left\|f_{\theta}(x_i)-f_\theta(x_j)\right\|^2\,
\end{align}
where $f_\theta:\R^d\to \R^r$ encodes the spectral embedding of $x_i$ while satisfying the constraint 
\begin{align}\label{eq-def:SpectralNet constraint}
\my_\theta^\top\my_\theta=n \I_r,
\end{align} 
where $\my_\theta=[f_\theta(x_1),\dots, f_\theta(x_n)]$. This constraint is used to avoid a trivial solution. 
%
Note that \eqref{eq-def:SpectralNet constraint} is a global constraint. \cite{shaham2018spectralnet} have established a stochastic coordinate descent fashion to efficiently train SpectralNets. 
However, the stochastic training process in \cite{shaham2018spectralnet} can only guarantee \eqref{eq-def:SpectralNet constraint} holds approximately. 

Conceptually, the SpectralNet loss \eqref{eq-def:SpectralNet Loss} can also be written as 
\begin{align}\label{eq-def:SpectralNet Dirichlet Energy}
    \mathcal{L}_{\text {SpectralNet }}(\theta)=\frac{2}{n^2} \operatorname{trace}\left(\my_\theta^\top\left( \mD_\mG-\mG \right) \my_\theta\right)
\end{align}
where $\mG\in\R^{n\times n}$ such that $\mG_{ij}=\eta\left(\frac{\|x_i-x_j\|}{\eps}\right)$, and $\mD_\mG$ is a diagonal matrix where $(\mD_\mG)_{ii}=\sum_{j=1}^n \mG_{ij}$. The symmetric and positive semi-definite matrix $\mD_\mG-\mG$ encodes the unnormalized graph Laplacian.
Since $\mD_\mG-\mG$ is positive semi-definite, the ambient problem of \eqref{eq-def:SpectralNet Dirichlet Energy} is a constrained convex optimization problem. However, the parametrization and hard constraint \ref{eq-def:SpectralNet constraint} make understanding SpectralNet's training process from a theoretical perspective challenging.

\section{Numerical Details}
\label{app:exp}

\subsection{For Eigenvector Illustration}
\label{app:exp1}

We sample 2000 data points $x_i$ uniformly from a 2-dimensional sphere embedded in $\R^3$, and then construct a $30$ nearest neighbor graph among these points. Figure \ref{fig:Illustration NN} shows a 1-hidden layer neural network evaluated at $x_i$, with 10000 hidden neurons to learn the first eigenvector of the graph Laplacian. The Network is trained for $5000$ epochs using the full batch \textit{Adam} in \textit{Pytorch} and a learning rate of $2* 10^{-5}$.

\subsection{Ambient vs Parameterized Problem}
\label{app:exp2}

We took 100 data points from MNIST. We normalized the pixel values to live in $[0,1]$ and then computed $\mathcal{A}_n$ as the gran matrix. 

The neural network has one hidden layer with a width of 1000. To initialize the neural network near a saddle point, we randomly pick a saddle point and then pretrain the network to approach this saddle. We used full batch gradient descent with an initial learning rate of 3e-6. We trained the network for 10000 iterations and used Cosine annealing as the learning rate scheduler. 

After pretraining the network, we trained the network with the true objective. We used full batch gradient descent with an initial learning rate of 3e-6. We trained the network for 10000 iterations and used Cosine annealing as the learning rate scheduler. 

When we initialized the network near the optimal solution, we followed the same procedure but pretrained the network for 1250 iterations. 

For the ambient problem, we used full batch gradient descent with a learning rate 3e-6. We trained the network for 5000 iterations and again used Cosine annealing for the learning rate scheduler. 

We remark that the sublinearity convergence rate in Figures \ref{fig:NN} and \ref{fig:ambient} is due to the step size decaying in the optimizer. In $\mathcal{R}_1$, $H([\my])$ has been shown to be strongly convex, so keeping the same step size should guarantee a linear rate. In this work, we don't focus on the optimization problem of SNN, but use this to illustrate Theorem \ref{thm:R1intext}, \ref{thm-R2:Region with Negative Eigenvalue in the Riemannian Hessian in text} and \ref{thm:R3}.







\section{Multi-layer ReLU neural networks}
\label{app:MultilayerReLU}

For concreteness, in this work we use multi-layer ReLU neural networks. To be precise, our neural networks are parameterized functions $f : \R^d \rightarrow \R^r$ of the form:
\begin{align}\label{eq-def:ReLU}
    f(\mathbf{x})=\mathbf{W}_{L} \cdot \operatorname{ReLU}\left(\mathbf{W}_{L-1} \cdots \operatorname{ReLU}\left(\mathbf{W}_{1} \mathbf{x}+\mathbf{b}_{1}\right) \cdots+\mathbf{b}_{L-1}\right)+\mathbf{b}_{L}, \quad \x \in \R^d.
\end{align}

More specifically, for a given choice of parameters $r, \kappa, L, p, N $ we will consider the family of functions:
\vspace{-0.2cm}
\begin{equation}
    \begin{aligned}
            \mathcal{F}(r, \kappa, L, p, N) &= \Biggl\{f \mid f(\mathbf{x}) \text{ has the form } \ref{eq-def:ReLU}, \text{ where: }
            \\ & \qquad \mathbf{W}_{l} \in \R^{p \times p} , \mathbf{b}_l \in \R^p \text{ for } l=2, \dots, L-1, 
            \\& \qquad \mathbf{W}_1 \in \R^{p \times d}, \mathbf{b}_1 \in \R^p, \mathbf{W}_L \in \R^{r\times p}, \mathbf{b}_L \in \R^r.
            \\ &\qquad  \left\|\mathbf{W}_{l}\right\|_{\infty, \infty} \leq \kappa,\left\|\mathbf{b}_{l}\right\|_{\infty} \leq \kappa \text { for } l=1, \ldots, L, \\
            &\left.\qquad \sum_{l=1}^{L}\left\|W_{l}\right\|_{0}+\left\|\mathbf{b}_{l}\right\|_{0} \leq N\right\}
            \label{eq-def:F class}
    \end{aligned}
\end{equation}
where $\|\cdot\|_{0}$ denotes the number of nonzero entries in a vector or a matrix, $\left\| \cdot\right\|_{\infty}$ denotes the $\ell_{\infty}$ norm of a vector. For a matrix $M$, we use $\|M\|_{\infty, \infty}=\max _{i, j}\left|M_{i j}\right|$.

For convenience, after specifying the quantities $r,\kappa, L, p, N$, we denote by $\Theta$ the space of admissible parameters $\theta = (\mathbf{W}_{1}, \mathbf{b}_{1}, \dots, \mathbf{W}_{L}, \mathbf{b}_{L} )$ in the function class $\mathcal{F}(r, \kappa, L, p, N)$, and we use $f_\theta$ to represent the function in \eqref{eq-def:ReLU}.

\section{Properties of the matrix $\tx$ in \eqref{eq:Appr}}
\label{app:PropertiesAn}

\subsection{Proof of Proposition \ref{prop:An positive definite} }

\begin{proof}[Proof of Proposition \ref{prop:An positive definite}]
Notice that 
\begin{equation}
   \tx = - \Delta_n + (a + 1) \I_n, 
   \label{eqn:AnDeltan}
\end{equation}
from where it follows that the eigenvectors of $\tx$ associated to its $r$ largest eigenvalues coincide with the eigenvectors of $\Delta_n$ associated to its $r$ smallest eigenvalues. Since $\tx$ is obviously symmetric, it remains to show that its eigenvalues are non-negative. In turn, from the definition of $\tx$ in \eqref{eq-def:Graph Laplacian} and the fact that $a>1$, it is sufficient to argue that all eigenvalues of $\mD_{\mG}^{-1/2} \mG \mD_{\mG}^{-1/2}$ have absolute value less than or equal to $1$. This, however, follows from the following two facts:  1) the matrix $\mD_{\mG}^{-1/2} \mG \mD_{\mG}^{-1/2}$ is similar to the matrix $\mD_{\mG}^{-1} \mG$, given that
\[ \mD_{\mG}^{1/2} ( \mD_{\mG}^{-1} \mG) \mD_{\mG}^{-1/2} = \mD_{\mG}^{-1/2} \mG \mD_{\mG}^{-1/2}, \]
implying that $ \mD_{\mG}^{-1/2} \mG \mD_{\mG}^{-1/2}$ and $\mD_{\mG}^{-1} \mG$ have the same eigenvalues; and 2) all the eigenvalues of  $\mD_{\mG}^{-1}\mG$ have norm less than one, since $\mD_{\mG}^{-1} \mG$ is a transition probability matrix. 
\end{proof}

\begin{remark}
\label{rem:AnVsDeltan}
While one could set $\tx$ to be $\Delta_n$ itself (since $\Delta_n$ is PSD), solving the resulting problem \ref{eq:optimization} would return the eigenvectors of $\Delta_n$ with the \textit{largest} eigenvalues, which would not constitute a desirable output for data analysis, as the tail of the spectrum of $\Delta_n$ has little geometric information about the data set $\X_n$. It is interesting that we can still recover the relevant part of the spectrum of $\Delta_n$ indirectly, by studying the spectrum of the matrix $\tx$ that we use in this paper. Finally, it is worth mentioning that we add the term $a \I_n$ in the definition of $\tx$ in \ref{eq-def:Graph Laplacian} to guarantee that $\tx$ is always PSD, in this way simplifying the statements and proofs of our main results.  
\end{remark}

\subsection{Form of $\my^*$ and some notation}
\label{app:Y*}

Since $\tx$ is a PSD matrix, the Eckart–Young–Mirsky theorem (see \cite{eckart1936approximation})  implies that the global optimizers of \ref{eq:optimization} are the matrices $\my$ of the form $\my= \my^* O$, where $O \in \mathbb{O}_r$ and 
\[
\my^*:= \begin{bmatrix}
    \vert &   &\vert \\
    \sqrt{\sigma_1(\tx)}v_1   &  \dots  & \sqrt{\sigma_r(\tx)}v_r   \\
    \vert &    & \vert
\end{bmatrix}.
\] 
In the above, $\sigma_l(\tx)$ represents the $l$-th largest eigenvalue of $\tx$ and $v_l$ is a corresponding eigenvector with Euclidean norm one. In case there are repeated eigenvalues, the corresponding $v_l$ need to be chosen as being orthogonal to each other. 

For convenience, we rescale the vectors $v_l$ as follows:
\[ u_l := {\sqrt{n}} v_l. \]
In this way we guarantee that 
\[  \lVert u_l \rVert_{L^2(\X_n)}^2 := \frac{1}{n}\sum_{i=1}^n ( u_l(x_i))^2 = 1, \]
i.e., the rescaled eigenvectors $u_l$ are normalized in the $L^2$-norm with respect to the empirical measure $\frac{1}{n}\sum_{i=1}^n \delta_{x_i}$. In terms of the rescaled eigenvectors $u_l$, we can rewrite $\my^*$ as follows:
\begin{equation}
\my^*= \begin{bmatrix}
    \vert &   &\vert \\
    \sqrt{\frac{\sigma_1(\tx)}{n}}u_1   &  \dots  & \sqrt{\frac{\sigma_r(\tx)}{n}}u_r   \\
    \vert &    & \vert
\end{bmatrix}.
\label{eqn:FormulaY*}
\end{equation}

\begin{remark}
As discussed in Remark \ref{rem:OrderEigenvaluesAn} below, under Assumptions \ref{assum:data} we can assume that all the $\sigma_s(\tx)$ are quantities of order one.
\end{remark}

\section{Auxiliary Approximation Results}
\label{app:AuxApproxResults}


 





\nc

\subsection{Graph-Based Spectral Approximation of Weighted Laplace-Beltrami Operators}

\label{Append:Auxiliary Lemma:Manifold learning}

In this section, we discuss two important results characterizing the behavior of the spectrum of the normalized graph Laplacian matrix $\Delta_n$ defined in \eqref{prop:property2} when $n$ is large and $\veps$ scales with $n$ appropriately. In particular, $\Delta_n$'s spectrum is seen to be closely connected to that of the weighted Laplace-Beltrami operator $\Delta_\rho$ defined as
\begin{equation*}
  \Delta_{\rho} f := - \frac{1}{\rho^{3/2}} \mathrm{div}\left(   \rho^2 \nabla \left( \frac{f}{\sqrt{\rho}} \right)  \right)  
\end{equation*}
for all smooth enough $f: \mathcal{M} \rightarrow \mathbb{R}$; see section 1.4 in \cite{GARCIATRILLOS2018239}. In the above, $\text{div}$ stands for the divergence operator on $\mathcal{M}$, and $\nabla$ for the gradient in $\mathcal{M}$. $\Delta_\rho$ can be easily seen to be a positive semi-definite operator with respect to the $L^2(\M, \rho)$ inner product and its eigenvalues (repeated according to multiplicity) can be listed in increasing order as
$$
0 = \sigma_1^\M \leq \sigma_2^\M \leq \ldots
$$
We will use $f_1, f_2, \dots$ to denote associated normalized (in the $L^2(\M,\rho)$-sense) eigenfuntions of $\Delta_\rho$.

The first result, whose proof we omit as it is a straightforward adaptation of the proof of Theorem 2.4 in \cite{calder2022improved} --which considers the \textit{unnormalized} graph Laplacian case--, relates the eigenvalues of $\Delta_n$ and $\Delta_\rho$.

\begin{theorem}[Convergence of eigenvalues of graph Laplacian; Adapted from Theorem 2.4 in \cite{calder2022improved}]
Let $l \in \N$ be fixed. Under Assumptions \ref{assum:data}, with probability at least $1-Cn\exp \left(-c n \varepsilon^{m+4}\right)$ over the sampling of the $x_i$, we have:
\[   \left| \beta_\eta  \sigma_s^\M - \frac{\hat{\sigma}_s}{\veps^2} \right| \leq C_r \veps , \quad  \forall s=1, \dots, l.  \]
In the above, $\hat{\sigma}_1 \leq \dots \leq \hat{\sigma}_l$ are the first eigenvalues of $\Delta_n$ in increasing order, $C_l$ is a deterministic constant that depends on $\M$'s geometry and on $l$, and $\beta_\eta$ is a constant that depends on the kernel $\eta$ determining the graph weights (see \eqref{eq:eps-graph}). We also recall that $m$ denotes the intrinsic dimension of the manifold $\M$.
\label{thm:EigenvalueConvergence}
\end{theorem}

\begin{remark}
\label{rem:OrderEigenvaluesAn}
From Theorem \ref{thm:EigenvalueConvergence} and \eqref{eqn:AnDeltan} we see that the top $l$ eigenvalues of $\tx$ (for $l$ fixed), i.e., $\sigma_1(\tx), \dots, \sigma_{l}(\tx)$, can be written as
\[  \sigma_s(\tx) = 1+a - \beta_\eta \sigma_s^\M\veps^2 + O(\veps^3)\]
with very high probability.

In particular, although each individual $\sigma_s(\tx)$ is an order one quantity, the difference between any two of them is an order $\veps^2$ quantity.
\end{remark}

Next we discuss the convergence of eigenvectors of $\Delta_n$ toward eigenfunctions of $\Delta_\rho$. For the purposes of this paper (see some discussion below) we follow a strong, almost $C^{0,1}$ convergence result established in \cite{calder2022lipschitz} for the case of unnormalized graph Laplacians. A straightforward adaptation of Theorem 2.6 in \cite{calder2022lipschitz} implies the following.

\begin{theorem}[Almost $C^{0,1}$ convergence of graph Laplacian eigenvectors; Adapted from Theorem 2.6 in \cite{calder2022lipschitz}]\label{prop:Lipschitz of eigenvector} Let $r \in \N$ be fixed and let $u_1, \dots, u_r$ be normalized eigenvectors of $\Delta_n$ as in Appendix \ref{app:Y*}. Under Assumptions \ref{assum:data}, with probability at least $1-C \varepsilon^{-6m} \exp \left(-c n \varepsilon^{m+4}\right)$ over the sampling of the $x_i$, we have:
\begin{align}\label{eq:eigenvector lipschitz}
    \lVert  f_s - u_s \rVert_{L^\infty(\X_n)} + [f_s - u_s]_{\veps, \X_n} \leq C_r \veps. \quad \forall s=1, \dots, r,
\end{align}
for normalized eigenfuctions $f_i: \M \rightarrow \R$ of $\Delta_\rho$, as introduced at the beginning of this section.
In the above, $\lVert \cdot \rVert_{L^\infty(\X_n)}$ is the norm $\lVert v \rVert_{L^\infty(\X_n)} := \max_{x_i \in \X_n} | v(x_i)|$, and $[\cdot]_{\veps, \X_n}$ is the seminorm
\[ [ v]_{\veps, \X_n} := \max_{x_i, x_j \in \X_n } \frac{|v(x_i) - v(x_j)|}{d_\M(x_i, x_j) + \veps}.  \]
$d_\M(\cdot,\cdot)$ denotes the geodesic distance on $\M$.
\end{theorem}

An essential corollary of the above theorem is the following set of regularity estimates satisfied by eigenvectors of the normalized graph Laplacian $\Delta_n$.

\begin{corollary}
\label{cor:RegularityEigenvectors}
Under the same setting, notation, and assumptions as in Theorem \ref{prop:Lipschitz of eigenvector}, the functions $u_s$ satisfy
\begin{equation}
 | u_s(x_i) - u_s(x_j)| \leq L_s ( d_\M(x_i, x_j) + \veps^2), \quad \forall x_i, x_j \in \X_n
 \label{eqn:RegularityEigenImproved}
 \end{equation}
for some deterministic constant $L_s$.
\end{corollary}

\begin{proof}
From \eqref{eq:eigenvector lipschitz} we have 
\[  | (u_s(x_i) - f_s(x_i)) -(  u_s(x_j) - f_s(x_j) )  | \leq  C_s\veps ( d_\M(x_i, x_j) +\veps  ) , \quad \forall x_i, x_j \in \X_n.  \]
It follows from the triangle inequality that
\begin{align*}
 |  u_s(x_i) - u_s(x_j)| & \leq | u_s(x_i) - f_s(x_i) -(  u_s(x_j) - f_s(x_j) )  | + | f_s(x_i) - f_s(x_j) |  
 \\& \leq C_s\veps(d_\M(x_i, x_j) + \veps) + C_s' d_\M(x_i, x_j)
 \\& \leq L_s( d_\M(x_i, x_j) + \veps^2).
\end{align*}
In the above, the second inequality follows from inequality \ref{eq:eigenvector lipschitz} and the fact that $f_s$, being a normalized eigenfunction of the elliptic operator $\Delta_\rho$, is Lipschitz continuous with some Lipschitz constant $C_s'$.

\end{proof}

\begin{remark}
\label{rem:AboutEpsSq}
    We observe that the $\veps^2$ term on the right hand side of \eqref{eqn:RegularityEigenImproved} is strictly better than the $\veps$ term that appears in the explicit regularity estimates in Remark 2.4 in \cite{calder2022lipschitz}. It turns out that in the proof of Theorem \ref{thm:eigenappr} it is essential to have a correction term for the distance that is $o(\veps)$; see more details in Remark \ref{rem:epsSqSaddles} below.
\end{remark}

\subsection{Neural Network Approximation of Lipschitz Functions on Manifolds}

\label{append:NNUniversal}

\cite{chen2019nonparametric} shows that Lipschitz functions $f$ defined over an $m$-dimensional smooth manifold $\M$ embedded in $\R^d$ can be approximated with a ReLU neural network with a number of neurons that doesn't grow exponentially with the ambient space dimension $d$. Precisely:

\begin{theorem}[Theorem 1 in \cite{chen2019nonparametric}]\label{prop:universal appro of manifold}

Let $f: \mathcal{M} \rightarrow \mathbb{R}$ be a Lipschitz function with Lipschitz constant less than $K$. Given any $\delta \in (0,1)$, there are $\kappa, L, p, N$ satisfying: 
\begin{enumerate}
    \item $ L \leq C_K\left(\log \frac{1}{\delta}+\log d\right)$, and  $p \leq C_K\left(\delta^{-m}+d\right)$,
    \item  $N \leq C_K\left(\delta^{-m} \log \frac{1}{\delta}+d \log \frac{1}{\delta}+d \log d\right)$, and $\kappa \leq C_K $,
\end{enumerate}
such that there is a neural network $f_\theta \in \mathcal{F}(1, \kappa, L, p, N)$ (as defined in \eqref{eq-def:F class}), for which  
\[\|f_\theta-f\|_{L^\infty(\M)} \leq \delta.\] 
In the above, $C_K$ is a constant that depends on $K$ and on the geometry of the manifold $\M$.
\end{theorem}

\section{Proofs of Theorem \ref{thm-appr:main} and Corollary \ref{cor:universalappmanifold}}\label{appen:universal approximation:main result}


\begin{lemma}\label{lem:almost lipschitz to lipschitz}
    Let $u:\X_n\to\R$ be a function satisfying 
    \begin{align}
        |u(x)-u(\tilde x)|\le L (d_\M(x,\tilde x)+\eps^2), \quad \forall x, \tilde x \in \X_n
    \end{align}
    for some $L$ and $\veps>0$. Then
    there exists a $3L$-Lipschitz function $\tilde{g}:\M\to\R$ such that
    \begin{align}\label{eq:|f-g|_infty is small}
        \|u-\tilde{g}\|_{L^{\infty}(\X_n) }\le 5L \eps^2.
    \end{align}
\end{lemma}
\begin{proof}
We start by constructing a subset $\X_n'$ of $\X_n$ satisfying the following properties:
\begin{enumerate}
   \item Any two points $x,\tilde x \in \X_n'$ (different from each other) satisfy $d_\M(x,\tilde x) \geq \frac{1}{2}\veps^2$.
   \item For any $x\in \X_n$ there exists $\tilde x \in \X_n'$ such that $d_\M(x, \tilde x) \leq \veps^2$.
\end{enumerate}
The set $\X_n'$ can be constructed inductively, as we explain next. First, we enumerate the points in $\X_n$ as $x_1, \dots, x_n$. After having decided whether to include or not in $\X_n'$ the first $s$ points in the list, we decide to include $x_{s+1}$ as follows: if the ball of radius $\veps^2/2$ centered at $x_{s+1}$ intersects any of the balls of radius $\veps^2/2$ centered around the points already included in $\X_n'$, then we do not include $x_{s+1}$ in $\X_n'$, otherwise we include it. It is clear from this construction that the resulting set $\X_n'$ satisfies the desired properties (property 2 follows from the triangle inequality).

Now, notice that the function $u: \X_n' \rightarrow \R$ (i.e., $u$ restricted to $\X_n'$) is $3L$-Lipschitz, since 
\[ | u(x) - u(\tilde x)| \leq L(d_\M(x, \tilde x) + \veps^2) \leq 3Ld_\M(x, \tilde x)   \]
for any pair of points $x, \tilde x$ in $\X_n'$. Using McShane-Whitney theorem we can extend the function $u: \X_n'\rightarrow \R $ to a $3L$-Lipschitz function $\tilde g : \M \rightarrow \R$. It remains to prove \eqref{eq:|f-g|_infty is small}. To see this, let $x \in \X_n$ and let $\tilde x \in \X_n$ be as in property 2 of $\X_n'$. Then
\begin{align*}
| u(x) - \tilde{g}(x) | & \leq  | u(x) - u(\tilde x) | + | u(\tilde x ) - g(x)   |
\\&  =  | u(x) - u(\tilde x) | + | g(\tilde x ) - g(x)   |
\\& \leq L(d_\M(x, \tilde x) + \veps^2) + 3L d_\M(x, \tilde x)
\\& \leq 5L\veps^2.
\end{align*}

    This completes the proof.
\end{proof}

We are ready to prove Theorem \ref{thm-appr:main}, which here we restate for convenience.

\UniversalAppr*

\begin{proof}

 Let $s \leq r$. As in the discussion of section \ref{app:Y*} we let $u_s$ be a $\lVert \cdot \rVert_{L^2(\X_n)}$-normalized eigenvector of $\Delta_n$ corresponding to its $s$-th smallest eigenvalue. Thanks to Corollary \ref{cor:RegularityEigenvectors}, we know that, with very high probability, the function $u_s: \X_n \rightarrow \R$ satisfies 
 \begin{align}
        |u_s(x_i)- u_s(x_j)|\le L_s (d_\M(x_i,x_j)+\eps^2), \quad \forall x_i, x_j \in \X_n,
    \end{align}
for some deterministic constant $L_s$. Using the fact that $\sqrt{\sigma_s(\tx)}$ is an order one quantity (according to Remark \ref{rem:OrderEigenvaluesAn}) in combination with Lemma \ref{lem:almost lipschitz to lipschitz}, we deduce the existence of a $C L_s$-Lipschitz function $g_s : \M \rightarrow \R$ satisfying
\begin{equation}
   \lVert   g_s - \sqrt{\sigma_s(\tx)} u_s \rVert_{L^\infty(\X_n)} \leq  5CL_s \veps^2. 
   \label{eqn:AuxApprox1}
\end{equation}
In turn, Theorem \ref{prop:universal appro of manifold} implies the existence of parameters $\kappa, L, p,N$ as in the statement of the theorem and a (scalar-valued) neural network $f_{\tilde{\theta}}$ in the class $\F(1,\kappa,L,p,N)$ such that
  \begin{align}
        \lVert f_{\tilde \theta}(x)-g_s(x)\rVert_{L^\infty(\M)} \le \delta.
    \end{align}
Using the fact that the ReLU is a homogeneous function of degree one, we can deduce that
\[  \frac{1}{\sqrt{n}} f_{\tilde \theta} = f_\theta,     \]
where $\theta := \frac{1}{n^{1/(2L)}} \tilde{\theta} $ and thus $f_\theta \in \F(1,\frac{\kappa}{n^{1/(2L)}},L,p,N)$. It follows that the neural network $f_\theta$ satisfies
\[ \sqrt{n} \lVert  f_\theta - \frac{1}{\sqrt{n}} g_s    \rVert_{L^\infty(\M)}  \leq \delta,   \]
and also, thanks to \eqref{eqn:AuxApprox1}, 
\[ \sqrt{n} \left\lVert  f_\theta - \sqrt{\frac{\sigma_s(\tx)}{n}} u_s\right\rVert_{L^\infty(\X_n)} \leq (5CL_s+1)(\delta+ \veps^2). \]

Stacking the scalar neural networks constructed above to approximate each of the functions $u_s$ for $s=1, \dots r$, and using \eqref{eqn:FormulaY*}, we obtain the desired vector valued neural network approximating $\my^*$.

\end{proof}

\begin{remark}
Notice that the term $\sqrt{n}\lVert \my^*  \rVert_{\infty, \infty}$ is of order one. Consequently, the estimate in Theorem \ref{thm-appr:main} is a non-trivial error bound. 
\end{remark}


The bound in $\lVert \cdot \lVert_{\infty, \infty}$ between $\my_\theta$ and $\mys$ in Theorem \ref{thm-appr:main} can be used to bound the difference between $\my_\theta \my_\theta^\top$ and $\mys\my^{*\top}$ in $\lVert \cdot \rVert_{\infty, \infty}$.

\begin{restatable}{corollary}{frobBound}
    \label{corollary:|f_theta(X)f_theta(X)-YY|_infty<=upper bound}
     For $f_\theta$ as in Theorem \ref{thm-appr:main} we have
    \begin{align}
        \sqrt{n} \|\my_\theta\my_\theta^\top-\mys\my^{*\top}\|_{\infty,\infty}\le  C_r (\delta + \veps^2),
    \end{align}
    and thus also
    \[  \|\my_\theta\my_\theta^\top-\mys\my^{*\top}\|_{\rF}\le  \sqrt{r}C_r (\delta + \veps^2),  \]
    for some deterministic constant $C_r$.
\end{restatable}
\begin{proof}
    \begin{align*}
        \sqrt{n}\|\my_\theta\my_\theta^\top-\mys\my^{*\top}\|_{\infty,\infty}&= \sqrt{n}\|\my_\theta\left(\my_\theta^\top-\my^{*\top}\right)+\left(\my_\theta-\mys\right)\my^{*\top}\|_{\infty,\infty}\\
        &\le \sqrt{n} \|\my_\theta\left(\my_\theta^\top-\my^{*\top}\right)\|_{\infty,\infty}+ \sqrt{n}\|\left(\my_\theta-\mys\right)\my^{*\top}\|_{\infty,\infty}\\
        &\le \sqrt{n r} \|\my_\theta\|_{\rF}\|\my_\theta^\top-\my^{*\top}\|_{\infty,\infty}+ \sqrt{n{r}} \|\my_\theta-\mys\|_{\infty,\infty}\|\my^{*\top}\|_{\rF}\\
         &\le  \sqrt{r}(C_r(\delta+\eps^2)+2\|\mys\|_{\rF})C_r(\delta+\eps^2)
        \\ & \le C_r(\delta + \veps^2),
    \end{align*}
    where the second to last inequality follows from our estimate for $\sqrt{n}\|\my_\theta-\mys\|_{\infty,\infty}\le C_r(\delta+\eps^2)$ in Theorem \ref{thm-appr:main}, and the last inequality follows from the fact that $\lVert \my^* \rVert^2_\rF = \sum_{s=1}^r \sigma_s(\tx) = \mathcal{O}(r)$. 
\end{proof}

\subsection{Eigenfunction approximation}
The neural network $f_\theta$ constructed in the proof of Theorem \ref{thm-appr:main} can be used to approximate eigenfunctions of $\Delta_\rho$. We restate Corollary \ref{cor:universalappmanifold} for the convenience of the reader.



\universalappmanifold*
\begin{proof}
   Let $g_s : \M \rightarrow \R$ be the Lipschitz function appearing in \eqref{eqn:AuxApprox1} and recall that the scalar neural network $f_\theta$ constructed in the proof of Theorem \ref{thm-appr:main} satisfies 
\begin{align}\label{eq:ftheta-gs}
    \sqrt{n} \lVert  f_\theta - \frac{1}{\sqrt{n}} g_s    \rVert_{L^\infty(\M)}  \leq \delta.
\end{align}

It can be shown that except on an event with probability less than  $n \exp(- n \veps^{m})$, for any $x\in\M$, there exists $x_i\in\X_n$ such that $d_\M(x_i,x)\le \eps$. From the triangle inequality, it thus follows that 
    \begin{equation}\label{eq:|u_lambda(x)-g_lambda(x)|}
        \begin{split}
            |f_s(x)-\sqrt{n/(1+a)}f_\theta(x)|
            \le&|f_s(x)-f_s(x_i)|+ |f_s(x_i)-u_s(x_i)|
            \\&+  | u_s(x_i) - \frac{1}{\sqrt{\sigma_s(\tx)}}g_s(x_i)| +| \frac{1}{\sqrt{\sigma_s(\tx)}}g_s(x_i) -\frac{1}{\sqrt{1+a}} g_s(x_i)|
            \\& + | \frac{1}{\sqrt{1+a}} g_s(x_i)-\frac{1}{\sqrt{1+a}} g_s(x)| 
            \\& +|\frac{1}{\sqrt{1+a}} g_s(x)-\sqrt{\frac{n}{1+a}}f_\theta(x)|\\
            \le& C_s (\delta+\eps),
        \end{split}
    \end{equation}
    where we have used the Lipschitz continuity of $f_s$ and $g_s$, Theorem \ref{prop:Lipschitz of eigenvector}, Remark \ref{rem:OrderEigenvaluesAn}, and \eqref{eq:ftheta-gs}.


\end{proof}

\begin{remark}\label{remark:memorization}
    We notice that, while one could use existing \textit{memorization} results (e.g., Theorem 3.1 in \cite{NEURIPS2019_dbea3d0e}) to show that there is a neural network with ReLU activation function and $\mathcal{O}(\sqrt{n})$ neurons that fits $\my^*$ perfectly, this does not constitute an improvement over our results in Theorem \ref{thm-appr:main} and Corollary \ref{cor:universalappmanifold}. Indeed, by using this type of memorization result, we can not state any bounds on the size of the parameters of the network, and none of the out-of-sample generalization properties that we have discussed before (i.e., approximation of eigenfunctions of $\Delta_\rho$) can be guaranteed.
\end{remark}

\section{Proof of Theorem \ref{thm:eigenappr}}
\label{appendix:Eigenappr}

In this section we prove our main result on the spectral approximation of $\tx$ using the matrices induced by global minimizers of SNN's objective. We start our proof with a lemma from linear algebra.

\begin{lemma}
\label{lem:LinearALgebraIneq}
For any $\my \in \R^{n \times r}$ we have
\[ \|\my\my^\top-\tx\|^2_\rF-\|\mys\my^{*\top}-\tx\|^2_\rF \leq \lVert \my\my^\top - \mys\my^{*\top} \rVert^2_\rF. \]
\end{lemma}

\begin{proof}
A straightforward computation reveals that 
\begin{equation}
        \begin{split}
            &\|\my \my^\top-\tx\|^2_\rF-\|\mys\my^{*\top}-\tx\|^2_\rF\\
            &=\|(\my \my^\top-\mys\my^{*\top})+(\mys\my^{*\top}-\tx)\|^2_\rF- \|\mys\my^{*\top}-\tx\|^2_\rF \\
            &= \|\my \my^\top-\mys\my^{*\top}\|^2_\rF
            +2\langle\my\my^\top-\mys\my^{*\top}
            , \mys\my^{*\top}-\tx\rangle\\
            &= \|\my \my^\top-\mys\my^{*\top}\|^2_\rF
            +2\langle\my \my^\top
            , \mys\my^{*\top}-\tx\rangle\\
            &\le \|\my\my^\top-\mys\my^{*\top}\|^2_\rF,
        \end{split}
    \end{equation}
where the last inequality follows thanks to the fact that $\my\my^\top$ is positive semi-definite and the fact that $\my^*\my^{*\top} - \tx$ is negative semi-definite, as can be easily deduced from the form of $\my^*$ discussed in section \ref{app:Y*}.
\end{proof}

Invoking Corollary \ref{corollary:|f_theta(X)f_theta(X)-YY|_infty<=upper bound} with $\delta= \tilde \delta \veps$ we immediately obtain the following approximation estimate. 

\begin{corollary}
\label{cor:ApproximationImproved}

With probability at least 
\[1-C \varepsilon^{-6m} \exp \left(-c n \varepsilon^{m+4}\right),
    \]
    for every $\tilde \delta \in (0,1)$
there is $f_\theta \in \mathcal{F}(r, \kappa, L, p,N)$ with $\kappa,L, p, N$ as specified in Theorem \ref{thm:eigenappr} such that 
    \begin{align}\label{eq:||YY-Y*Y*||_F2}
        \|\my_\theta\my_\theta^\top-\mys\my^{*\top}\|_{\rF}\le  C_r \veps(\tilde \delta +\veps).
    \end{align}
\end{corollary}


\begin{corollary}
\label{lem:normDifference}
Let $f_\theta$ be as in Corollary \ref{cor:ApproximationImproved}. Then 
\[   \lVert \my_{\theta} \my_{\theta}^{\top} - \tx  \rVert_\rF^2 -  \lVert \my^* \my^{*\top} - \tx  \rVert_\rF^2 \leq C_r\veps^2(\tilde \delta  +\veps)^2. \]
\end{corollary}
\begin{proof}
Let $\theta$ be as in Corollary \ref{cor:ApproximationImproved}. Then
\begin{align*}
 \lVert \my_{\theta} \my_{\theta}^{\top} - \tx  \rVert_\rF^2 -  \lVert \my^* \my^{*\top} - \tx  \rVert_\rF^2 \leq   \lVert \my_\theta\my_\theta^\top - \mys\my^{*\top} \rVert^2_\rF  \leq  C_r^2 \veps^2 (\tilde{\delta} + \veps)^2,
\end{align*}
where the second to last inequality follows from Lemma \ref{lem:LinearALgebraIneq}.
\end{proof}

In what follows we will write the SVD (eigendecomposition) of $\tx$ as $\overline{\mU}\mSigma\overline{\mU}^\top$. Using the fact that $\overline{\mU}$ is invertible, we can easily see that $\my_{\theta^*}$ can be written as $\my_{\theta^*}=\overline{\mU}(\E^1+\E^2)$ where $\E^1,\E^2\in\R^{n\times r}$ are matrices satisfying: the $i^{\mathrm{th}}$ row $\E^1_i=\0$ for $i\ge r+1$, $i^{\mathrm{th}}$ row $\E^2_i=\0$ for $i\le r$. We thus have $(\E^2)^\top\E^1=\0$.

In what follows we will make the following assumption. 
 \begin{assumption}\label{assum:eigenappr}
     $\eps$ and $\tilde{\delta}$ in Corollary \ref{cor:ApproximationImproved} satisfy the following condition:
     \begin{align}
         \eps^2 E <  \sigma_r^2(\tx)-\sigma_{r+1}^2(\tx),
     \end{align}
     where $E:= C_r(\tilde{\delta}+\veps)^2 $.
 \end{assumption}

 \begin{remark}
 \label{rem:epsSqSaddles}
     Assumption \ref{assum:eigenappr} is satisfied under the assumptions in the statement of Theorem \ref{thm:eigenappr}. To see this, notice that 
     $\sigma_r^2(\tx)-\sigma^2_{r+1}(\tx)\sim\eps^2$ according to Remark \ref{rem:OrderEigenvaluesAn} and the fact that $\sigma_r^\M < \sigma_{r+1}^\M$. Thus, taking $\tilde{\delta}$ to be sufficiently small, we can guarantee that indeed $\veps^2E < \sigma_r^2(\tx)-\sigma_{r+1}^2(\tx) $. 
 \end{remark}

 \begin{remark}
 \label{rem:ngt1}
Returning to Remark \ref{rem:AboutEpsSq}, if the correction term in the Lipschitz estimate for graph Laplacian eigenvectors had been $\veps$, and not $\veps^2$, the term $\veps^2 E $ would have to be replaced with the term $(C_r\veps\tilde \delta + C_r \veps)^2$, but the latter cannot be guaranteed to be smaller than $\sigma^2_r(\tx) - \sigma^2_{r+1}(\tx)$.     
\end{remark}

\begin{remark}
\label{rem:SaddleEnergyGap}
The energy gap between $\my^*$ and the constructed $\my_\theta$ is, according to Corollary \ref{lem:normDifference}, $\veps^2 E$, whereas the energy gap between $\my^*$ and any other critical point of $\ell$ that is not a global optimizer is in the order of $\veps^2$, as it follows from Remark \ref{rem:OrderEigenvaluesAn}. Continuing the discussion from Remark \ref{rem:ngt1}, it was thus relevant to use estimates that could guarantee that, at least energetically, our constructed $\my_\theta$ was closer to $\my^*$ than any other saddle of $\ell$. 
\end{remark}

\begin{proof}[Proof of Theorem \ref{thm:eigenappr}]
Due to the definition of $\theta^*$, we have
    \begin{align}
        \|\mys\my^{*\top}-\tx\|_\rF^2
        \le\|\my_{\theta^*}\my^\top_{\theta^*}-\tx\|_\rF^2
        \le \|\my_\theta\my_\theta^\top-\tx\|_\rF^2.
    \end{align}
    Also,
    \begin{equation}\label{eigenvector SNN0}
        \begin{split}
            0&\ge 
            \|\my_{\theta^*}\my^\top_{\theta^*}-\tx\|_\rF^2
        - \|\my_\theta\my_\theta^\top-\tx\|_\rF^2\\
        &=\|(\my_{\theta^*}\my^\top_{\theta^*}-\mys\my^{*\top})+(\mys\my^{*\top}-\tx)\|_\rF^2
        - \|\my_\theta\my_\theta^\top-\tx\|_\rF^2\\
        &=\|\my_{\theta^*}\my^\top_{\theta^*}-\mys\my^{*\top}\|_\rF^2+\|\mys\my^{*\top}-\tx\|_\rF^2 +2\langle \my_{\theta^*}\my^\top_{\theta^*}-\mys\my^{*\top},\mys\my^{*\top}-\tx\rangle
        - \|\my_\theta\my_\theta^\top-\tx\|_\rF^2\\
        &=\|\my_{\theta^*}\my^\top_{\theta^*}-\mys\my^{*\top}\|_\rF^2+\|\mys\my^{*\top}-\tx\|_\rF^2 
        +2\langle \my_{\theta^*}\my^\top_{\theta^*},\mys\my^{*\top}-\tx\rangle
        - \|\my_\theta\my_\theta^\top-\tx\|_\rF^2 
        \end{split}
    \end{equation}
    where the third equality follows from the fact that $\langle \mys\my^{*\top},\mys\my^{*\top}-\tx\rangle=0$. 
    Notice that
    \begin{equation}\label{eq:eigenvector_SNN2}
        \begin{split}
            \|\my_{\theta^*}\my^\top_{\theta^*}-\mys\my^{*\top}\|_\rF^2
        +2\langle \my_{\theta^*}\my^\top_{\theta^*},\mys\my^{*\top}-\tx\rangle
        =\|\my_{\theta^*}\my^\top_{\theta^*}\|_\rF^2+\|\mys\my^{*\top}\|_\rF^2
        -2\langle \my_{\theta^*}\my^\top_{\theta^*},\tx\rangle
        \end{split}
    \end{equation}
    By combining \eqref{eigenvector SNN0}, Lemma \ref{lem:normDifference} and \eqref{eq:eigenvector_SNN2}, we have
    \begin{equation}\label{eq:error bound}
        \|\my_{\theta^*}\my^\top_{\theta^*}\|_\rF^2+\|\mys\my^{*\top}\|_\rF^2
        -2\langle \my_{\theta^*}\my^\top_{\theta^*},\tx\rangle\le \eps^2 E 
    \end{equation}

    From $(\E^1)^\top\E^2=\0$ and $\mathrm{Tr}(AB)=\mathrm{Tr}(BA)$, we have
    \begin{equation}\label{eq:inner product is 0}
        \begin{split}
            \langle \E^1(\E^1)^\top, \E^2(\E^2)^\top\rangle=0\\
            \langle \E^1(\E^2)^\top, \E^2(\E^2)^\top\rangle=0\\
            \langle \E^1(\E^2)^\top, \E^1(\E^1)^\top\rangle=0\\
            \langle \E^2(\E^1)^\top, \E^1(\E^1)^\top\rangle=0\\
            \langle \E^2(\E^1)^\top, \E^2(\E^2)^\top\rangle=0\\
        \end{split}
    \end{equation}
    
    Let $\mSigma^{1}$ be the diagonal matrix such that $(\mSigma^{1})_{ii}=\mSigma_{ii}$ for $i\le r$, and $(\mSigma^{1})_{ii}=0$ for $i> r$; let $\mSigma^{2}$ be the diagonal matrix such that $(\mSigma^{2})_{ii}=0$ for $i\le r$, and $(\mSigma^{2})_{ii}=\mSigma_{ii}$ for $i> r$. By plugging the decomposition of $\my_{\theta^*}$ in \eqref{eq:error bound}, we deduce
    \begin{equation}\label{eq:first inequality}
        \begin{split}
           \eps^2 E\ge 
            &\|\my_{\theta^*}\my^\top_{\theta^*}\|_\rF^2+\|\mys\my^{*\top}\|_\rF^2
        -2\langle \my_{\theta^*}\my^\top_{\theta^*},\tx\rangle\\
        =&\|\overline{\mU}(\E^1+\E^2)(\E^1+\E^2)^\top\overline{\mU}^\top\|_\rF^2+\|\mys\my^{*\top}\|_\rF^2
        -2\langle \overline{\mU}(\E^1+\E^2)(\E^1+\E^2)^\top\overline{\mU}^\top,\tx\rangle\\
        =&\|(\E^1+\E^2)(\E^1+\E^2)^\top\|_\rF^2+\|\mys\my^{*\top}\|_\rF^2
        -2\langle (\E^1+\E^2)(\E^1+\E^2)^\top,\mSigma\rangle\\
        \stackrel{\text{\eqref{eq:inner product is 0}}}{=}&\|\E^1(\E^1)^\top\|_\rF^2+\|\E^2(\E^2)^\top\|_\rF^2 + 2 \|\E^2(\E^1)^\top\|_\rF^2 + 2\langle (\E^1)^\top\E^1, (\E^2)^\top\E^2\rangle\\
        &+\|\mys\my^{*\top}\|_\rF^2 -2\langle (\E^1+\E^2)(\E^1+\E^2)^\top,\mSigma\rangle\\
        \stackrel{(\E^1)^\top\mSigma\E^2=\0}{=}&\|\E^1(\E^1)^\top\|_\rF^2+\|\E^2(\E^2)^\top\|_\rF^2 + 2 \|\E^2(\E^1)^\top\|_\rF^2+ 2\langle (\E^1)^\top\E^1, (\E^2)^\top\E^2\rangle\\
        &+\|\mys\my^{*\top}\|_\rF^2 - 2\langle \E^1(\E^1)^\top+\E^2(\E^2)^\top,\mSigma\rangle\\
        =& \|\E^1(\E^1)^\top\|_\rF^2+\|\E^2(\E^2)^\top\|_\rF^2 + 2 \|\E^2(\E^1)^\top\|_\rF^2+ 2\langle (\E^1)^\top\E^1, (\E^2)^\top\E^2\rangle\\
        &+\|\mSigma^1\|_\rF^2 - 2\langle \E^1(\E^1)^\top,\mSigma^1\rangle - 2\langle \E^2(\E^2)^\top,\mSigma^2\rangle \\
        =&\|\E^1(\E^1)^\top-\mSigma^1\|_\rF^2 +\|\E^2(\E^2)^\top\|_\rF^2 + 2 \|\E^2(\E^1)^\top\|_\rF^2+ 2\langle (\E^1)^\top\E^1, (\E^2)^\top\E^2\rangle\\
        &- 2\langle \E^2(\E^2)^\top,\mSigma^2\rangle \\
        \ge& \|\E^1(\E^1)^\top-\mSigma^1\|_\rF^2 +\|\E^2(\E^2)^\top\|_\rF^2 + 2 \|\E^2(\E^1)^\top\|_\rF^2+ 2\langle (\E^1)^\top\E^1, (\E^2)^\top\E^2\rangle\\
        &- 2\|\E^2(\E^2)^\top\|_\rF \cdot \sigma_{r+1}(\tx)\\
        \ge & \|\E^1(\E^1)^\top-\mSigma^1\|_\rF^2 +\|\E^2(\E^2)^\top\|_\rF^2 + (2 \|\E^2\|_\rF^2+2\|\E^2(\E^2)^\top\|_\rF) \cdot \sigma^2_r(\E^1) \\
        &- 2\|\E^2(\E^2)^\top\|_\rF \cdot \sigma_{r+1}(\tx).
        \end{split}
    \end{equation}
    On the other hand, we have
    \begin{equation}\label{eq:my_theta^*my_theta^*-mysmys}
        \begin{split}
            \|\my_{\theta^*}\my_{\theta^*}^\top-\mys\my^{*\top}\|_\rF^2
            =& \|\my_{\theta^*}\my^\top_{\theta^*}\|_\rF^2+\|\mys\my^{*\top}\|_\rF^2
        -2\langle \my_{\theta^*}\my^\top_{\theta^*},\tx\rangle+2\langle \my_{\theta^*}\my^\top_{\theta^*},\tx-\mys\my^{*\top}\rangle\\
        = & \|\my_{\theta^*}\my^\top_{\theta^*}\|_\rF^2+\|\mys\my^{*\top}\|_\rF^2
        -2\langle \my_{\theta^*}\my^\top_{\theta^*},\tx\rangle+2\langle \E^2(\E^2)^\top,\mSigma^2\rangle\\
        \le & \eps^2 E+ 2\|\E^2(\E^2)^\top\|_\rF \cdot \sigma_{r+1}(\tx).
        \end{split}
    \end{equation}

    It remains to show that  $\|\E^2(\E^2)^\top\|_\rF$ can be controlled by a term of the form $C\veps^2 E$.  We split the following discussion into two cases. First, we assume that $\sigma^2_r(\E^1)$ is large compared with $\sigma_{r+1}(\tx)$. In this first case  $\|\E^2(\E^2)^\top\|_\rF$ can be guaranteed to be small according to \eqref{eq:first inequality}. Second, when $\sigma^2_r(\E^1)$ is small, we'll show that $\|\E^1(\E^1)^\top-\mSigma^1\|_\rF^2$ is large, which will contradict \eqref{eq:first inequality}.
    
    \textbf{Case 1: If $\sigma^2_r(\E^1)\ge\frac{2}{3}\sigma_{r+1}(\tx)$}.

We have $3 \|\E^2\|_\rF^2 \cdot \sigma^2_r(\E^1) - 2\|\E^2(\E^2)^\top\|_\rF \cdot \sigma_{r+1}(\tx)\ge 0$. Then, from \eqref{eq:first inequality} and the fact that $\|AB\|_\rF\le \|A\|_\rF \cdot \|B\|_\rF$, we have
    \begin{align}
        \|\E^1(\E^1)^\top-\mSigma^1\|_\rF^2 +\|\E^2(\E^2)^\top\|_\rF^2 + 2\|\E^2(\E^2)^\top\|_\rF \cdot \sigma_{r+1}(\tx)\le\eps^2 E.
    \end{align}
    This immediately implies
    \begin{align}\label{eq:E2 bound 1}
        \|\E^2(\E^2)^\top\|_\rF\le\frac{\eps^2 E}{\sigma_{r+1}(\tx)}.
    \end{align}
    Combining \eqref{eq:E2 bound 1} and \eqref{eq:my_theta^*my_theta^*-mysmys}, we obtain
    \begin{align}
        \|\my_{\theta^*}\my_{\theta^*}^\top-\mys\my^{*\top}\|_\rF^2\le \eps^2 E+ \|\E^2(\E^2)^\top\|_\rF \cdot \sigma_{r+1}(\tx) \le 2\eps^2 E.
    \end{align}

    \textbf{Case 2: If $0\le\sigma^2_r(\E^1)<\frac{2}{3}\sigma_{r+1}(\tx)$.}
    
    Invoking \eqref{eq:first inequality}, we have
    \begin{equation}\label{eq:fourth bound}
        \begin{split}
            \eps^2 E
            &\ge \|\E^1(\E^1)^\top-\mSigma^1\|_\rF^2 +\|\E^2(\E^2)^\top\|_\rF^2 + (2 \|\E^2\|_\rF^2+2 \|\E^2(\E^2)^\top\|_\rF) \cdot \sigma^2_r(\E^1) - 2\|\E^2(\E^2)^\top\|_\rF \cdot \sigma_{r+1}(\tx)\\
            &\ge (\sigma_r^2(\E^1)-\sigma_r(\tx))^2+\|\E^2(\E^2)^\top\|_\rF^2 + 4 \|\E^2(\E^2)^\top\|_\rF \cdot \sigma^2_r(\E^1) - 2\|\E^2(\E^2)^\top\|_\rF \cdot \sigma_{r+1}(\tx)\\
            &= (\sigma_r^2(\E^1)-\sigma_r(\tx))^2+\|\E^2(\E^2)^\top\|_\rF^2  - 2\|\E^2(\E^2)^\top\|_\rF \cdot (\sigma_{r+1}(\tx)-2\sigma^2_r(\E^1))\\
            &=(\sigma_r^2(\E^1)-\sigma_r(\tx))^2+\left(\|\E^2(\E^2)^\top\|_\rF-(\sigma_{r+1}(\tx)-2\sigma_r^2(\E^1))\right)^2-(\sigma_{r+1}(\tx)-2\sigma_r^2(\E^1))^2\\
            &\ge  (\sigma_r^2(\E^1)-\sigma_r(\tx))^2-(\sigma_{r+1}(\tx)-2\sigma_r^2(\E^1))^2,
        \end{split}
    \end{equation}
    where the second inequality follows from Weyl's inequality \cite{stewart1998matrix}.
    
    It is straightforward to check that $(\sigma_r^2(\E^1)-\sigma_r(\tx))^2-(\sigma_{r+1}(\tx)-2\sigma_r^2(\E^1))^2$ is an increasing function with respect to $\sigma_r^2(\E^1)$ in the range $0\le\sigma^2_r(\E^1)<\frac{2}{3}\sigma_{r+1}(\tx)$. The smallest value of $(\sigma_r^2(\E^1)-\sigma_r(\tx))^2-(\sigma_{r+1}(\tx)-2\sigma_r^2(\E^1))^2$ in this range is thus $\sigma_r^2(\tx)-\sigma_{r+1}^2(\tx)$. However, the resulting inequality contradicts Assumption \ref{assum:eigenappr}. Case 2 is thus void. 

    ~\\
    
    By combining the aforementioned two cases, we conclude
    \begin{align}
\|\my_{\theta^*}\my_{\theta^*}^\top-\mys\my^{*\top}\|_\rF^2 \le 2E\eps^2.
    \end{align}
 By using \eqref{App:eq:d([Y1,Y2])}, we have 
    \begin{align}
        d^2([\my_{\theta^*}],[\mys])\le\frac{1}{2(\sqrt{2}-1)\sigma_r^2(\mys)} \|\my_{\theta^*}\my_{\theta^*}^\top-\mys\my^{*\top}\|_\rF^2\le \frac{\eps^2 E}{(\sqrt{2}-1)\sigma_r^2(\mys)},
    \end{align}
    where $d([\my_{\theta^*}],[\mys])=\min_{\mO\in\mathbb{O}_r} \lVert \my_{\theta^*} - \my^* \mO \rVert_\rF$. This completes the proof.
\end{proof}


\section{Ambient Optimization}\label{appen:ambient problem}

This section contains the proof of the results from Section \ref{sec:Riemannian Optimization}.

\subsection{Setup from Main Text}

Let us recall the quotient manifold that we are interested in. Let $\overline{\NM}^n_{r+}$ be the space of $n \times r$ matrices with full column rank. 
To define the quotient manifold, we encode the invariance mapping, i.e., $\my \to \my\mO$, by defining the equivalence classes $[\my]=\{\my\mO:\mO\in\mathbb{O}_r\}$. Since the invariance mapping is performed via the Lie group $\mathbb{O}_r$ smoothly, freely and properly, we have $\NM_{r_+}^n\defeq \overline{\NM}_{r_+}^n / \mathbb{O}_r$ is a quotient manifold of $\overline{\NM}_{r+}^n$ \cite{lee2018introduction}.
Moreover, we equip the tangent space $T_{\mathbf{Y}} \overline{\NM}_{r_+}^n =\mathbb{R}^{n \times r}$ with the metric $\bar{g}_{\mathbf{Y}}\left(\eta_{\mathbf{Y}}, \theta_{\mathbf{Y}}\right)=\operatorname{tr}\left(\eta_{\mathbf{Y}}^{\top} \theta_{\mathbf{Y}}\right)$.

For convenience, we recall the following.
\begin{equation}\label{eq-def:gradHessian}
    \begin{split}
        \overline{\operatorname{grad} H([\mathbf{Y}])}&=2\left(\mathbf{Y} \mathbf{Y}^{\top}-\tx\right) \mathbf{Y},\\
\overline{\operatorname{Hess} H([\mathbf{Y}])}\left[\theta_{\mathbf{Y}}, \theta_{\mathbf{Y}}\right]&=\left\|\mathbf{Y} \theta_{\mathbf{Y}}^{\top}+\theta_{\mathbf{Y}} \mathbf{Y}^{\top}\right\|_{\mathrm{F}}^2+2\left\langle\mathbf{Y} \mathbf{Y}^{\top}-\tx, \theta_{\mathbf{Y}} \theta_{\mathbf{Y}}^{\top}\right\rangle
\end{split}
\end{equation}


\begin{equation}\label{eq-def:R}
    \begin{split}
&\mathcal{R}_1\defeq \left\{\mathbf{Y} \in \mathbb{R}_*^{n \times r} \middle| d\left([\mathbf{Y}],\left[\mathbf{Y}^*\right]\right) \leqslant \mu \sigma_r\left(\mathbf{Y}^*\right) / \kappa^*\right\}, \\
&\mathcal{R}_2\defeq \left\{\mathbf{Y} \in \mathbb{R}_*^{n \times r} \middle| \begin{array}{l}
d\left([\mathbf{Y}],\left[\mathbf{Y}^*\right]\right)>\mu \sigma_r\left(\mathbf{Y}^*\right) / \kappa^*,\|\overline{\operatorname{grad} H([\mathbf{Y}])}\|_{\mathrm{F}} \leqslant \alpha \mu \sigma_r^3\left(\mathbf{Y}^*\right) /\left(4 \kappa^*\right) \\
\|\mathbf{Y}\| \leqslant \beta\left\|\mathbf{Y}^*\right\|,\left\|\my\mathbf{Y}^{\top}\right\|_{\mathrm{F}} \leqslant \gamma\left\|\mathbf{Y}^* \mathbf{Y}^{* \top}\right\|_{\mathrm{F}}
\end{array}\right\}, \\
&\mathcal{R}_3^{\prime}\defeq \left\{\mathbf{Y} \in \mathbb{R}_*^{n \times r} \middle| \|\overline{\operatorname{grad} H([\mathbf{Y}])}\|_{\mathrm{F}}>\alpha \mu \sigma_r^3\left(\mathbf{Y}^*\right) /\left(4 \kappa^*\right),\|\mathbf{Y}\| \leqslant \beta\left\|\mathbf{Y}^*\right\|,\left\|\my\mathbf{Y}^{\top}\right\|_{\mathrm{F}} \leqslant \gamma\left\|\mathbf{Y}^* \mathbf{Y}^{* \top}\right\|_{\mathrm{F}}\right\}, \\
&\mathcal{R}_3^{\prime \prime}\defeq \left\{\mathbf{Y} \in \mathbb{R}_*^{n \times r} \middle| \| \mathbf{Y}\|>\beta\| \mathbf{Y}^*\|,\| \mathbf{Y} \mathbf{Y}^{\top}\left\|_{\mathrm{F}} \leqslant \gamma\right\| \mathbf{Y}^* \mathbf{Y}^{* \top} \|_{\mathrm{F}}\right\}, \\
&\mathcal{R}_3^{\prime \prime \prime}\defeq \left\{\mathbf{Y} \in \mathbb{R}_*^{n \times r} \middle| \|\mathbf{Y} \mathbf{Y}^{\top}\right\|_{\mathrm{F}}>\gamma\left\|\mathbf{Y}^* \mathbf{Y}^{* \top}|_{\mathrm{F}}\right\},
\end{split}
\end{equation}

\subsection{Some auxiliary inequalities}\label{Append:Auxiliary Lemma:Inequality}

In this section, we collect results from prior work that will be useful for us. First, we provide the characterization of and results about the geodesic distance on $\NM_{r+}^n$ from \cite{Massart2020Quotient} and \cite{luo2022nonconvex}. 

\begin{lemma}[Lemma 2, \cite{luo2022nonconvex}]\label{prop: def for d}
Let $\mathbf{Y}_1, \mathbf{Y}_2 \in \mathbb{R}_*^{n \times r}$, and $\mathbf{Q}_U \boldsymbol{\Sigma} \mathbf{Q}_V^{\top}$ be the SVD of $\mathbf{Y}_1^{\top} \mathbf{Y}_2$. Denote $\mathbf{Q}^*=\mathbf{Q}_V \mathbf{Q}_U^{\top}$. Then
\begin{enumerate}
    \item  $\mathbf{Y}_2 \mathbf{Q}^*-\mathbf{Y}_1 \in \mathcal{H}_{\mathbf{Y}_1} \overline{\NM}_{r+}^n, \mathbf{Q}^*$ is one of the best orthogonal matrices aligning $\mathbf{Y}_1$ and $\mathbf{Y}_2$, i.e., $\mathbf{Q}^* \in \arg \min _{\mathbf{Q} \in \mathbb{O}_r}\left\|\mathbf{Y}_2 \mathbf{Q}-\mathbf{Y}_1\right\|_{\mathrm{F}}$ and the geodesic distance between $\left[\mathbf{Y}_1\right]$ and $\left[\mathbf{Y}_2\right]$ is $d\left(\left[\mathbf{Y}_1\right],\left[\mathbf{Y}_2\right]\right)=\left\|\mathbf{Y}_2 \mathbf{Q}^*-\mathbf{Y}_1\right\|_{\mathrm{F}}$;
    \item if $\mathbf{Y}_1^{\top} \mathbf{Y}_2$ is nonsingular, then $\mathbf{Q}^*$ is unique and the Riemannian logarithm $\log _{\left[\mathbf{Y}_1\right]}\left[\mathbf{Y}_2\right]$ is uniquely defined and its horizontal lift at $\mathbf{Y}_1$ is given by $\overline{\log _{\left[\mathbf{Y}_1\right]}\left[\mathbf{Y}_2\right]}=\mathbf{Y}_2 \mathbf{Q}^*-\mathbf{Y}_1$; moreover, the unique minimizing geodesic from $\left[\mathbf{Y}_1\right]$ to $\left[\mathbf{Y}_2\right]$ is $\left[\mathbf{Y}_1+t\left(\mathbf{Y}_2 \mathbf{Q}^*-\mathbf{Y}_1\right)\right]$ for $t \in[0,1]$. 
\end{enumerate}
\end{lemma}

\begin{lemma}[Lemma 12 in \cite{luo2022nonconvex}]\label{lem:d([Y1,Y2])}
For any $\mathbf{Y}_1, \mathbf{Y}_2 \in \mathbb{R}_*^{n \times r}$, we have
\begin{align}\label{App:eq:d([Y1,Y2])}
    d^2\left(\left[\mathbf{Y}_1\right],\left[\mathbf{Y}_2\right]\right) \leqslant \frac{1}{2(\sqrt{2}-1) \sigma_r^2\left(\mathbf{Y}_2\right)}\left\|\mathbf{Y}_1 \mathbf{Y}_1^{\top}-\mathbf{Y}_2 \mathbf{Y}_2^{\top}\right\|_{\mathrm{F}}^2
\end{align}
and
\begin{align}\label{app:eq:(Y1-Y2Q)}
\left\|\left(\mathbf{Y}_1-\mathbf{Y}_2 \mathbf{Q}\right)\left(\mathbf{Y}_1-\mathbf{Y}_2 \mathbf{Q}\right)^{\top}\right\|_{\mathrm{F}}^2 \leqslant 2\left\|\mathbf{Y}_1 \mathbf{Y}_1^{\top}-\mathbf{Y}_2 \mathbf{Y}_2^{\top}\right\|_{\mathrm{F}}^2 \text {, }
\end{align}
where $\mathbf{Q}=\arg\min_{\mathbf{O} \in \mathbb{O}_r}\left\|\mathbf{Y}_1-\mathbf{Y}_2 \mathbf{O}\right\|_{\mathrm{F}}$.

In addition, for any $\mathbf{Y}_1, \mathbf{Y}_2 \in \mathbb{R}_*^{n \times r}$ obeying $d\left(\left[\mathbf{Y}_1\right],\left[\mathbf{Y}_2\right]\right) \leqslant \frac{1}{3} \sigma_r\left(\mathbf{Y}_2\right)$, we have
\begin{align}\label{app:eq:||Y1Y1-Y2Y2||}
\left\|\mathbf{Y}_1 \mathbf{Y}_1^{\top}-\mathbf{Y}_2 \mathbf{Y}_2^{\top}\right\|_{\mathrm{F}} \leqslant \frac{7}{3}\left\|\mathbf{Y}_2\right\| d\left(\left[\mathbf{Y}_1\right],\left[\mathbf{Y}_2\right]\right)
\end{align}

\end{lemma}

Given any $\mathbf{Y} \in \mathbb{R}_*^{n \times r}$ and $x>0$, let $B_x([\mathbf{Y}]):=\left\{\left[\mathbf{Y}_1\right]: d\left(\left[\mathbf{Y}_1\right],[\mathbf{Y}]\right)<x\right\}$ be the geodesic ball centered at $[\mathbf{Y}]$ with radius $x$. For any Riemannian manifold, there exists a convex geodesic ball at every point (Chapter 3.4, \cite{do1992riemannian}). The next result quantifies the convexity radius around a point $[\mathbf{Y}]$ in the manifold $\NM_{r+}^n$.

\begin{lemma}[Theorem 2, \cite{luo2022nonconvex}]
     Given any $\mathbf{Y} \in \mathbb{R}_*^{n \times r}$, the geodesic ball centered at $[\mathbf{Y}]$ with radius $x \leqslant r_{\mathbf{Y}}:=\sigma_r(\mathbf{Y}) / 3$ is geodesically convex. In fact, for any two points $\left[\mathbf{Y}_1\right],\left[\mathbf{Y}_2\right] \in B_x([\mathbf{Y}])$, there is a unique shortest geodesic joining them, which is entirely contained in $B_x([\mathbf{Y}])$.

\end{lemma}


Finally, we provide some useful inequalities. 

\begin{lemma}[Proposition 2 in \cite{luo2021geometric}]\label{lemma-opt: ||YthetaY+thetaYY||}
Let $\mathbf{Y} \in \mathbb{R}_*^{n \times r}$, and let $\mathbf{X}=\mathbf{Y} \mathbf{Y}^{\top}$. Then $2 \sigma_r^2(\mathbf{Y})\left\|\theta_{\mathbf{Y}}\right\|_{\mathrm{F}}^2 \leqslant$ $\left\|\mathbf{Y} \theta_{\mathbf{Y}}^{\top}+\theta_{\mathbf{Y}} \mathbf{Y}^{\top}\right\|_{\mathrm{F}}^2 \leqslant 4 \sigma_1^2(\mathbf{Y})\left\|\theta_{\mathbf{Y}}\right\|_{\mathrm{F}}^2$ holds for all $\theta_{\mathbf{Y}} \in \mathcal{H}_{\mathbf{Y}} \overline{\mathcal{M}}_{r+}^q$.
\end{lemma}

\begin{lemma}
\label{lem:Frobenius norm inequality}
For $\ma\in\R^{m\times n}$, $\mb\in \R^{n\times n}$ where $\mb$ is positive semi-definite, we have
\begin{align}
    \|\ma\|_\mathrm{F}\cdot\sigma_n(\mb)\le\|\ma\mb\|_\mathrm{F}\le \|\ma\|_\mathrm{F}\cdot\sigma_1(\mb)
\end{align}
\end{lemma}
\begin{proof}
    When $m=1$, this statement is direct by the definition of the Frobenius norm. When $m>1$, we denote $\ma_i$ to be the $i^{\mathrm{th}}$ row of $\ma$, and then 
    \begin{align*}
        \|\ma\mb\|_\mathrm{F}^2=\sum_{i=1}^m \|\ma_i\mb\|_\mathrm{F}^2\le \sum_{i=1}^m\|\ma_i\|_\mathrm{F}\cdot\sigma_1(\mb)  = \|\ma\|_\mathrm{F}\cdot\sigma_1(\mb)
    \end{align*}
    Similarly, 
    \begin{align*}
        \|\ma\mb\|_\mathrm{F}^2=\sum_{i=1}^m \|\ma_i\mb\|_\mathrm{F}^2\ge \sum_{i=1}^m\|\ma_i\|_\mathrm{F}\cdot\sigma_n(\mb)  = \|\ma\|_\mathrm{F}\cdot\sigma_n(\mb)
    \end{align*}
\end{proof}

\subsection{Proof of Results}\label{sec-opt:global convergence}




In this section, we provide the proofs for Theorems \ref{thm:R1intext}, \ref{Thm:FOSP in text}, \ref{thm-R2:Region with Negative Eigenvalue in the Riemannian Hessian in text}, and \ref{thm:R3}.

\geodesicallyConvex*
\begin{proof}
    Denote by $\mathbf{Q}$ the best orthogonal matrix that aligns $\mathbf{Y}$ and $\mathbf{Y}^*$. Then by the assumption on $\mathbf{Y}\in \mathcal{R}_1$ as defined in \eqref{eq-def:R}, we have
\begin{align}\label{eq:Thm2-||Y-Y*Q||<sigma_r(Y)}
    \left\|\mathbf{Y}-\mathbf{Y}^* \mathbf{Q}\right\| \leqslant\left\|\mathbf{Y}-\mathbf{Y}^* \mathbf{Q}\right\|_{\mathrm{F}}=d\left([\mathbf{Y}],\left[\mathbf{Y}^*\right]\right) \leqslant \mu \sigma_r\left(\mathbf{Y}^*\right) / \kappa^*.
\end{align}
Thus
\begin{equation}\label{eq-Thm2:sigma(Y)bound}
    \begin{split}
    &\sigma_r(\mathbf{Y})=\sigma_r\left(\mathbf{Y}-\mathbf{Y}^* \mathbf{Q}+\mathbf{Y}^* \mathbf{Q}\right) \geqslant \sigma_r\left(\mathbf{Y}^*\right)-\left\|\mathbf{Y}-\mathbf{Y}^* \mathbf{Q}\right\| \stackrel{\text{\eqref{eq:Thm2-||Y-Y*Q||<sigma_r(Y)}}}{\geqslant}\left(1-\mu / \kappa^*\right) \sigma_r\left(\mathbf{Y}^*\right) \\
&\sigma_1(\mathbf{Y})=\sigma_1\left(\mathbf{Y}-\mathbf{Y}^* \mathbf{Q}+\mathbf{Y}^* \mathbf{Q}\right) \leqslant \sigma_1\left(\mathbf{Y}^*\right)+\left\|\mathbf{Y}-\mathbf{Y}^* \mathbf{Q}\right\| \stackrel{\text{\eqref{eq:Thm2-||Y-Y*Q||<sigma_r(Y)}}}{\leqslant} \sigma_1\left(\mathbf{Y}^*\right)+\mu \sigma_r\left(\mathbf{Y}^*\right) / \kappa^*
    \end{split}
\end{equation}
where the first inequalities follow from Weyl's theorem \cite{stewart1998matrix}. Then,
\begin{center}
\begin{align*}
\overline{\operatorname{Hess} H([\mathbf{Y}])}\left[\theta_{\mathbf{Y}}, \theta_{\mathbf{Y}}\right] &= \left\|\mathbf{Y} \theta_{\mathbf{Y}}^{\top}+\theta_{\mathbf{Y}} \mathbf{Y}^{\top}\right\|_{\mathrm{F}}^2+2\left\langle\mathbf{Y} \mathbf{Y}^{\top}-\tx, \theta_{\mathbf{Y}} \theta_{\mathbf{Y}}^{\top}\right\rangle &[\text{\eqref{eq-def:gradHessian}}]\\
&\geqslant 2 \sigma_r^2(\mathbf{Y})\left\|\theta_{\mathbf{Y}}\right\|_{\mathrm{F}}^2+2\left\langle\mathbf{Y} \mathbf{Y}^{\top}-\tx, \theta_{\mathbf{Y}} \theta_{\mathbf{Y}}^{\top}\right\rangle &[\text { Lemma  \ref{lemma-opt: ||YthetaY+thetaYY||}}]\\
&= 2\sigma_r^2(\mathbf{Y})\left\|\theta_{\mathbf{Y}}\right\|_{\mathrm{F}}^2 + 2\left\langle\mathbf{Y} \mathbf{Y}^{\top}, \theta_{\mathbf{Y}}\theta_{\mathbf{Y}}^\top \right\rangle -2\left\langle\mys\my^{*\top},\theta_{\mathbf{Y}}\theta_{\mathbf{Y}}^\top\right\rangle \\
&\ \ \ \ -2\left\langle \mz\mz^\top, \theta_{\mathbf{Y}}\theta_{\mathbf{Y}}^\top \right\rangle & [\tx=\mys\my^{*\top}+\mz\mz^\top] \\
&\geqslant 2 \sigma_r^2(\mathbf{Y})\left\|\theta_{\mathbf{Y}}\right\|_{\mathrm{F}}^2-2 \left\|\mathbf{Y} \mathbf{Y}^{\top}-\mathbf{Y}^*\mathbf{Y}^{*\top}\right\|\left\|\theta_{\mathbf{Y}}\theta_{\mathbf{Y}}^\top\right\|_{\rF} & \\
&\ \ \ \ -2\|\mz\mz^\top\|\|\theta_{\mathbf{Y}}\theta_{\mathbf{Y}}^\top\|_{\rF} &[\langle A,B \rangle \le \|A\| \|B\|_{\rF}]\\
&\geqslant 2 \sigma_r^2(\mathbf{Y})\left\|\theta_{\mathbf{Y}}\right\|_{\mathrm{F}}^2-2 \left\|\mathbf{Y} \mathbf{Y}^{\top}-\mathbf{Y}^*\mathbf{Y}^{*\top}\right\|\left\|\theta_{\mathbf{Y}}\right\|_{\rF}^2 & \\
&\ \ \ \ -2\|\mz\mz^\top\|\|\theta_{\mathbf{Y}}\|_{\rF}^2 &[\|\theta_{\mathbf{Y}}\theta_{\mathbf{Y}}^\top\|_{\rF} = \|\theta_{\mathbf{Y}}\|_\rF^2]\\
&\geqslant 2\left(1-\frac{\mu}{\kappa^*}\right)^2 \sigma_r^2\left(\mathbf{Y}^*\right)\left\|\theta_{\mathbf{Y}}\right\|_{\mathrm{F}}^2 -2\|\mz\mz^\top\|\|\theta_{\mathbf{Y}}\|_{\rF}^2  & \\
&\ \ \ \ -2 \left\|\mathbf{Y} \mathbf{Y}^{\top} -\mathbf{Y}^*\mathbf{Y}^{*\top}\right\|\left\|\theta_{\mathbf{Y}}\right\|_{\rF}^2 &[\text{\eqref{eq-Thm2:sigma(Y)bound}}] \\
&\geqslant 2\left(1-\frac{\mu}{\kappa^*}\right)^2 \sigma_r^2\left(\mathbf{Y}^*\right)\left\|\theta_{\mathbf{Y}}\right\|_{\mathrm{F}}^2 -2\|\mz\mz^\top\|\|\theta_{\mathbf{Y}}\|_{\rF}^2 & \\
&\ \ \ \ -2 \cdot \frac{7}{3}\left\|\mathbf{Y}^*\right\| \frac{\mu \sigma_r\left(\mathbf{Y}^*\right)}{\kappa^*}\left\|\theta_{\mathbf{Y}}\right\|_{\mathrm{F}}^2 &[ \text {Lemma \ref{lem:d([Y1,Y2])}}, \mathbf{Y} \in \mathcal{R}_1 ]\\
&= 2\left(1-\frac{\mu}{\kappa^*}\right)^2 \sigma_r^2\left(\mathbf{Y}^*\right)\left\|\theta_{\mathbf{Y}}\right\|_{\mathrm{F}}^2 -2 \cdot \frac{7}{3}\left\|\mathbf{Y}^*\right\| \frac{\mu \sigma_r\left(\mathbf{Y}^*\right)}{\kappa^*}\left\|\theta_{\mathbf{Y}}\right\|_{\mathrm{F}}^2 & \\
&\ \ \ \ -2\sigma_{r+1}(\tx)\|\theta_{\mathbf{Y}}\|_{\rF}^2 &[\|\mz\mz^\top\|=\sigma_{r+1}(\tx)]\\
&=\left(\left(2\left(1-\frac{\mu}{\kappa^*}\right)^2-\frac{14}{3} \mu\right) \sigma_r\left(\tx\right)-2\sigma_{r+1}(\tx)\right)\|\theta_\mathbf{Y}\|_{\mathrm{F}}^2 &\left[\kappa^*=\frac{\|\mys\|}{\sigma_r(\mys)}\right]
\end{align*}
\end{center}

Likewise,
\begin{center}
\begin{align*}
    \overline{\operatorname{Hess} H([\mathbf{Y}])}\left[\theta_{\mathbf{Y}}, \theta_{\mathbf{Y}}\right]
    =&\left\|\mathbf{Y} \theta_{\mathbf{Y}}^{\top}+\theta_{\mathbf{Y}} \mathbf{Y}^{\top}\right\|_{\mathrm{F}}^2+2\left\langle\mathbf{Y} \mathbf{Y}^{\top}-\tx, \theta_{\mathbf{Y}} \theta_{\mathbf{Y}}^{\top}\right\rangle &[ \text{\eqref{eq-def:gradHessian}}]\\
    \le& 4 \sigma_1^2(\mathbf{Y})\left\|\theta_{\mathbf{Y}}\right\|_{\mathrm{F}}^2+2\left\langle\mathbf{Y} \mathbf{Y}^{\top}-\tx, \theta_{\mathbf{Y}} \theta_{\mathbf{Y}}^{\top}\right\rangle &[\text { Lemma } \ref{lemma-opt: ||YthetaY+thetaYY||}]\\
    \le& 4 \sigma_1^2(\mathbf{Y})\left\|\theta_{\mathbf{Y}}\right\|_{\mathrm{F}}^2+2\left\langle\mathbf{Y} \mathbf{Y}^{\top}-\mathbf{Y}^*\mathbf{Y}^{*\top}, \theta_{\mathbf{Y}} \theta_{\mathbf{Y}}^{\top}\right\rangle &[\text {$\tx - \mathbf{Y}^*\mathbf{Y}^{*\top}$ is PSD} ]\\
    \leqslant& 4 \sigma_1^2(\mathbf{Y})\left\|\theta_{\mathbf{Y}}\right\|_{\mathrm{F}}^2+2\left\|\mathbf{Y} \mathbf{Y}^{\top}-\mathbf{Y}^*\mathbf{Y}^{*\top}\right\|\left\|\theta_{\mathbf{Y}}\right\|_{\mathrm{F}}^2 & \\
    \leqslant& 4 \sigma_1^2(\mathbf{Y})\left\|\theta_{\mathbf{Y}}\right\|_{\mathrm{F}}^2 + 2\left\|\mathbf{Y} \mathbf{Y}^{\top}-\mathbf{Y}^*\mathbf{Y}^{*\top}\right\|_{\mathrm{F}}\left\|\theta_{\mathbf{Y}}\right\|_{\mathrm{F}}^2 & \\
    \leqslant& 4\left(\sigma_1\left(\mathbf{Y}^*\right)+\frac{\mu \sigma_r\left(\mathbf{Y}^*\right)}{\kappa^*}\right)^2\left\|\theta_{\mathbf{Y}}\right\|_{\mathrm{F}}^2 \\
    &+ 2\left\|\mathbf{Y} \mathbf{Y}^{\top}-\mathbf{Y}^*\mathbf{Y}^{*\top}\right\|_{\mathrm{F}}\left\|\theta_{\mathbf{Y}}\right\|_{\mathrm{F}}^2 &[ \text{\eqref{eq-Thm2:sigma(Y)bound}} ] \\
    \leqslant& \left(4\left(\sigma_1\left(\mathbf{Y}^*\right)+\frac{\mu \sigma_r\left(\mathbf{Y}^*\right)}{\kappa^*}\right)^2+\frac{14}{3} \mu \sigma_r^2\left(\mathbf{Y}^*\right)\right)\left\|\theta_{\mathbf{Y}}\right\|_{\mathrm{F}}^2 &[\text {Lemma \ref{lem:d([Y1,Y2])}}]
\end{align*}
\end{center}

From the above we conclude that when $\mu$ is chosen such that 
\[
    \left(2\left(1-\frac{\mu}{\kappa^*}\right)^2- \frac{14}{3} \mu\right) \sigma_r\left(\tx\right)-2\sigma_{r+1}(\tx)>0,
\]
we have $H([\my])$ in \eqref{eq:optimizationRiemannian} is geodesically strongly convex and smooth in $\mathcal{R}_1$ as $\mathcal{R}_1$ is a geodesically convex
set by \cite{luo2022nonconvex}. Note that this is equivalent to 
\[
    \left(\left(1-\frac{\mu}{\kappa^*}\right)^2- \frac{7}{3} \mu\right) > \frac{\sigma_{r+1}(\tx)}{\sigma_{r}(\tx)}.
\]
Then note as $\mu \to 0$, the left hand side approaches $1$ and the inequality becomes true as $\sigma_{r}(\tx) > \sigma_{r+1}(\tx)$.
\end{proof}

\begin{remark}
    Compared with the bound in Theorem 8 of \cite{luo2022nonconvex}, the smoothness and geodesically strongly convexity are as follows,
    \begin{align*}
        \begin{aligned}
& \sigma_{\min }(\overline{\operatorname{Hess} H([\mathbf{Y}])}) \geqslant\left(2\left(1-\mu / \kappa^*\right)^2-(14 / 3) \mu\right) \sigma_r^2\left(\mathbf{Y}^*\right), \\
& \sigma_{\max }(\overline{\operatorname{Hess} H([\mathbf{Y}])}) \leqslant 4\left(\sigma_1\left(\mathbf{Y}^*\right)+\mu \sigma_r\left(\mathbf{Y}^*\right) / \kappa^*\right)^2+14 \mu \sigma_r^2\left(\mathbf{Y}^*\right) / 3 .
\end{aligned}
    \end{align*}
    There is an extra term $-2\sigma_{r+1}(\tx)$ in our lower bound of the strong convexity because even if $d([\my],[\mys])$ is small, $\tx-\my\my^\top$ is not close to $\0$, which leads to the extra error term.
\end{remark}

In the next three theorems, we show that for $\mathbf{Y} \notin \mathcal{R}_1$, either the Riemannian Hessian evaluated at $\mathbf{Y}$ has a large negative eigenvalue, or the norm of the Riemannian gradient is large. Let us recall that $\mathbf{Y} = \mU\mD\mV^\top$, $\mathbf{Y}^* = \mU^*\boldsymbol{\Sigma}^{*1/2}$. Also, recall $\tx=\overline{\mathbf{U}} \boldsymbol{\Sigma} \overline{\mathbf{U}}^{\top}$, and $\mLambda = \boldsymbol{\Sigma}^{1/2}$.

\fosp*
\begin{proof}
    From \eqref{eq-def:gradHessian}, the gradient can be written down as,
    \begin{align*}
         \overline{\operatorname{grad} H([\mathbf{Y}])}&=2\left(\mathbf{Y} \mathbf{Y}^{\top}-\tx\right) \mathbf{Y}=2\left(\mU\mD\mV^\top(\mU\mD\mV^\top)^\top-\tx\right)\mU\mD\mV^\top\\
         &=2\left(\mU\mD^3\mV^\top-\tx\mU\mD\mV^\top \right)
    \end{align*}
    Therefore, whenever $\overline{\operatorname{grad} H([\mathbf{Y}])}=\0$, we have $\mU\mD^3\mV^\top-\tx\mU\mD\mV^\top=\0$. Since both $\mV$ and $\mD$ are of full rank, the condition is equivalent to 
    \begin{align}\label{eq:FOSP condition}
        \mU\mD^2-\tx\mU=\0
    \end{align}
    Since $\mD^2$ is also a diagonal matrix, to satisfy \eqref{eq:FOSP condition}, the columns of $\mU$ have to be the eigenvectors of $\tx$, and the diagonal of $\mD^2$ has to be the eigenvalues of $\tx$. This completes the proof.
\end{proof}

Before we can prove the next main result, Theorem \ref{thm-R2:Region with Negative Eigenvalue in the Riemannian Hessian in text}, we need to discuss some of the assumptions. Specifically, we want to quantify the statement $\alpha$ is sufficiently small. 

\begin{assumption}[Parameters Settings]\label{assum:R2:Region with Negative Eigenvalue in the Riemannian Hessian} $ $
Denote $\eee,\e$ and $\ee$ to be some error terms.
\[
    \eee\defeq\frac{\alpha \mu \sigma_r^3\left(\mathbf{Y}^*\right)} {2\sqrt{2} \kappa^*\sigma_{r+1}(\mLambda)}, \ \ \ \ 
    \e=\frac{\eee}{\sqrt{2}},  \ \ \ \ \text{ and } \ \ \ \ 
    \ee=\e\cdot\sigma_{r+1}(\mLambda)
\]

Note that $e_1, e_2, e_3 \to 0$ and $\alpha \to 0$. Hence, pick $\alpha$ small enough such that the following are true. 
    \begin{enumerate}[nosep,leftmargin=*]
        \item $\sigma_r^2(\mLambda) - 2\eee -\sigma_{r+1}^2(\mLambda)> 0$.
        \item $\sigma_r^2(\mLambda) \left(1-\frac{\eee^2}{\left|\sigma_r^2(\mLambda) - \eee-\sigma_{r+1}^2(\mLambda)\right|^2}\right)-\eee-\sigma_{r+1}^2(\mLambda)> 0$.\\
   \item $(\alpha-2(\sqrt{2}-1)) \sigma_r^2\left(\mathbf{Y}^*\right)+6\frac{\alpha^2 \sigma_r^4\left(\mathbf{Y}^*\right)\sigma^2_{r+1}(\mLambda) /16}{\left|\sigma_{r}^2(\mLambda) - \e-\sigma_{r+1}^2(\mLambda)\right|^2 }<0$.
    \end{enumerate}
\end{assumption}

Note that for the first two, we have that as $\alpha \to 0$. They converge to $\sigma_r^2(\mathbf{\sigma}) - \sigma_{r+1}^2(\mathbf{\sigma})$ which is positive due to the eigengap assumption. For the last condition, we have that as $\alpha \to 0$, it converges to $-2(\sqrt{2}-1)\sigma_r^2\left(\mathbf{Y}^*\right)$ which is negative.

Hence, notice that this assumption is only related to the eigengap assumption $\sigma_r(\mLambda)$ and $\sigma_{r+1}(\mLambda)$ in Assumption \ref{assum-riemannian:eigengap in text}. As soon as $\alpha$ is small enough, Assumption \ref{assum:R2:Region with Negative Eigenvalue in the Riemannian Hessian} is satisfied. 

\begin{theorem}[Region with Negative Eigenvalue in the Riemannian Hessian of Equation~\ref{eq:optimization} (formal Theorem \ref{thm-R2:Region with Negative Eigenvalue in the Riemannian Hessian in text})]\label{thm-R2:Region with Negative Eigenvalue in the Riemannian Hessian}
Assume that Assumption \ref{assum-riemannian:eigengap in text} holds. Given any $\mathbf{Y} \in \R_*^{n \times r}$ , let $\theta_{\mathbf{Y}}^1 = [\0,\0,\dots,\0,\va,\0,\dots,\0]\mV^\top$ where $\va$ such that
\begin{align}\label{eq-def:va}
        \va = \argmax_{\va:\my^\top\va=\0} \frac{\va^\top \tx\va}{\|\va\|^2}
    \end{align}
and $[\0,\0,\dots,\0,\va,\0,\dots,\0]\in\R^{n\times r}$ such that the $\tilde{i}^{\mathrm{th}}$ columns is $\va$ and other columns are $\0$ where 
\begin{align}\label{eq-def:tildei}
    \tilde{i}\defeq\argmin_{j\in [r]} \mD_{jj}.
\end{align}
 
Denote $\theta_{\mathbf{Y}}^2=\mathbf{Y}-\mathbf{Y}^* \mathbf{Q}$, where $\mathbf{Q} \in \mathbb{O}_r$ is the best orthogonal matrix aligning $\mathbf{Y}^*$ and $\mathbf{Y}$. We choose $\theta_\my$ to be either $\theta_\my^1$ or $\theta_\my^2$. Then 
\begin{equation*}
    \begin{split}
        \overline{\operatorname{Hess} H([\mathbf{Y}])}\left[\theta_{\mathbf{Y}}, \theta_{\mathbf{Y}}\right] \leqslant
        &\min\left\{-\frac{\sigma_{r+1}^2(\mLambda)}{2}\|\theta_\my\|^2,\right.\\
        &-2\left(\sigma_r^2(\mLambda)\left(1-\frac{\eee^2}{\left|\sigma_r^2(\mLambda) - \eee-\sigma_{r+1}^2(\mLambda)\right|^2}\right)-\eee-\sigma_{r+1}^2(\mLambda)\right)\|\theta_\my\|^2,\\
        &\left.\left((\alpha-2(\sqrt{2}-1)) \sigma_r^2\left(\mathbf{Y}^*\right)+6\frac{\alpha^2 \sigma_r^4\left(\mathbf{Y}^*\right)\sigma^2_{r+1}(\sigma) /16}{\left|\sigma_{r}^2(\mLambda) - \e-\sigma_{r+1}^2(\mLambda)\right|^2}\right)\left\|\theta_{\mathbf{Y}}\right\|_{\mathrm{F}}^2\right\}
    \end{split}
\end{equation*}
In particular, if $\alpha$ and $\mu$ satisfies Assumption \ref{assum:R2:Region with Negative Eigenvalue in the Riemannian Hessian}, we have $\overline{\operatorname{Hess} H([\mathbf{Y}])}$ has at least one negative eigenvalue and $\theta_{\mathbf{Y}}$ is the escaping direction.
\end{theorem}
\begin{proof}
    
    By the definition of $\va$, $\va\in \mathrm{Span}\{\overline{\mU}_{1,\dots,r+1}\}$. This is because the null space of $\mathbf{Y}$ has dimension $n-r$. Hence, its intersection with a dimension $r+1$ space has a dimension of at least 1. 

    Using the SVD decomposition of $\my$, we have, $\mU^\top \va =\0$. Then, by using \eqref{eq-def:gradHessian}, we have
    \begin{align*}
        \overline{\operatorname{Hess} H([\mathbf{Y}])}\left[\theta_{\mathbf{Y}}^1, \theta_{\mathbf{Y}}^1\right] &= \left\|\mathbf{Y} (\theta_{\mathbf{Y}}^1)^{\top}+\theta_{\mathbf{Y}}^1 \mathbf{Y}^{\top}\right\|_{\mathrm{F}}^2+2\left\langle\mathbf{Y} \mathbf{Y}^{\top}-\tx, \theta_{\mathbf{Y}}^1 (\theta_{\mathbf{Y}}^1)^{\top}\right\rangle &[\text{\eqref{eq-def:gradHessian}}]\\
        &= \left\|\mathbf{Y} (\theta_{\mathbf{Y}}^1)^{\top}+\theta_{\mathbf{Y}}^1 \mathbf{Y}^{\top}\right\|_{\mathrm{F}}^2 - 2\left\langle\tx, \theta_{\mathbf{Y}}^1 (\theta_{\mathbf{Y}}^1)^{\top}\right\rangle &[\mathbf{Y}^\top \va = 0]\\
        &=2\langle \my^\top\my,(\theta_{\my}^1)^\top\theta_{\my}^1 \rangle + 2\langle \mathbf{Y}(\theta^1_{\mathbf{Y}})^\top, \theta^1_{\mathbf{Y}}\mathbf{Y}^\top\rangle - 2\left\langle\tx, \theta_{\mathbf{Y}}^1 (\theta_{\mathbf{Y}}^1)^{\top}\right\rangle &[\|A\|_\rF^2 = \langle A, A\rangle] \\
        &=2\langle \my^\top\my,(\theta_{\my}^1)^\top\theta_{\my}^1 \rangle - 2\left\langle\tx, \theta_{\mathbf{Y}}^1 (\theta_{\mathbf{Y}}^1)^{\top}\right\rangle &[\mathbf{Y}^\top \va = 0]\\
        &= 2\langle \mV\mD^2\mV^\top,(\theta_{\my}^1)^\top\theta_{\my}^1 \rangle - 2\left\langle\tx, \theta_{\mathbf{Y}}^1 (\theta_{\mathbf{Y}}^1)^{\top}\right\rangle\\
        &= 2\mD^2_{\tilde{i}\tilde{i}}\|\va\|^2 - 2\va^\top \tx\va
    \end{align*}
    
    where the last equality comes from the definition of $\va$ and the fact that the $\mV^\top \mV = \I$ in $\theta_{\mathbf{Y}}^1 (\theta_{\mathbf{Y}}^1)^{\top}$. Recall $\tilde{i}=\argmin \mD_{ii}$, then 
    \begin{align}\label{eq:Hessian_theta_Y^1}
        \overline{\operatorname{Hess} H([\mathbf{Y}])}\left[\theta^1_{\mathbf{Y}}, \theta_{\mathbf{Y}}^1\right]=2\min_{i}\mD^2_{ii}\|\va\|^2 - 2\va^\top \tx \va
    \end{align}

    In the following, we separate the proof into three regimes of $\min_{i}\mD^2_{ii}$, corresponding to different escape directions.
    
    \textbf{Case 1: \big(When $\min_{i}\mD^2_{ii}<\frac{\sigma_{r+1}^2(\mLambda)}{2}$\big)}. For this case we must have that 
    \[
        \overline{\operatorname{Hess} H([\mathbf{Y}])}\left[\theta^1_{\mathbf{Y}}, \theta^1_{\mathbf{Y}}\right]\le -\frac{\sigma_{r+1}^2(\mLambda)}{2}\|\theta_\my^1\|^2.
    \]
    This is because $\va^\top \tx\va\ge \sigma_{r+1}^2(\mLambda)\|\va\|^2$ and $\|\va\|=\|\theta_\my^1\|$.
    
    \textbf{Case 2: \big(When $\min_{i}\mD^2_{ii}\ge\frac{\sigma_{r+1}^2(\mLambda)}{2}$\big)}.
    
    From the proof of Theorem \ref{Thm:FOSP in text}, the gradient condition of $\mathcal{R}_2$ can be written as 
    \begin{align*}
        \alpha \mu \sigma_r^3\left(\mathbf{Y}^*\right) /\left(4 \kappa^*\right)
        &\ge \|\overline{\operatorname{grad} H([\mathbf{Y}])}\|_{\mathrm{F}} &[\my\in \mathcal{R}_2]\\
        &=\|2\left(\mU\mD^3\mV^\top-\tx\mU\mD\mV^\top \right)\|_{\mathrm{F}}&[\text{\eqref{eq-def:gradHessian}}]\\
        &=\|2\left(\mU\mD^2-\tx\mU \right)\mD\|_{\mathrm{F}}
    \end{align*}
    Assume $\mU=\overline{\mU}\mC $ where $\mC\in \R^{n\times r}$. Since $\mU^\top\mU=\I_r$ and $\overline{\mU}^\top \overline{\mU}=\I_n$, we have $\mC^\top \mC=\I_r$. Furthermore, 
    \begin{align*}
    \|2\left(\mU\mD^2-\tx\mU \right)\mD\|_{\mathrm{F}}&=\|2\left(\overline{\mU}\mC\mD^2-\tx \overline{\mU}\mC \right)\mD\|_{\mathrm{F}}&[\mU=\overline{\mU}\mC]\\
    &= \|2\left(\overline{\mU}\mC\mD^2-\overline{\mU}\mSigma \mC \right)\mD\|_{\mathrm{F}}&[\tx=\overline{\mU}\mSigma \overline{\mU}^\top] \\
    &=2\|\left(\mC\mD^2-\mSigma \mC \right)\mD\|_{\mathrm{F}}.
    \end{align*}
    Here the third equality follows from $\overline{\mU}^\top \overline{\mU}=\I_n$. By a direct computation, the $i^{\mathrm{th}}$ column of $\left(\mC\mD^2-\mSigma \mC \right)\mD$ is $\mD_{ii}^3 \mC_i-\mD_{ii}\mSigma\mC_i$. Therefore, the gradient condition of $\mathcal{R}_2$ can be written as 
    \begin{align}\label{eq-condition:D_ii^3C_i-D_iiSigmaC_i}
        \sum_{i,j} \left( \mD_{ii}^3 \mC_{ji}-\mD_{ii}\mSigma_{jj}\mC_{ji} \right)^2\le \alpha^2 \mu^2 \sigma_r^6\left(\mathbf{Y}^*\right) /\left(4 \kappa^*\right)^2
    \end{align}
    We fix $i$ in the left hand side of \eqref{eq-condition:D_ii^3C_i-D_iiSigmaC_i}, we have
    \begin{align}\label{eq-R2:mD_ii^2-Sigma_jj}
        \sum_{j} \left( \mD_{ii}^2-\mSigma_{jj} \right)^2 \mD_{ii}^2 \mC_{ji}^2\le \alpha^2 \mu^2 \sigma_r^6\left(\mathbf{Y}^*\right) /\left(4 \kappa^*\right)^2
    \end{align}
    where $\sum_{j}\mC^2_{ji}=1$. From $\mD_{ii}^2\ge \frac{\sigma_{r+1}^2(\mLambda)}{2}$, we must have 
    \begin{align}\label{eq:min|D_ii-sigma_ii|}
        \min_{j}|\mD_{ii}^2-\mSigma_{jj}|^2\le\sum_{j} \left( \mD_{ii}^2-\mSigma_{jj} \right)^2 \mC_{ji}^2 \le \frac{\alpha^2 \mu^2 \sigma_r^6\left(\mathbf{Y}^*\right)} {(4 \kappa^*)^2\frac{\sigma_{r+1}^2(\mLambda)}{2}}.
    \end{align}
    We use \eqref{eq-R2:mD_ii^2-Sigma_jj} for the second inequality. Equation~\ref{eq:min|D_ii-sigma_ii|} is important in the proof because this essentially guarantees that $\mD_{ii}^2$ must be close to some $\mSigma_{jj}$. This is because $\frac{\alpha^2 \mu^2 \sigma_r^6\left(\mathbf{Y}^*\right)} {(4 \kappa^*)^2\frac{\sigma_{r+1}^2(\mLambda)}{2}}$ is guaranteed small according to Assumption \ref{assum:R2:Region with Negative Eigenvalue in the Riemannian Hessian}.
    
    We decompose $\mC_{\tilde{i}}$ into $\xi^1+\xi^2$ where $\xi^1_{j}=0$ for all $j\ge r+1$ and $\xi^2_{j}=0$ for all $j\in [r]$. Since $\langle \xi^1,\xi^2\rangle=0$ and $\mC^\top\mC =\I$, 
    \begin{align}\label{eq:xi_1^2+xi_2^2}
        \|\xi^1\|^2+\|\xi^2\|^2=1
    \end{align}

    In the following, we divide all the cases into different regimes based on which of the eigenvalues of $\mLambda$ is close to $\mD_{\tilde{i}\tilde{i}}$.
     
     \textbf{Case 2.1: \big(When $\frac{\sigma_{r+1}^2(\mLambda)}{2}\le\mD^2_{\tilde{i}\tilde{i}}\le \frac{\alpha \mu \sigma_r^3\left(\mathbf{Y}^*\right)} {2\sqrt{2} \kappa^*\sigma_{r+1}(\mLambda)}+\sigma_{r+1}^2(\mLambda)$\big)}.
     
     Notice that the first assumption in Assumption \ref{assum:R2:Region with Negative Eigenvalue in the Riemannian Hessian} essentially guarantees a small $\eee=\frac{\alpha \mu \sigma_r^3\left(\mathbf{Y}^*\right)} {2\sqrt{2} \kappa^*\sigma_{r+1}(\mLambda)}$.
    
    Hence, we have
    \begin{align*}
        \alpha^2 \mu^2 \sigma_r^6\left(\mathbf{Y}^*\right) /\left(4 \kappa^*\right)^2&\ge \sum_{j} \left( \mD_{\tilde{i}\tilde{i}}^2-\mSigma_{jj} \right)^2 \mD_{\tilde{i}\tilde{i}}^2 \mC_{j\tilde{i}}^2 &[\text{\eqref{eq-condition:D_ii^3C_i-D_iiSigmaC_i}}]\\ 
        &\ge  \sum_{j \le r}\left|\sigma_j^2(\mLambda) - \mD^2_{\tilde{i}\tilde{i}}\right|^2\cdot \mD_{\tilde{i}\tilde{i}}^2 \cdot \mC_{j\tilde{i}}^2 \\
        &\ge \left|\sigma_r^2(\mLambda) - \mD^2_{\tilde{i}\tilde{i}}\right|^2\cdot \mD_{\tilde{i}\tilde{i}}^2 \cdot \|\xi^1\|^2 \\
        &\ge \left|\sigma_r^2(\mLambda) - \eee -\sigma_{r+1}^2(\mLambda)\right|^2\cdot\frac{\sigma_{r+1}^2(\mLambda)}{2}\cdot \|\xi^1\|^2.
    \end{align*}
    Where in the last two inequalities, we use the condition $\frac{\sigma_{r+1}^2(\mLambda)}{2}\le\mD^2_{\tilde{i}\tilde{i}}\le \eee+\sigma_{r+1}^2(\mLambda)$ and that  $ e_1 < (\sigma_{r}^2(\mLambda) - \sigma_{r+1}^2(\mLambda))/2$ (follows from  Assumption \ref{assum:R2:Region with Negative Eigenvalue in the Riemannian Hessian}). 
    
    
    By reordering the inequality, we have
    \begin{align}\label{eq:xi1}
        \|\xi^1\|\le \frac{\eee}{\left|\sigma_r^2(\mLambda) - \eee-\sigma_{r+1}^2(\mLambda)\right|}
    \end{align}
    
    Recall that $\my=\mU\mD\mV^\top$, then $\va^\top \my=\0$ reduces to $\va^\top\mU\mD\mV^\top=\0$. Since both $\mD,\mV\in \R^{r*r}$ are full rank, then we have $\va^\top \mU=\0$, in turn $\va^\top \overline{\mU}\mC=\0$ because $\mU=\overline{\mU}\mC$. Denote $\vb^\top\defeq \va^\top \overline{\mU}$, then 
    \begin{equation}\label{eq:va to vb}
        \begin{aligned}
            \max_{\va:\my^\top\va=\0} \frac{\va^\top \tx\va}{\|\va\|^2}
        &=\max_{\va:\va^\top \overline{\mU}\mC=\0} \frac{\va^\top \tx\va}{\|\va\|^2} & \\
        &= \max_{\va:\va^\top \overline{\mU}\mC=\0} \frac{\va^\top \overline{\mU}\mLambda\overline{\mU}^\top\va}{\|\va\|^2} &[\tx=\overline{\mU}\mLambda\overline{\mU}^\top] \\
        &= \max_{\vb:\vb^\top\mC=\0} \frac{\vb^\top\mLambda\vb}{\|\vb\|^2}  &[\overline{\mU}^\top\overline{\mU}=\I]
        \end{aligned}
    \end{equation}
    Since $\va\in \mathrm{Span}\{\overline{\mU}_{1,\dots,r+1}\}$, we have $\vb_{j}=0$ for $j> r+1$. From $\vb^\top\mC=\0$, we have $\vb^\top \mC_{\tilde{i}}=0$, which can be written as $\vb^\top (\xi^1+\xi^2)=0$. 
    Since there are in total $r$ constraints in $\vb^\top\mC=\0$, there must exist a $\vb$ satisfying the constraints $\vb^\top\mC=\0$, and the norm of $\vb_{r+1:n}$ is relatively small compared with the norm of $\vb_{1:r}$. Specifically, denote $\mC_{1:r}$ to be the $1^{\mathrm{st}}$ to $r^{\mathrm{th}}$ rows of $\mC$. 
    We consider $\vb$ to be $\vb^1+\vb^2$ such that $\vb^1_i=0$ for $i>r$, and $\vb^2_i=0$ for $i\in[r]$. We discuss two cases of $\mC_{1:r}\in\R^{r*r}$ in the following.

    \textbf{Case 2.1.1: If $\mC_{1:r}$ is not full rank}.
    
    In this case, there exists $\tilde{\vb}^1\in \R^{r}$ such that $\|\tilde{\vb}^1\|>0$  and $(\tilde{\vb}^1)^\top \mC_{1:r}=\0$. Therefore, by denoting $\bar{\vb}_{1:r}=t\tilde{\vb}^1+\vb^1_{1:r}$, and $\bar{\vb}_{r+1:n}=\vb^2_{r+1:n}$. From the definition of $\bar{\vb}$ and the fact that $\vb^\top\mC=\0$, we have $\bar{\vb}^\top\mC=\0$. By letting $t\to\infty$, we have
    \begin{align}\label{eq:C_1:r not full rank}
        \max_{\vb^\top\mC=\0} \frac{\vb^\top\mLambda\vb}{\|\vb\|^2}\ge \frac{\bar{\vb}^\top\mLambda\bar{\vb}}{\|\bar{\vb}\|^2}\ge  \sigma^2_r(\mLambda)
    \end{align}
    Combining \eqref{eq:C_1:r not full rank}, \eqref{eq:Hessian_theta_Y^1} and the Assumption that $\mD^2_{\tilde{i}\tilde{i}}\le \eee+\sigma_{r+1}^2(\mLambda)$, this implies,
    \begin{align}
        \overline{\operatorname{Hess} H([\mathbf{Y}])}\left[\theta^1_{\mathbf{Y}}, \theta^1_{\mathbf{Y}}\right] \leqslant -(\sigma^2_r(\mys)-\sigma_{r+1}(\mLambda)-\eee) \|\theta^1_{\mathbf{Y}}\|_\rF^2
    \end{align}
    According to Assumption \ref{assum:R2:Region with Negative Eigenvalue in the Riemannian Hessian}, this satisfies the bound in Theorem \ref{thm-R2:Region with Negative Eigenvalue in the Riemannian Hessian} with $\theta^1_\my $ being a negative escaping direction.
    
   \textbf{Case 2.1.2 : If $\mC_{1:r}$ is full rank}.
   In this case, we denote $\vb^2=\xi^2$. Since $\mC_{1:r}$ is full rank, there exists $\vb^1$ to have $(\vb^1_{1:r})^\top \mC_{1:r}=-(\vb^2)^\top \mC$; this is because $(\vb^1_{1:r})^\top \mC_{1:r}=-(\xi^2)^\top\mC$ has in total $r$ constraints, and there are in total $r$ parameters in $\vb^1_{1:r}$. Specifically, one can choose $\vb^1$ to be $\vb^1_{1:r}=-\xi^2 \mC (\mC_{1:r})^{-1}$ to satisfy $\vb^\top\mC=\0$. 
   In addition, from the specific condition $\vb^\top \mC_{\tilde{i}}=\0$, we know that 
    \begin{align}
        \vb^1\cdot\xi^1+\|\xi^2\|^2=0
    \end{align}
    By using the Cauchy inequality, this further implies that
    \begin{align}\label{eq:|b^1|>=}
        \|\vb^1\|\ge \frac{\|\xi^2\|^2}{\|\xi^1\|}
    \end{align}
    Since we only choose a specific $\vb$ such that $\vb^\top\mC=\0$ holds, we have
    \begin{equation}
        \begin{split}
            \max_{\vb^\top\mC=\0} \frac{\vb^\top\mLambda\vb}{\|\vb\|^2}&\ge \frac{(\vb^1+\vb^2)^\top\mLambda(\vb^1+\vb^2)}{\|\vb^1+\vb^2\|^2}\\
            &=\frac{(\vb^1)^\top\mLambda \vb^1 + (\vb^2)^\top\mLambda\vb^2}{\|\vb^1\|^2+\|\vb^2\|^2}\\
            &\ge \frac{(\vb^1)^\top\mLambda \vb^1}{\|\vb^1\|^2+\|\xi^2\|^2}\\
            &\ge \frac{\|\vb^1\|^2\cdot \sigma_r^2(\mLambda)}{\|\vb^1\|^2+\|\xi^2\|^2}\\
            &\ge \frac{\frac{\|\xi^2\|^4}{\|\xi^1\|^2}\cdot \sigma_r^2(\mLambda)}{\frac{\|\xi^2\|^4}{\|\xi^1\|^2}+\|\xi^2\|^2}\\
            &=\|\xi^2\|^2  \cdot \sigma_r^2(\mLambda)
        \end{split}
    \end{equation}
    where the first equality follows from the definition of $\vb^1$ and $\vb^2$; the second inequality follows from the assumption that $\mLambda$ is PSD, and $\vb^2=\xi^2$; the third inequality follows from the fact that $\vb^1_i=0$ for $i>r$; the fourth inequality follows from \eqref{eq:|b^1|>=}; the last equality follows from \eqref{eq:xi_1^2+xi_2^2}. By using \eqref{eq:va to vb}, this can be written as    
    \begin{align}\label{eq-a^top mlambda a>= bound}
    \max_{\va:\my^\top\va=\0} \frac{\va^\top \tx\va}{\|\va\|^2}\ge \|\xi^2\|^2  \cdot \sigma_r^2(\mLambda)
    \end{align}
    
    By the definition in \eqref{eq-def:va} and \eqref{eq:Hessian_theta_Y^1}, we have
    \begin{align*}
        \overline{\operatorname{Hess} H([\mathbf{Y}])}\left[\theta^1_{\mathbf{Y}}, \theta_{\mathbf{Y}}^1\right]
        &=2\min_{i}\mD^2_{ii}\cdot\|\va\|^2 - 2\va^\top \tx\va &[\text{\eqref{eq:Hessian_theta_Y^1}}]\\
        &\le 2\mD^2_{\tilde{i}\tilde{i}}\cdot\|\va\|^2-2\sigma_r^2(\mLambda) \|\xi^2\|^2\cdot\|\va\|^2&[\text{\eqref{eq-a^top mlambda a>= bound}}]\\
        &= 2\mD^2_{\tilde{i}\tilde{i}}\cdot\|\va\|^2-2\sigma_r^2(\mLambda) (1-\|\xi^1\|^2)\cdot\|\va\|^2 &[\|\xi^1\|^2+\|\xi^2\|^2=1]\\
        &\le \left(-2\sigma_r^2(\mLambda) (1-\|\xi^1\|^2)+2\eee+2\sigma_{r+1}^2(\mLambda) \right)\|\theta^1_\my\|^2
    \end{align*}
    where the last inequality follows from $\mD^2_{\tilde{i}\tilde{i}}\le \eee+\sigma_{r+1}^2(\mLambda)$ and the fact that $\|\theta_\my^1\|=\|\va\|$. Finally, by applying \eqref{eq:xi1} to control $\|\xi^1\|$, we conclude that
    \begin{align}\label{eq:second bound}
        \overline{\operatorname{Hess} H([\mathbf{Y}])}\left[\theta^1_{\mathbf{Y}}, \theta_{\mathbf{Y}}^1\right]
        \le -2\left(\sigma_r^2(\mLambda) \left(1-\frac{\eee^2}{\left|\sigma_r^2(\mLambda) - \eee-\sigma_{r+1}^2(\mLambda)\right|^2}\right)-\eee-\sigma_{r+1}^2(\mLambda)\right)\|\theta^1_\my\|^2
    \end{align}
    According to the second assumption in Assumption \ref{assum:R2:Region with Negative Eigenvalue in the Riemannian Hessian}, \eqref{eq:second bound} guarantees an escape direction. 
    
    \textbf{Case 2.2: \big(When $\mD_{\tilde{i}\tilde{i}}^2> \eee+\sigma_{r+1}^2(\mLambda)$\big)}.
    
    Recall the first assumption in Assumption \ref{assum:R2:Region with Negative Eigenvalue in the Riemannian Hessian}, we have $\eee$ is small enough, which is viewed as an error term.  In the following, we will show that $\theta_\my^2$ is the escaping direction.
    We have
    \begin{align}\label{eq:R2-2b-D_ii|D_ii-Sigma_jj|}
        \min_{j} \mD_{\tilde{i}\tilde{i}}^2 |\mD_{\tilde{i}\tilde{i}}^2-\mSigma_{jj}|^2
        \le \mD_{\tilde{i}\tilde{i}}^2 \sum_{j} \left( \mD_{\tilde{i}\tilde{i}}^2-\mSigma_{jj} \right)^2 \mC_{j\tilde{i}}^2 
        \le \frac{\alpha^2 \mu^2 \sigma_r^6\left(\mathbf{Y}^*\right)} {(4 \kappa^*)^2}
    \end{align}
    where we use \eqref{eq-R2:mD_ii^2-Sigma_jj} in the last inequality.

Recall that Assumption \ref{assum:R2:Region with Negative Eigenvalue in the Riemannian Hessian} guarantees small  $\eee$ and $\e$ , by combining \eqref{eq:R2-2b-D_ii|D_ii-Sigma_jj|} and the assumption $\mD^2_{\tilde{i}\tilde{i}}>\sigma_{r+1}^2(\mLambda)+\eee$, we must have 
\begin{align}\label{eq:R2:2b:D_ii bound}
    \mD^2_{\tilde{i}\tilde{i}}\ge  \sigma_r^2(\mLambda)-\e
\end{align}
where $\e$ is defined in Assumption \ref{assum:R2:Region with Negative Eigenvalue in the Riemannian Hessian}. Otherwise, if $\sigma^2_{r+1}(\mLambda)+\eee<\mD^2_{\tilde{i}\tilde{i}}<\sigma^2_{r}(\mLambda)-\e$, this contradicts to \eqref{eq:R2-2b-D_ii|D_ii-Sigma_jj|}; see an illustration of this fact in Figure \ref{fig:thm proof}.

\begin{figure}
		\centering
		\includegraphics[scale=0.5]{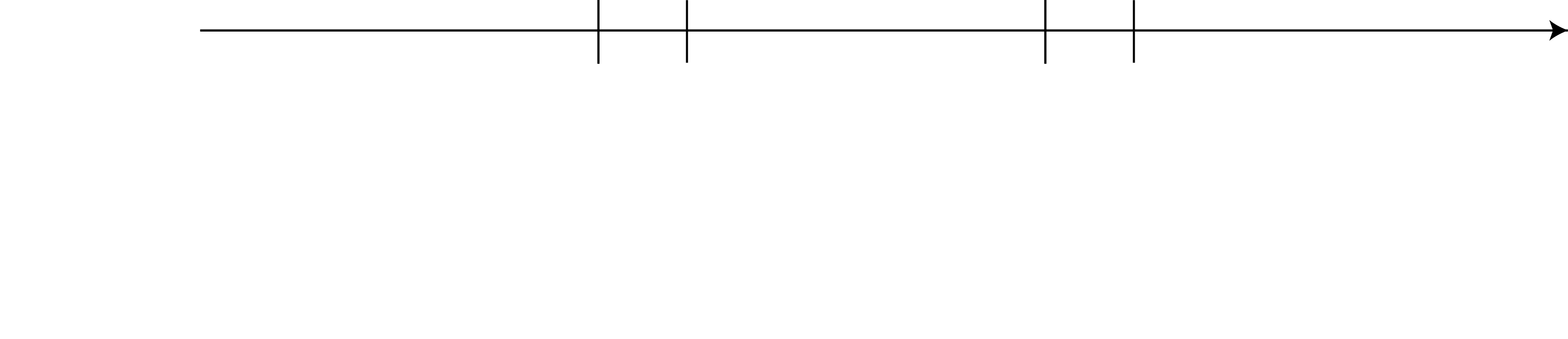}
        \put(-138,90){\tiny{$\sigma^2_{r}-\e$}}
        \put(-104,90){\tiny{$\sigma^2_{r}$}}
        \put(-213,90){\tiny{$\sigma^2_{r+1}+\eee$}}
        \put(-240,90){\tiny{$\sigma^2_{r+1}$}}
		\caption{The value of $\mD_{\tilde{i}\tilde{i}}$ must be close to some $\sigma_i(\mLambda)$ according to \eqref{eq:R2-2b-D_ii|D_ii-Sigma_jj|}. If $\mD^2_{\tilde{i}\tilde{i}}>\sigma^2_{r+1}+\eee$, then we must have $\mD^2_{\tilde{i}\tilde{i}}\ge \sigma^2_{r}-\e$.}
		\label{fig:thm proof}
	    \end{figure}
    
   In this scenario, we consider the escaping direction $\theta_\my^2$ to be $\mathbf{Y}-\mathbf{Y}^* \mathbf{Q}$. From the fact that $\mD_{ii}\ge \mD_{\tilde{i}\tilde{i}}$, we have
    %
    \begin{align*}
        \alpha^2 \mu^2 \sigma_r^6\left(\mathbf{Y}^*\right) /\left(4 \kappa^*\right)^2
        &\ge \sum_{i=1}^n\sum_{j=1}^n \left( \mD_{ii}^2-\mSigma_{jj} \right)^2 \mD_{ii}^2 \mC_{ji}^2 &[\text{\eqref{eq-condition:D_ii^3C_i-D_iiSigmaC_i}}]\\ 
        &\ge \sum_{i=1}^n\sum_{j=r+1}^n \left|\sigma_{j}^2(\mLambda) + \e -\sigma_r^2(\mLambda)\right|^2\cdot \mD_{ii}^2 \mC_{ji}^2
    \end{align*}
    where \eqref{eq:R2:2b:D_ii bound} and the first assumption in Assumption \ref{assum:R2:Region with Negative Eigenvalue in the Riemannian Hessian} guarantees the last inequality because $\e$ is small with respect to $\sigma_{r}^2(\mLambda) -\sigma_{r+1}^2(\mLambda)$. Therefore, 
    \begin{align}\label{eq-R2-2b:xi^2 bound}
        \sum_{i=1}^n\sum_{j=r+1}^n \mD_{ii}^2 \mC_{ij}^2 \le \frac{\ee^2}{\left|\sigma_{r}^2(\mLambda) - \e-\sigma_{r+1}^2(\mLambda)\right|^2} 
    \end{align}
    where $\ee$ is defined in Assumption \ref{assum:R2:Region with Negative Eigenvalue in the Riemannian Hessian}. Recall that $\ee$ is small enough, guaranteed in Assumption \ref{assum:R2:Region with Negative Eigenvalue in the Riemannian Hessian}. Also recall that $\e=\frac{\eee}{\sqrt{2}}$, which is guaranteed to be small enough as in the first assumption in Assumption \ref{assum:R2:Region with Negative Eigenvalue in the Riemannian Hessian}, so $\sigma_r(\mLambda)^2-\ee-\sigma_{r+1}^2(\mLambda)>0$.
    
    Denote $\mSigma_{(r+1):n}$ to be a diagonal matrix with only $r+1^\mathrm{th}$ to $n^{\mathrm{th}}$ eigenvalues of $\mSigma$, then we have
    %
    \begin{equation}\label{eq:<Delta_n-mxs,YY>}
        \begin{split}
            \langle \tx-\mxs,\my\my^\top\rangle
            &=\langle \tx-\mxs,\mU\mD^2\mU^\top\rangle\\
            &=\langle \tx-\mxs,\overline{\mU}\mC\mD^2\mC^\top\overline{\mU}^\top\rangle\\
            &=\langle \mSigma_{(r+1):n},\mC\mD^2\mC^\top\rangle\\
            &\le \sigma_{r+1}^2(\mLambda)\sum_{j=r+1}^n\sum_{i}\mC_{ij}^2\mD^2_{ii}\\
            &\le  \frac{\ee^2\sigma_{r+1}^2(\mLambda)}{\left|\sigma_{r}^2(\mLambda) - \e -\sigma_{r+1}^2(\mLambda)\right|^2}
        \end{split}
    \end{equation}
    where the last inequality follows from \eqref{eq-R2-2b:xi^2 bound}. \eqref{eq:<Delta_n-mxs,YY>} directly implies,
    \begin{equation}\label{eq:<Delta_n-mxs,theta_Y theta_Y>}
        \begin{split}
            \langle \tx-\mxs,\theta^2_\my(\theta_\my^2)^\top\rangle
        &=
        \langle \tx-\mxs,\my\my^\top\rangle \\
        &\le
        \frac{\ee^2\sigma_{r+1}^2(\mLambda)}{\left|\sigma_{r}^2(\mLambda) - \e -\sigma_{r+1}^2(\mLambda)\right|^2}
        \end{split}
    \end{equation}
    because $(\tx-\mxs)\mys=\0$ and $\theta_\my^2=\my-\my^*\boldsymbol{Q}$.
    
    Recall $\mxs=\mys\my^{*\top}$. A simple calculation yields 
    \begin{align}\label{eq:YthetaY+thetaYY}
        \mathbf{Y} (\theta_{\mathbf{Y}}^2)^{\top}-\mathbf{X}^*+\theta_{\mathbf{Y}}^2(\theta_{\mathbf{Y}}^2)^{\top}=\mathbf{Y} (\theta_{\mathbf{Y}}^2)^{\top}+\theta_{\mathbf{Y}}^2\mathbf{Y}^{\top}
    \end{align}
    and by using \eqref{eq-def:gradHessian}, 
    \begin{equation}\label{eq:<grad,theta>}
        \begin{split}
            \begin{aligned}
\langle\overline{\operatorname{grad} H([\mathbf{Y}])}, \theta_\mathbf{Y}^2\rangle &=\left\langle 2\left(\mathbf{Y} \mathbf{Y}^{\top}-\tx\right) \mathbf{Y}, \theta_{\mathbf{Y}}^2\right\rangle \\
&= \left\langle 2(\my\my^\top - \tx), \theta_\my^2 \my^\top\right\rangle \\
&=\left\langle\mathbf{Y} \mathbf{Y}^{\top}-\tx, \theta_{\mathbf{Y}}^2 \mathbf{Y}^{\top}+\mathbf{Y} (\theta_{\mathbf{Y}}^2)^{\top}\right\rangle &[\text{first argument is symmetric}]\\& =\left\langle\mathbf{Y} \mathbf{Y}^{\top}-\tx, \theta_{\mathbf{Y}}^2 (\theta_{\mathbf{Y}}^2)^{\top}+\mathbf{Y} \mathbf{Y}^{\top}-\mathbf{X}^*\right\rangle.
\end{aligned}
        \end{split}
    \end{equation} 
where the last equality follows from \eqref{eq:YthetaY+thetaYY}.
    \begin{align*}
        \begin{aligned}
&\overline{\operatorname{Hess} H([\mathbf{Y}])}\left[\theta^2_{\mathbf{Y}}, \theta^2_{\mathbf{Y}}\right]
=\left\|\mathbf{Y} (\theta^2_{\mathbf{Y}})^{\top}+\theta^2_{\mathbf{Y}} \mathbf{Y}^{\top}\right\|_{\mathrm{F}}^2
+2\left\langle\mathbf{Y} \mathbf{Y}^{\top}-\tx, \theta^2_{\mathbf{Y}} (\theta^2_{\mathbf{Y}})^{\top}\right\rangle &[\text{\eqref{eq-def:gradHessian}}]\\
&=\left\|\mathbf{Y} \mathbf{Y}^{\top}-\mathbf{X}^*+\theta^2_\mathbf{Y} (\theta^2_{\mathbf{Y}})^{\top}\right\|_{\mathrm{F}}^2+2\left\langle\mathbf{Y} \mathbf{Y}^{\top}-\tx, \theta^2_{\mathbf{Y}} (\theta^2_{\mathbf{Y}})^{\top}\right\rangle &[\text{\eqref{eq:YthetaY+thetaYY}}] \\
&=\left\|\theta^2_{\mathbf{Y}} (\theta^2_{\mathbf{Y}})^{\top}\right\|_{\mathrm{F}}^2+\left\|\mathbf{Y} \mathbf{Y}^{\top}-\mathbf{X}^*\right\|_{\mathrm{F}}^2+ 4\left\langle\mathbf{Y} \mathbf{Y}^{\top}-\mxs, \theta^2_{\mathbf{Y}} (\theta^2_{\mathbf{Y}})^{\top}\right\rangle\\
&-2\left\langle \tx-\mxs, \theta^2_{\mathbf{Y}} (\theta^2_{\mathbf{Y}})^{\top}\right\rangle\\
&=\left\|\theta^2_{\mathbf{Y}} (\theta^2_{\mathbf{Y}})^{\top}\right\|_{\mathrm{F}}^2-3\left\|\mathbf{Y} \mathbf{Y}^{\top}-\mathbf{X}^*\right\|_{\mathrm{F}}^2+4\left\langle\mathbf{Y} \mathbf{Y}^{\top}-\mathbf{X}^*, \mathbf{Y} \mathbf{Y}^{\top}-\mathbf{X}^*+\theta^2_{\mathbf{Y}} (\theta^2_{\mathbf{Y}})^{\top}\right\rangle\\
&-2\left\langle \tx-\mxs, \theta^2_{\mathbf{Y}} (\theta^2_{\mathbf{Y}})^{\top}\right\rangle\\
&=\left\|\theta^2_{\mathbf{Y}} (\theta^2_{\mathbf{Y}})^{\top}\right\|_{\mathrm{F}}^2-3\left\|\mathbf{Y} \mathbf{Y}^{\top}-\mathbf{X}^*\right\|_{\mathrm{F}}^2+4\left\langle\mathbf{Y} \mathbf{Y}^{\top}-\tx, \mathbf{Y} \mathbf{Y}^{\top}-\mathbf{X}^*+\theta^2_{\mathbf{Y}} (\theta^2_{\mathbf{Y}})^{\top}\right\rangle\\
&+2\left\langle \tx-\mxs, \theta^2_{\mathbf{Y}} (\theta^2_{\mathbf{Y}})^{\top}\right\rangle+4\left\langle \tx-\mxs,\mathbf{Y} \mathbf{Y}^{\top}-\mathbf{X}^*\right\rangle\\
&=\left\|\theta^2_{\mathbf{Y}} (\theta^2_{\mathbf{Y}})^{\top}\right\|_{\mathrm{F}}^2-3\left\|\mathbf{Y} \mathbf{Y}^{\top}-\mathbf{X}^*\right\|_{\mathrm{F}}^2+4\left\langle \tx-\mxs,\mathbf{Y} \mathbf{Y}^{\top}-\mathbf{X}^*\right\rangle\\
& +2\left\langle \tx-\mxs, \theta^2_{\mathbf{Y}} (\theta^2_{\mathbf{Y}})^{\top}\right\rangle+4\left\langle\overline{\operatorname{grad} H([\mathbf{Y}])}, \theta^2_{\mathbf{Y}}\right\rangle &[\text{\eqref{eq:<grad,theta>}}]
\end{aligned}
\end{align*}
This decomposes $\overline{H([\mathbf{Y}])}\left[\theta^2_{\mathbf{Y}}, \theta^2_{\mathbf{Y}}\right]$ into 2 parts, which will be bounded separately.

First, for $\left\|\theta^2_{\mathbf{Y}} (\theta^2_{\mathbf{Y}})^{\top}\right\|_{\mathrm{F}}^2-3\left\|\mathbf{Y} \mathbf{Y}^{\top}-\mathbf{X}^*\right\|_{\mathrm{F}}^2+2\left\langle \tx-\mxs, \theta^2_{\mathbf{Y}} (\theta^2_{\mathbf{Y}})^{\top}\right\rangle+4\left\langle \tx-\mxs,\mathbf{Y} \mathbf{Y}^{\top}-\mathbf{X}^*\right\rangle$, we have
\begin{align*}
\begin{aligned}
&\left\|\theta^2_{\mathbf{Y}} (\theta^2_{\mathbf{Y}})^{\top}\right\|_{\mathrm{F}}^2-3\left\|\mathbf{Y} \mathbf{Y}^{\top}-\mathbf{X}^*\right\|_{\mathrm{F}}^2+2\left\langle \tx-\mxs, \theta^2_{\mathbf{Y}} (\theta^2_{\mathbf{Y}})^{\top}\right\rangle\\
&+4\left\langle \tx-\mxs,\mathbf{Y} \mathbf{Y}^{\top}-\mathbf{X}^*\right\rangle\\
&\le-\left\|\mathbf{Y} \mathbf{Y}^{\top}-\mathbf{X}^*\right\|_{\mathrm{F}}^2+2\left\langle \tx-\mxs, \theta^2_{\mathbf{Y}} (\theta^2_{\mathbf{Y}})^{\top}\right\rangle\\
&+4\left\langle \tx-\mxs,\mathbf{Y} \mathbf{Y}^{\top}-\mathbf{X}^*\right\rangle &[\text{\eqref{app:eq:(Y1-Y2Q)}} ]\\
&=-\left\|\mathbf{Y} \mathbf{Y}^{\top}-\mathbf{X}^*\right\|_{\mathrm{F}}^2+2\left\langle \tx-\mxs, \theta^2_{\mathbf{Y}} (\theta^2_{\mathbf{Y}})^{\top}\right\rangle+4\left\langle \tx-\mxs,\mathbf{Y} \mathbf{Y}^{\top}\right\rangle &[\left\langle \tx-\mxs, \mxs\right\rangle =0 ] \\
&\le -\left\|\mathbf{Y} \mathbf{Y}^{\top}-\mathbf{X}^*\right\|_{\mathrm{F}}^2+6\frac{\ee^2\sigma_{r+1}^2(\mLambda)}{\left|\sigma_{r}^2(\mLambda) - \e +\sigma_{r+1}^2(\mLambda)\right|^2}
&[\text{\eqref{eq:<Delta_n-mxs,YY>},\eqref{eq:<Delta_n-mxs,theta_Y theta_Y>}}] \\
&\le -2(\sqrt{2}-1) \sigma_r^2\left(\mathbf{Y}^*\right)\left\|\theta^2_{\mathbf{Y}}\right\|_{\mathrm{F}}^2 +6\frac{\ee^2\sigma_{r+1}^2(\mLambda)}{\left|\sigma_{r}^2(\mLambda) - \e +\sigma_{r+1}^2(\mLambda)\right|^2} 
&[ \text{\eqref{App:eq:d([Y1,Y2])}}]\\
\end{aligned}
    \end{align*}
Second, for $\left\langle\overline{\operatorname{grad} H([\mathbf{Y}])}, \theta^2_{\mathbf{Y}}\right\rangle$,
\begin{align*}
    \left\langle\overline{\operatorname{grad} H([\mathbf{Y}])}, \theta^2_{\mathbf{Y}}\right\rangle\le &\|\overline{\operatorname{grad} H([\mathbf{Y}])}\|_{\mathrm{F}}\left\|\theta^2_{\mathbf{Y}}\right\|_{\mathrm{F}} \\
&\le\alpha \sigma_r^2\left(\mathbf{Y}^*\right)\left\|\theta^2_{\my}\right\|_{\mathrm{F}}^2 
\end{align*}
where the last inequality is because $\|\overline{\operatorname{grad} H([\mathbf{Y}])}\|_{\mathrm{F}} \leqslant \alpha \mu \sigma_r^3\left(\mathbf{Y}^*\right) /\left(4 \kappa^*\right)$. According to the definition of $\mathcal{R}_2$ in \eqref{eq-def:R in text}, $\my\in \mathcal{R}_2$ also implies $d\left([\mathbf{Y}],\left[\mathbf{Y}^*\right]\right)>\mu \sigma_r\left(\mathbf{Y}^*\right) / \kappa^*$, then
$$
\|\overline{\operatorname{grad} H([\mathbf{Y}])}\|_{\mathrm{F}} \leqslant \alpha d\left([\mathbf{Y}],\left[\mathbf{Y}^*\right]\right) \sigma_r^2\left(\mathbf{Y}^*\right) / 4=\alpha\left\|\theta^2_{\mathbf{Y}}\right\|_{\mathrm{F}} \sigma_r^2\left(\mathbf{Y}^*\right) / 4
$$

By combining the above three inequalities, we have
\begin{align*}
    &\overline{\operatorname{Hess} H([\mathbf{Y}])}\left[\theta^2_{\mathbf{Y}}, \theta^2_{\mathbf{Y}}\right]=\left\|\theta^2_{\mathbf{Y}} (\theta^2_{\mathbf{Y}})^{\top}\right\|_{\mathrm{F}}^2-3\left\|\mathbf{Y} \mathbf{Y}^{\top}-\mathbf{X}^*\right\|_{\mathrm{F}}^2\\
    &+4\left\langle\overline{\operatorname{grad} H([\mathbf{Y}])}, \theta^2_{\mathbf{Y}}\right\rangle+2\left\langle \tx-\mxs, \theta^2_{\mathbf{Y}} (\theta^2_{\mathbf{Y}})^{\top}\right\rangle
    +4\left\langle \tx-\mxs,\mathbf{Y} \mathbf{Y}^{\top}-\mathbf{X}^*\right\rangle \\
    &\le(\alpha-2(\sqrt{2}-1)) \sigma_r^2\left(\mathbf{Y}^*\right)\left\|\theta^2_{\my}\right\|_{\mathrm{F}}^2 +6\frac{\ee^2\sigma_{r+1}^2(\mLambda)}{\left|\sigma_{r}^2(\mLambda) - \e -\sigma_{r+1}^2(\mLambda)\right|^2} \\
    &\le \left((\alpha-2(\sqrt{2}-1)) \sigma_r^2\left(\mathbf{Y}^*\right)+6\frac{\alpha^2 \sigma_r^4\left(\mathbf{Y}^*\right)\sigma^2_{r+1}(\mLambda) /16}{\left|\sigma_{r}^2(\mLambda) - \e-\sigma_{r+1}^2(\mLambda)\right|^2 }\right)\left\|\theta^2_{\my}\right\|_{\mathrm{F}}^2
\end{align*}
where the last inequality follows from $\mu\sigma_r\left(\mathbf{Y}^*\right)/\kappa^*\le d\left([\mathbf{Y}],\left[\mathbf{Y}^*\right]\right)=\|\theta_\my\|_{\mathrm{F}} $ and the definition of $\ee$ in Assumption \ref{assum:R2:Region with Negative Eigenvalue in the Riemannian Hessian}.

Finally, according to the third assumption in Assumption \ref{assum:R2:Region with Negative Eigenvalue in the Riemannian Hessian}, one can guarantee the right-hand side of this bound is negative, which implies that $\theta_\my^2$ is the escaping direction in this scenario.

Combining all the discussion, this finishes the proof of this theorem.
\end{proof}

\begin{remark}
    Theorem \ref{thm-R2:Region with Negative Eigenvalue in the Riemannian Hessian} suggests that if some spectral values of $\my$ are small, then the descent direction $\theta_\my^1$ should increase them, If all of the spectral values of $[\my]$ are large enough compared with $\sigma_r(\mLambda)$, then $\theta_\my^2$ should directly point $[\my]$ to $[\mys]$. Theorem \ref{thm-R2:Region with Negative Eigenvalue in the Riemannian Hessian} fully characterizes the regime of $\my$ with respect to different minimum spectral values of $\my$.
    
    \begin{itemize}[nosep, leftmargin=*]
        \item If any spectral value of $\my\my^\top$ is smaller than $\frac{\sigma_{r+1}^2(\mLambda)}{2}$, then we have 
        \begin{align*}
            \overline{\operatorname{Hess} H([\mathbf{Y}])}\left[\theta^1_{\mathbf{Y}}, \theta_{\mathbf{Y}}^1\right]\le -\frac{\sigma_{r+1}^2(\mLambda)}{2}\|\theta^1_\my\|^2
        \end{align*}
        \item When the smallest absolute spectral value of $\my\my^\top$ is larger than $\frac{\sigma_{r+1}^2(\mLambda)}{2}$ and smaller than $\eee+\sigma_{r+1}^2(\mLambda)$, then we have 
        \begin{align*}
            \overline{\operatorname{Hess} H([\mathbf{Y}])}\left[\theta^1_{\mathbf{Y}}, \theta_{\mathbf{Y}}^1\right]\le
            -2\left(\sigma_r^2(\mLambda) \left(1-\frac{\eee^2}{\left|\sigma_r^2(\mLambda) - \eee-\sigma_{r+1}^2(\mLambda)\right|^2}\right)-\eee-\sigma_{r+1}^2(\mLambda)\right)\|\theta^1_\my\|^2
        \end{align*}
        \item If all of the spectral values of $\my\my^\top$ is larger than $\frac{\alpha \mu \sigma_r^3\left(\mathbf{Y}^*\right)} {2\sqrt{2} \kappa^*\sigma_{r+1}(\mLambda)}+\sigma_{r+1}^2(\mLambda)$, then we have $\overline{\operatorname{Hess} H([\mathbf{Y}])}\left[\theta^2_{\mathbf{Y}}, \theta_{\mathbf{Y}}^2\right]$ is smaller than 
        \begin{align*}
         \left((\alpha-2(\sqrt{2}-1)) \sigma_r^2\left(\mathbf{Y}^*\right)+6\frac{\alpha^2 \sigma_r^4\left(\mathbf{Y}^*\right)\sigma^2_{r+1}(\sigma) /16}{\left|\sigma_{r}^2(\mLambda) - \e-\sigma_{r+1}^2(\mLambda)\right|^2}\right)\left\|\theta^2_{\mathbf{Y}}\right\|_{\mathrm{F}}^2
        \end{align*}
    \end{itemize}
\end{remark}

\begin{remark}
    The eigengap assumption is crucial in discussing the three regions of the minimum singular value of $\my$. Without this eigengap assumption and under the current quotient geometry, the third regime cannot lead to a strong convexity result because any span on the eigenspace are all global solutions. We comment that it is possible to change the quotient geometry to show a new strong convexity result when the eigengap assumption does not hold.
\end{remark}



Finally, we look at the last main result. Theorem \ref{thm:R3} guarantees that when $\my\in \mathcal{R}_3$, the magnitude of the Riemannian gradient descent is large. The proof of Theorem \ref{thm:R3} directly follows from the proof of \cite{luo2022nonconvex} without any modification. Hence, we do not repeat it here. Notice that $\my\in \mathcal{R}_3$ does not require Assumption \ref{assum-riemannian:eigengap in text} because $\mathcal{R}_3$ describes the case that $[\my]$ is far away from the FOSP.

\largeGradient*

\end{document}